\documentclass[11pt]{article} 

\RequirePackage[OT1]{fontenc}
\RequirePackage[colorlinks,citecolor=blue,urlcolor=blue]{hyperref}

\usepackage{verbatim}
\usepackage{url}
\usepackage[margin=2.9cm]{geometry}

\usepackage{graphicx}
\usepackage{epstopdf}
\usepackage{amsmath,amsfonts,amssymb,amsthm}
\usepackage{array}
\usepackage{graphicx}
\usepackage{bm}
\usepackage{multirow}
\usepackage[font=footnotesize]{caption}
\usepackage{subcaption}
\usepackage{algorithm}
\usepackage{algpseudocode}
\usepackage{cleveref}

\crefname{equation}{}{}

\newfloat{algorithm}{t}{lop}

\usepackage{arash_macros2}

\newcommand{\Zc}{\mathcal{Z}}

\newcommand{\onem}{\bm{E}}
\newcommand{\onev}{\bm{1}}
\newcommand{\pp}{p}
\newcommand{\qqq}{q}
\newcommand{\ipbig}[1]{\big\langle #1 \big\rangle}
\DeclareMathOperator*{\tr}{tr}
\newcommand{\symBM}{\text{PP}}
\newcommand{\bsymBM}{\text{PP}^{\text{bal}} }

\newcommand{\Zcz}[1]{\Zc_{\text{orbit}}{#1}}

\newcommand{\Xc}{\mathcal{X}}
\newcommand{\Xcf}{\Xc_{\text{free}}}
\newcommand{\Xcfg}[1]{\Xc_{\text{free}}^{#1}}
\newcommand{\Xcz}[1]{\Xc_{\text{orbit}}{#1}}

\newcommand{\Nc}{\mathcal{N}}
\newcommand{\nb}{{\underline{\bm{n}}}}
\newcommand{\Ac}{\mathcal{A}}
\newcommand{\proj}{\Pi}
\newcommand{\Sc}{\mathcal{S}}
\newcommand{\mnorm}[1]{|\!|\!|#1|\!|\!|}

\newcommand{\Xh}{\widehat{X}}
\newcommand{\Zh}{\widehat{Z}}
\newcommand{\Mh}{\widehat{M}}

\newcommand{\Uh}{\widehat{U}}

\newcommand{\Lc}{\mathcal{L}}

\newcommand{\Symm}[1]{\mathbb{S}^{#1}}
\newcommand{\subb}[1]{S_{#1}}
\newcommand{\Xtru}{X_0}

\newcommand{\Xt}{\widetilde{X}}
\DeclareMathOperator{\rangeS}{range}
\DeclareMathOperator{\Span}{span}
\DeclareMathOperator{\var}{Var}
\newcommand{\pb}{\bar{p}}
\newcommand{\qb}{\bar{q}}

\newcommand{\dg}[2]{d_{#1}(\subb{#2})}
\newcommand{\mub}{\bar{\mu}}
\DeclareMathOperator{\bern}{Bern}
\newcommand{\projonep}{P_{\onev_m^\perp}}
\newcommand{\dav}[2]{d_{\text{av}}(\subb{#1},\subb{#2})}
\newcommand{\basisv}[2]{e_{#1}^{#2}}
\newcommand{\bv}[1]{e_{#1}}
\newcommand{\Star}{\bm{\Phi}}

\newcommand{\bBM}{\textrm{BM}^{\textrm{bal}}}
\newcommand{\At}{\widetilde{A}}
\DeclareMathOperator{\SDP}{SDP_{sol}}

\newcommand{\Qm}{\Psi}
\newcommand{\Qmt}{\widetilde{\Psi}}
\newcommand{\pt}{\widetilde{p}}
\newcommand{\qt}{\widetilde{q}}

\newcommand{\qs}{q^*}
\newcommand{\qsmax}{q^*_{\max}}
\newcommand{\qsb}{\bar{q}^*}
\newcommand{\qsub}{\widetilde{q}^{\,*}}
\newcommand{\qsbmax}{\bar{q}^{*}_{\max}}
\newcommand{\qsubmax}{\widetilde{q}^{\,*}_{\max}}
\newcommand{\qmax}{q_{\max}}

\newcommand{\mubar}{\phi}
\newcommand{\phib}{\bar{\phi}}

\newcommand{\dualtri}{(\mu,\nu,\Gamma)}

\newcommand{\Stru}{\bm{S}_0}
\newcommand{\pub}{\widetilde{p}}
\newcommand{\qub}{\widetilde{q}}

\newcommand{\Bt}{\widetilde{B}}

\newcommand{\psd}[1]{\Symm{#1}_+}

\newcommand{\SBM}{SBM}
\newcommand{\ChenXu}{Chen \& Xu}
\newcommand{\CaiLi}{Cai \& Li}

\newcommand{\pminus}{p^-}
\newcommand{\qplus}{q^+}
\newcommand{\Qmh}{\widehat{\Qm}}

\newcommand{\bs}{n}
\providecommand{\Mt}{\widetilde{M}}

\DeclareRobustCommand{\textprime}{$'$}
\newcommand{\SDPp}[1]{SDP-#1\textprime}

\newcommand{\avgd}{d}

\numberwithin{equation}{section}
\theoremstyle{plain}
\newtheorem{thm}{Theorem}[section]
\newtheorem{prop}{Proposition}[section]
\newtheorem{cor}{Corollary}[section]
\newtheorem{lem}{Lemma}[section]
\theoremstyle{definition}
\newtheorem{defn}{Definition}[section]
\newtheorem{rem}{Remark}[section]

\usepackage{authblk}

\title{On semidefinite relaxations for the block model}

\author{Arash A. Amini}
\author{Elizaveta Levina}
\affil{Department of Statistics, UCLA and \\
        Department of Statistics, University of Michigan}

\begin{document}

\maketitle

\begin{abstract}
	The stochastic block model (SBM) is a popular tool for community
	detection in networks, but fitting it by maximum likelihood (MLE)
	involves a computationally infeasible optimization problem.  We propose a new
	semidefinite programming (SDP) solution to the problem of fitting the
	SBM, derived as a relaxation of the MLE.   We put ours and previously proposed SDPs in a unified framework, as relaxations of the MLE over various sub-classes of the SBM, which also reveals a connection to the well-known problem of sparse PCA.      Our main relaxation, which we call SDP-1, is tighter
	than other recently proposed SDP relaxations, and thus previously
	established theoretical guarantees carry over. However, we show that SDP-1 exactly recovers true communities over a wider class of SBMs than those covered by current results. In particular, the assumption of strong assortativity of the SBM, implicit in consistency conditions for previously proposed SDPs, can be relaxed to weak assortativity for our approach, thus significantly broadening the class of SBMs covered by the  consistency results.  
	We also show that strong assortativity is indeed a necessary condition for exact recovery for previously proposed SDP approaches and not an artifact of the proofs.  Our analysis of SDPs is based on primal-dual witness constructions, which provides some insight into the nature of the solutions of various SDPs. In particular, we show how to combine features from SDP-1 and already available SDPs to achieve the most flexibility in terms of both assortativity and block-size constraints, as our relaxation has the tendency to produce communities of similar sizes.    This tendency makes it the ideal
	tool for fitting network histograms,  a method gaining popularity in the graphon estimation literature, as we illustrate on an example of a social networks of dolphins. We also provide empirical evidence that SDPs outperform spectral methods for fitting SBMs with a large number of blocks.

\end{abstract}

\section{Introduction}
Community detection, one of the fundamental problems in network analysis, has attracted a lot of attention in a number of fields, including computer science, statistics, physics, and sociology.   The stochastic block model (\SBM) \cite{Holland83} is a well-established and widely used model for community detection, attractive for its analytical tractability and connections to fundamental properties of random graphs \cite{Aldous81, Bickel&Chen2009, Mossel.et.al.2012}, but fitting it to data is a challenge due to the need to optimize over $K^n$ assignments of $n$ nodes to $K$ communities.  Many fitting methods have been proposed, including profile likelihood~\cite{Bickel&Chen2009},  MCMC~\cite{Snijders&Nowicki1997,Nowicki2001},  variational approaches~\cite{Airoldi2008,Celisseetal2011,Bickel&Choi&etal2012},   belief propagation~\cite{Decelle.et.al.2011}, and pseudo-likelihood~\cite{Amini.et.al.2013}, the latter two being more or less the current state of the art in speed and accuracy.  However, all these methods rely on a good initial value and can be sensitive to starting points.   In contrast, spectral clustering methods do not require an initial value, are fast and have also been popular in community detection \cite{Rohe2011, Chaudhuri&Chung&Tsiatas2012, Lei&Rinaldo2013, Qin&Rohe2013}.  Spectral clustering works reasonably well in dense networks with balanced communities but fails on sparse networks~\cite{Le2015}.  Regularization can help \cite{Chaudhuri&Chung&Tsiatas2012, Amini.et.al.2013,Joseph2013}, but even regularized spectral clustering does not achieve the accuracy of likelihood-based methods when they are given a good initial value \cite{Amini.et.al.2013}. 

Recently, semidefinite programming (SDP) approaches to fitting the \SBM\ have appeared in the literature~\cite{Chen2012,Chen2014,Cai2014}, which rely on a SDP relaxation of the computationally infeasible likelihood optimization problem.   They are attractive because, on one hand, they solve a global optimization problem and require no initial value, and on the other hand, they are still maximizing the likelihood,  and one can therefore hope for better performance than from generic methods like spectral clustering, which do not use the likelihood in any way.  As global optimization methods, they are easier to analyze than iterative methods depending on a starting value.  It also appears that SDP relaxations in themselves have a regularization effect, which makes their solutions more robust to noise and outliers (see  Remark~\ref{rem:outliers}).    One drawback of SDP methods is the higher computational cost of SDP solvers.  However, by formulating the problem as a SDP, we can benefit from continuous advances in solving large scale SDPs, an active area of research in optimization.

In this paper, we propose a new SDP relaxation of the likelihood optimization problem, which is tighter than any of the previously proposed SDP relaxations \cite{Chen2012,Chen2014,Cai2014}.   We also put all these relaxations into a unified framework, by viewing them as versions of the MLE restricted to different parameter spaces, and show their connection to the well-studied problem of sparse PCA.   Empirically, the tighter relaxation gives better results, and we derive a first-order SDP implementation via ADMM which keeps computing costs reasonable.

On the theoretical side, our focus for the most part will be on balanced models, i.e., those with equal community sizes. We obtain sufficient conditions on the parameters of the block model for strong consistency (i.e., exact recovery of communities) of our relaxation, SDP-1. These conditions guarantee success over a wider class of \SBM s than in previous literature. Current conditions for the success of SDP relaxations implicitly impose what we will call strong assortativity, whereas our SDP succeeds for any 
weakly assortative \SBM~(cf.\ Definition~\ref{def:strong:weak:assort}), when the expected degree grows as $\Omega(\log n)$. We also show that the requirement of strong assortativity is necessary for the success of previous SDP relaxations (SDP-2 and SDP-3 in Table~\ref{SDP-table}), and it is not an artifact of proof techniques (Section~\ref{sec:failure:SDP-2}). Our proof of the success of SDP-1 is based on a primal-dual witness construction
which has already been used successfully in the context of sparse recovery problems; see for example~\cite{Wainwright2009,Amini2009}. In the context of SDP relaxations for the \SBM, however, the only instance of this approach that we know of is the recent work of~\cite{Abbe2014}, for the case of the $K=2$ \SBM. Our approach can be viewed as a non-trivial extension of~\cite{Abbe2014}  to the case of general $K$, and a more complex SDP with the doubly nonnegative cone constraint and more equality constraints. As a by-product, we also recover the current results for SDP-2 for the class of strongly assortative \SBM s. Our results suggest that the greater divide for SDP relaxations is not between strongly and weakly assortative \SBM s, but between purely assortative (or dissortative) and mixed models, those with both assortative and dissortative communities. 

SDP-1, in its basic form, tends to partition the network into blocks of similar sizes.  This is sometimes an unwelcome feature in practice, and sometimes a desirable one, since very large and very small communities are generally difficult to interpret.    If this feature is not desirable, SDP-1 can be modified to allow for different block sizes, as discussed in Section~\ref{sec:unbalanced:extension}.    The equal sized blocks are especially suitable for consructing network histograms, a method for graphon estimation proposed by \cite{Olhede2013}.  Viewing the \SBM\ as a nonparametric approximation to a general reasonably smooth mean function of the adjacency matrix (the graphon) is analogous to constructing a histogram to approximate a general smooth density function.  A number of methods for graphon estimation have been proposed recently \cite{Airoldi.et.al.2013,Wolfe2013,Zhang2015}, and the network histogram as a graphon estimator has been proposed in~\cite{Olhede2013}.   
A histogram is appealing because it is controlled by the number of bins (blocks) $K$, which is a single parameter that can be chosen to balance fitting the data with robustness to noise. In this case, it is particularly appropriate to fit blocks of equal or similar sizes, just like in the usual histogram.   
%
We show empirically in Section \ref{sec:numeric} that our SDP relaxation provides the best tool for histogram estimation, as well as  generally cleaner solutions, compared to other less tight SDP relaxations and generic methods like spectral clustering.

The rest of the paper is organized as follows.  In Section~\ref{sec:SBM:def}, we introduce the \SBM\ and its submodels. We derive a general blueprint for MLE relaxations in Section~\ref{sec:MLE:relax}, introduce our proposed SDP and compare with the ones existing in the literature, including a brief discussion of the connection with sparse PCA. Section~\ref{sec:consist:results} presents our consistency results for balanced block models, along with an overview of the proofs. A result showing the failure of SDP-2 in the absence of strong assortativity (which is not needed for our relaxation)  appears in Section~\ref{sec:failure:SDP-2}.  Extension to the case of unbalanced communities is dicussed in Section~\ref{sec:unbalanced:extension}.  Section~\ref{sec:graphon} presents application of SDP-1 to graphon estimation via fitting network histograms. Section~\ref{sec:numeric} compares several SDPs numerically, and we conclude with a discussion in Section~\ref{sec:discuss}. Technical details of the proofs and a brief discussion of a first-order method for implementing SDP-1 can be found in Appendices.


\medskip
\textbf{Notation.} We use $\otimes$ to denote Kronecker product of matrices, and $\circ$ to denote Schur (element-wise) product of matrices.  $\Symm{n}$ denotes the set of symmetric $n\times n$ matrices, and $\ip{A,X} := \tr(AX)$ the corresponding inner product.  $(\psd{n},\succeq)$  is the cone of positive semidefinite (PSD) $n\times n$ matrices, and its natural partial order, namely, $A \succeq B$ iff $A-B \in \psd{n}$.  $\onem_{n,m}$ is the $n \times m$ matrix of all ones and $\onem_n$ is the $n \times n$ matrix of all ones.   A $n \times 1$ vector of all ones is denoted $\onev_n$.    $\|u\| = \|u\|_2$ is the $\ell_2$ norm of vector $u$, and $\mnorm{A} = \mnorm{A}_2$ is the $\ell_2  \rightarrow \ell_2$ operator norm of matrix $A$.
 $\ker(A)$ and $\rangeS(A)$ denote the kernel (null space) and the range (column space) of matrix $A$. $\diag: \reals^{n \times n} \to \reals^n $ acts on square matrices and extracts the diagonal.  $\diag^*: \reals^n \to \reals^{n \times n}$ is the adjoint of $\diag$, acting on vectors, producing the natural diagonal matrix. For a matrix $X$, let $\supp(X) := \{(i,j):\; X_{ij} \neq 0\}$ be its support. More specialized notation is introduced in Section~\ref{sec:aux:results}.

\section{The stochastic block model}\label{sec:SBM:def}
We now formally introduce the \SBM.   The network data (nodes and edges connecting them) are represented by a
simple undirected graph on $n$ nodes via its $n \times n$ adjacency
matrix $A$, a binary symmetric matrix with 
$A_{ij} = 1$ if there is an edge between nodes $i$ and $j$, and 0
otherwise.   Each node belongs to exactly one community, specified by
its membership vector $z_i \in \{0,1\}^K$, with exactly one
nonzero entry, $z_{ik} = 1$, indicating that node $i$ belongs to community $k$.  The vectors $z_i$ are not observed.   
The \SBM\ is parametrized through the symmetric probability matrix $\Qm \in
[0,1]^{K \times K}$, where $\Qm_{kr}$ is the probability of an edge
forming between a pair of nodes from communities $k$ and $r$. For
simplicity, we assume $n$ is a multiple of $K$. 

Given $z_i$ and $\Qm$, $\{A_{ij}, i < j\}$ are drawn independently as
Bernoulli random variables with $\ex[A_{ij} | z_i,z_j] = z_i^T \Qm
z_j$. Let $Z$ be the $n \times K$ matrix with rows $z_1^T,\dots,z_n^T$. Then, we can write the model as
\begin{align}\label{eq:bm:mean:def}
 	M_Z:= \ex[A | Z] = Z \Qm Z^T. 
\end{align}
Note that $A_{ii}$'s are so far undefined. They can be defined based on convenience, but we will always assume that they are defined so that~\eqref{eq:bm:mean:def} holds over all elements. (For example, one possibility is to set $A_{ii} := [M_Z]_{ii}$.) We do not treat $\{A_{ii}\}$ as part of the observed data.

$M_Z$ is a block constant, rank $K$ matrix, and we can think of the
operation $\Qm \mapsto Z \Qm Z^T$ as a block constant embedding of a $K
\times K$ matrix into the space of $n \times n$ matrices.   This
provides us with a simple but useful property:  for any matrix $M$ and function $f$ on $\reals$, let $f \circ M$ be the pointwise application of $f$ to the entires of $M$, $[f \circ M]_{ij} = f(M_{ij})$. Then, we have 
\begin{align}\label{eq:blk:const:ident}
  f \circ (Z \Qm Z^T) = Z (f \circ \Qm) Z^T . 
\end{align}
Using~\eqref{eq:blk:const:ident}, we can write the log-likelihood of the
\SBM\ in a compact form. First note that
\begin{align*}
  \ell(Z,\Qm) &= \sum_{i <j} A_{ij} \log [M_Z]_{ij} +
  (1-A_{ij}) \log\big(1-[M_Z]_{ij}\big)  \\
  &= \sum_{i <j} A_{ij} [f\circ M_Z]_{ij} + [g\circ M_Z]_{ij} 
\end{align*}
where $f(x) := \log \frac{x}{1-x}$ and $g(x) := \log(1-x)$ are
functions on $[0,1]$. Recall that for symmetric matrices $A$ and $B$, we defined $\ip{A,B} := \tr(AB)$, and $\onem_n$ is the $n
\times n$ matrix of ones. Let $\ip{A,B}_0 := \ip{A,B} - \sum_i A_{ii} B_{ii}$, that is, the inner product defined through off-diagonal elements of $A$ and $B$. 
Using~\eqref{eq:blk:const:ident}, 
\begin{align}\label{eq:log:like:1}
  2\ell(Z,\Qm) &= \ip{A, f\circ M_Z}_0 + \ip{\onem_n, g\circ M_Z}_0 \notag 
  \\&= \ip{A, Z (f\circ \Qm) Z^T}_0 + \ip{\onem_n, Z (g \circ \Qm) Z^T}_0.  
\end{align}

In deriving the SDPs, our focus will be on the following two special cases of the \SBM.  In Section~\ref{sec:consist:results}, we will show the SDPs work for larger classes of \SBM s than those they are derived for.
\begin{itemize}
	\item[($\symBM$)] The planted partition (PP) model,
          $\symBM(p,q)$,  defined by just two parameters $\pp$ and
          $\qqq$ via  
\begin{equation}\label{eq:PP:Psi}
  \Qm = \qqq \onem_K + (\pp-\qqq) I_K, 
\end{equation}
 where $I_K$ is the $K \times K$ identity matrix, and following the PP
 literature we assume $\pp > \qqq$. Note that~\eqref{eq:PP:Psi} simply means that the diagonal elements are $p$ and the off-diagonal elements are $q$.
	\item[($\bsymBM$)] The balanced planted partition model,
          $\bsymBM(p,q)$, which is  $\symBM(p,q)$ with the additional assumption that the blocks have equal sizes.
\end{itemize}

For $\symBM(p,q)$, the likelihood greatly simplifies, since $f \circ \Qm$
and $g \circ \Qm$ take only two values, 
\begin{align*}
  f\circ \Qm = f(q) \onem_K + [f(p) - f(q)] I_K 
\end{align*}
and similarly for $g \circ \Qm$.  Since $Z \onem_K Z^T = \onem_n$, 
\eqref{eq:log:like:1} becomes 
\begin{align*}
  2\ell(Z,\Qm)  &= [f(p) - f(q)] \ipbig{A, ZZ^T}_0 + [g(p) - g(q)] \ipbig{\onem_n, ZZ^T}_0 + \text{const} ,
\end{align*}
where the constant term does not depend on $Z$. With the condition $p > q$, we have $f(p) > f(q)$ and $g(p) < g(q)$. Then, we obtain,
\begin{align}\label{eq:log:like:Z-form:1}
  \frac{2\ell(Z,\Qm)}{f(p) - f(q)} =
   \ipbig{A, ZZ^T} - \lambda  \ipbig{\onem_n, ZZ^T} + \text{const}, \quad 
    \lambda := \frac{g(q) - g(p)}{f(p) - f(q)} > 0.
\end{align}
Note that we have safely replaced $\ip{\cdot\,,\cdot}_0$ with $\ip{\cdot\,,\cdot}$, possibly changing the constant, since $[ZZ^T]_{ii} = 1, \, \forall i$ regardless of $Z$. A similar calculation appears in~\cite{Cai2014}, albeit in a slightly different form.

\section{Relaxing the maximum likelihood estimator (MLE)}\label{sec:MLE:relax}
Given the adjacency matrix $A$,  the MLE for $(Z,\Qm)$ is obtained by
maximizing the likelihood of the \SBM.   It is known to have desirable
consistency and in some sense optimality properties
\cite{Bickel&Chen2009}, but the exact computation of the MLE is in
general NP-hard, due to the optimization over $Z$. However, it can
be relaxed to computationally feasible convex problems.

We can obtain a class of MLEs by varying the domain over which the
likelihood~\eqref{eq:log:like:Z-form:1} is maximized. 
That is, we have the general estimator 
\begin{align}\label{eq:gen:optim:Zspace}
	\widehat{Z} := \argmax_{Z\, \in \Zc} \; \ipbig{A, ZZ^T} -
        \lambda  \ipbig{\onem_n, ZZ^T} \ . 
\end{align}
Each $Z$ corresponds to a \emph{clustering} matrix $X = ZZ^T \in
\{0,1\}^{n \times n}$, where $X_{ij} = 1$ if $i$ and $j$ belong to the
same community, and $X_{ij}=0$ otherwise. Any subset $\Zc$ in the
$Z$-space induces a corresponding subset $\Xc$ in the $X$-space. We
can consider estimators of $X$, and our blueprint for
deriving different relaxations will be varying the space $\Xc$ in the
optimization problem  
\begin{align}\label{eq:gen:optim:Xspace}
	\widehat{X} := \argmax_{X\, \in \Xc} \; \ipbig{A, X} 
	- \lambda  \ipbig{\onem_n, X} \ .
\end{align}

\subsection{Our relaxation: SDP-1}
Our relaxation corresponds to the balanced model $\bsymBM(p,q)$, in
which each community is of size $n/K$. In this case, all admissible
$Z$'s can be obtained by permutation of any fixed admissible $Z_0 = I_K  \otimes \onev_{n/K}$, and we can take the feasible set $\Zc$ in~\eqref{eq:gen:optim:Zspace} to be
\begin{align*}
  \Zcz(Z_0) := \big\{P Z_0 Q:\;  \;
    P, \,Q \, \text{ are permutation matrices}  \big\},
\end{align*}
 where $\otimes$ is the Kronecker product and $\onev_{n/K}$ is the
 vector of all ones of length $n/K$. This choice of $Z_0$ is for
 convenience and corresponds to assigning nodes consecutively to
 communities 1 through $K$.    Recalling $X = ZZ^T$, the corresponding feasible set in the $X$-space is 
\begin{align}\label{eq:Xorbit}
  \Xcz(\Xtru) := \big\{ P \Xtru P^T: \; \; 
    P\,  \text{ is a permutation matrix} \big\}, \; \Xtru = I_K \otimes \onem_{n/K}.
\end{align}
Note that $\Xtru$ is block-diagonal with all the diagonal 
blocks equal to $\onem_{n/K}$. 

In order to relax $\Xcz(\Xtru)$, we first note that any $X$ in this set
is clearly positive semidefinite (PSD), denoted by $X \succeq 0$, since
$X = (PZ_0)(PZ_0)^T$. In addition, $0 \le X_{ij} \le 1$ for all $i$, $j$, which we
write as  $0 \le X \le 1$,  and $\diag(X) = \onev_n$.  Note that the
latter condition, $X \succeq 0$ and $X \ge 0$ imply $X \le 1$, since $1-X_{ij}^2 = X_{ii} X_{jj} - X_{ij}^2 \ge 0$ implying $X_{ij} = |X_{ij}| \le 1$. Finally, it is easy to see
that each row of $X$ should sum to $n/K$, i.e.,  $X \onev_n = (n/K) \onev_n$.  This implies we can remove the term $\lambda \ip{\onem_n, X}$ from the objective function in~\eqref{eq:gen:optim:Xspace}, since
\begin{align}\label{eq:X:Em}
	\ip{X,\onem_n} = \tr(X \onev_n \onev_n^T) = \onev_n^T X \onev_n = \onev_n^T (n/K) \onev_n = n^2/K,
\end{align}
which is a constant. Thus, we arrive at our proposed relaxation, {\bf SDP-1}:
\begin{align}\label{eq:our:sdp}
\begin{split}
  \def\arraystretch{1.3}
  \begin{array}{ll}
    \argmax_X & \ip{A,X}  \\
    \text{subject to} 
    & X \onev_n = (n/K) \onev_n, \; \; \diag(X) = \onev_n, \;\;
      X \succeq 0, \;X \ge 0  .  
  \end{array}
\end{split}
\end{align}

\subsection{Other relaxations:  SDP-2 and SDP-3}
Two other interesting SDP relaxations have recently appeared in the
literature. First, we will consider the relaxation of
\ChenXu~\cite{Chen2014}; see also~\cite{Chen2012}. They essentially work with the same $\bsymBM(p,q)$,
although their model is slightly more general (see
Remark~\ref{rem:outliers}). The main relaxation proposed
in~\cite{Chen2014} is via constraining the nuclear norm of $X$, a common heuristic for constraining the rank. Since $X$ is
PSD, we obtain $\mnorm{X}_* = \tr(X) = n$. In addition, they impose a
single affine constraint, namely ~\eqref{eq:X:Em}. Thus, their main
focus is on the relaxation which replaces $\Xcz(\Xtru)$ with $\{X: \,
\mnorm{X}_* \le n,\; \ip{X,\onem_n} = n^2/K, \; 0 \le X \le 1\}$. However,
they briefly mention a much tighter SDP relaxation which imposes
positive semi-definiteness directly. This is what we have called {\bf SDP-2}, shown in Table~\ref{SDP-table}.

Note that $X \succeq 0$ and $\tr(X) = n$ imply $\mnorm{X}_* = n$, which is much tighter than $\mnorm{X}_* \le n$.
The main difference between SDP-2 and our relaxation is that we impose the constraint $\ip{\onem_n,X} = n^2/K$ more restrictively, by breaking it into $n$ separate affine constraints. We also break the $\tr(X) = n$ into $n$ pieces, but that does not seem to make much of a difference.

\begin{table}[t]
\caption{SDP relaxations}
\label{SDP-table}
\renewcommand{\arraystretch}{1.3}
\small
\begin{center}
\begin{tabular}{c|c|c|c|c}
& SDP-1 & SDP-2 & SDP-3 & EVT \\ \hline
maximize & $\ip{A,X}$ & $\ip{A,X}$ & 
    $\ip{A,X} - \lambda\ip{\onem_n,X}$ & $\ip{A,X}$
\\ \hline 
\multirow{3}{*}{subject to} & $X \onev_n = (n/K) \onev_n$ 
     &  $\ip{\onem_n,X} = n^2/K$ & \\
     & $\diag(X)=\onev_n$ & $\tr(X) = n$ &  & $\tr(X) = n$
\\ \cline{2-5}
     &  $X \ge 0$ & \multicolumn{2}{c|}{$0 \le X \le 1$} & $\mnorm{X} _2\le n/K$
\\ \cline{2-5}
     & \multicolumn{3}{c}{$X \succeq 0$} 
\\ \hline

    model & \multicolumn{2}{c|}{$\bsymBM(p,q) \equiv \Xcz(\Xtru)$} 
    & $\symBM(p,q) \equiv \Xcf$
\end{tabular}
\end{center}
\end{table}

Next, we consider the relaxation of \CaiLi~\cite{Cai2014}, though in a
slightly different form.  This relaxation works for the more general model $\symBM(p,q)$. In this case, we are looking at the feasible set
\begin{align}\label{eq:Xfree}
	\Xcf = \{X = ZZ^T: \; \text{$Z$ is an admissible membership matrix} \}.
\end{align}
For $X \in \Xcf$, we still have $X \succeq 0$ and $X_{ij} \in
\{0,1\}$. Thus, one can simply relax to the problem denoted by {\bf SDP-3} in Table~\ref{SDP-table}.

Note that $\lambda \ip{\onem_n, X}$ remains in the objective, since
there are no constraints to make it constant. 
We cannot enforce an affine constraint involving $\ip{\onem_n,X}$
directly for $\Xcf$ without knowing the block sizes. In fact, let $\nb
= (n_1,\dots,n_K)$ be the vector of block sizes, and let $\onem_\nb :=
\diag^*(\onem_{n_1},\dots,\onem_{n_K})$ be the block-diagonal matrix with
diagonal blocks of all ones with sizes given by $\nb$. It is easy to see that $\Xcf$ is the union of orbits of all possible $\onem_\nb$, 
\begin{align}\label{eq:Xfree:union:orbit}
	\Xcf = \bigcup_{\nb: \; \|\nb\|_1 = n} \!\!\Xcz( \onem_\nb  ) = 
	\bigcup_{\|\nb\|_1 = n} \big\{ P  \onem_\nb P^T:\, \text{$P$ is a permutation matrix} \big\}
\end{align}
from which it follows that $\ip{\onem_n, X} = \|\nb\|_2^2 = \sum_j
n_j^2$, a function of the unknown $\{n_j\}$.

The optimal value for parameter $\lambda$, assuming the model is
$\symBM(p,q)$, is given in~\eqref{eq:log:like:Z-form:1} as a function
of $p$ and $q$. However, one
can think of $\lambda$ as a general regularization parameter
controlling the sparseness of $X$, noticing $\ip{\onem_n,X} =
\|X\|_1$ since $X \ge 0$. It is well known that the $\ell_1$ norm is a
good surrogate for a cardinality constraint when enforcing sparseness,
which leads us to a link to sparse PCA discussed in Section \ref{sec:sparse:PCA}. 

\begin{rem}\label{rem:outliers}
  Both~\cite{Chen2014} and~\cite{Cai2014} consider the effect of
  outliers on their SDPs. \CaiLi~\cite{Cai2014} derive the SDP for the
  model we described but they modify it by penalizing the trace, which
  is justified by their theory for a fairly general model of
  outliers. \ChenXu~\cite{Chen2014} start with a generalized version
  of $\bsymBM(p,q)$ which allows for a subset of nodes that belong to
  no community, and relax that model.  Our relaxation SDP-1 can also
  work for this generalized model if we replace $X \onev_n = (n/K)
  \onev_n$ with the inequality version $X \onev_n \le (n/K) \onev_n$. This has an advantage over \ChenXu's approach, since one does not need to know the number of outliers a priori. 


\end{rem}

\subsection{Connection with nonnegative sparse PCA}\label{sec:sparse:PCA} 
Representation~\eqref{eq:Xfree:union:orbit} suggests another natural
direction to restrict the parameter space. Note that $\|\nb\|_\infty =
\max_j n_j \in [n/K,n]$, as a consequence of $\|\nb\|_1 = n$. The
closer $\|\nb\|_\infty$ is to $n/K$, the more balanced the communities
are. This suggests the following class, 
\begin{align}\label{eq:Xfree:gam}
 	\Xcfg{\gamma} := \bigcup \Big\{
 	 \Xcz(\onem_\nb) :\; \|\nb\|_1 = n,
 		\; \|\nb\|_\infty \le \gamma (n/K) \Big\},
 \end{align} 
where $\gamma \in [1,K]$ measures the deviation from completely balanced communities. 
For $X \in \Xcfg{\gamma}$, note that 
$\mnorm{X}_2 = \mnorm{\onem_\nb}_2 
= \max_j \mnorm{ \onem_{n_j}}_2 
= \|\nb\|_\infty \le \gamma(n/K)$.  As before, we have $\tr(X) = n$,
$\|X\|_1 = \ip{\onem_n,X}$, and 
$X \in \Nc_+^n \!:= \{X\!: X \succeq 0, X \ge 0\}$, 
the doubly nonnegative cone.  Letting $\Xt = (K/n) X$, we have 
\begin{align}\label{eq:sPCA:SDP}
\begin{split}
  \def\arraystretch{1.3}
  \begin{array}{ll}
    \argmax_{\Xt} & \ip{A,\Xt} - \lambda\|\Xt\|_1 \\
        \text{subject to } 
        & \mnorm{\Xt}_2 \le \gamma, \; \tr(\Xt) = K, \;\;
        \Xt \succeq 0, \;\Xt \ge 0.
  \end{array}
\end{split}
\end{align}
Apart from the nonnegative constraint $\Xt \ge 0$ (which can be
removed to obtain a further relaxation), this is a generalization of
the SDP relaxation for  sparse PCA. Specifically, $\gamma = 1$
corresponds to the now well-known relaxation for recovering a sparse
$K$-dimensional leading eigenspace of $A$. The corresponding solution
$\Xt$ can be considered a generalized projection into this subspace,
see for example~\cite{Vu2013,D'Aspremont2004}, and note that $\Xt
\succeq 0, \mnorm{\Xt}_2 \le 1$ is equivalent to $0 \preceq \Xt
\preceq  I$.   We will not pursue this direction here, but it opens up
possibilities for leveraging sparse PCA results in network models.  

\subsection{Connection with adjacency-based spectral clustering}\label{sec:EVT:connection}
The first step in spectral clustering based on the adjacency matrix is the truncation of $A$ to its $K$ largest eigenvalues, which we call eigenvalue truncation (EVT). The resulting matrix $\Xt$ is the solution of a SDP maximizing $\ip{A,\Xt}$ subject to $\Xt \succeq 0,\; \tr(\Xt) = K,\; \mnorm{\Xt}_2 \le 1$. We can consider $X := (n/K) \Xt$ as an estimate of the cluster matrix by EVT. The resulting SDP appears in Table~\ref{SDP-table}, and will be our surrogate to compare the other SDPs to this particular version of spectral clustering. We should note that the more common form of spectral clustering, based on truncation to $K$ largest eignevalues in \emph{absolute} value is equivalent to applying EVT to $|A| = \sqrt{A^2}$.

The SDP formulation of EVT can be considered a relaxation of the MLE in $\bsymBM$, similar to the other SDPs we have considered. It is enough to note that $\mnorm{X}_2 = \mnorm{X_0}_2 = n/K$ for any $X \in \Xcz(\Xtru)$.  Also note that SDP-1 is a strictly tighter relaxation than  EVT.  To see that, take any $X$ which is feasible for SDP-1, and note that $X \onev_n = (n/K) \onev_n$ means that $\onev_n$ is an eigenvector of $X$ associated with eigenvalue $n/K$. The Perron-Frobenius theorem then implies that $\mnorm{X}_2 \le n/K$, hence $X$ is feasible for EVT.

\section{Strong consistency results}\label{sec:consist:results}
In this section, we provide consistency results for SDP-1 and a variant of SDP-2, which we will call SDP-$2'$. This version is obtained from SDP-2 by replacing $\tr(X) = n$ with $\diag(X) = \onev_n$ and removing the now redundant condition $X \le 1$.
This modification allows us to unify the  treatment of these two SDPs. For example, optimality conditions in Section~\ref{sec:optim:cond} are derived for a general blueprint~\eqref{eq:gen:SDP}, which includes both SDP-1 and SDP-$2'$ as special cases. The consistency results will go beyond the $\bsymBM(p,q)$ model originally used in deriving them.   Consider a general balanced block model, denoted as $\bBM(\Qm) = \bBM_m(\Qm)$ with block size $m$, and probability matrix $\Qm \in [0,1]^{K \times K}$. Note the relationship $ n = m K$
between the number of nodes $n$, the block size $m$, and the number of blocks $K$. 
For notational consistency, we will denote diagonal and off-diagonal entries of $\Qm$ differently,
\begin{align}\label{eq:Psi:element:naming}
  p_k := \Qm_{kk}, \quad q_{k \ell} := \Qm_{k \ell}, \quad k \neq \ell.
\end{align}
The balanced planted partition model $\bsymBM(p,q) = \bsymBM_{m,K}(p,q)$ is a special case of $ \bBM_m(\Qm)$ where $p_k = p$ and $q_{k \ell} = q$ for all $k$, $l$. 
 
We start with defining two notions of assortativity that will be key in our results.  Let
\begin{align}\label{eq:qs:def}
  \qs_k := \max_{r=k,s\neq k} q_{rs} = \max_{r\neq k,s= k} q_{rs}.
\end{align}
\begin{defn}[Strong and weak assortativity]\label{def:strong:weak:assort}
  Consider the balanced block model $\bBM_m(\Qm)$ determined by~\eqref{eq:Psi:element:naming}.
  \begin{itemize}
    \item  The model is strongly assortative (SA) if   $\min_{k} p_k \, > \, \max_k \qs_k$.
    \item  The model is weakly assortative (WA) if   $p_k > \qs_k$ for all $k$.
  \end{itemize}
\end{defn}
An alternative way to state strong assortativity is $\min_{k} p_k >  \max_{(k,\ell): \; k \neq \ell} q_{k,\ell}$.   Strong assortativity implies weak assortativity. See~\eqref{eq:K4:Psi:exam} for an example where weak assortativity holds but not the strong one.

%
These definitions apply to general block models since they are defined only in terms of the edge probability matrix $\Qm$. 
We also define a partial order among balanced block models, which reflects the hardness of recovering the underlying cluster matrix $X$. 
\begin{defn}[Strong assortativity (SA) ordering]\label{def:BM:ordering}
  The collection $\{\bBM_m(\Qm): \Psi \in [0,1]^{K\times K} \cap \Symm{K} \}$ is partially ordered by
  \begin{align}\label{eq:BM:ordering:def}
    \bBM_m(\Qmt) \ge \bBM_m(\Qm) \quad \iff \quad \pt_k \ge p_k, \;\; \qt_{k \ell} \le q_{k \ell}, \;\;\forall k \neq \ell.
  \end{align}
  This ordering or the one induced on matrices in $[0,1]^{K \times K} \cap \Symm{K}$  is referred to as SA-ordering.
\end{defn}
%
Intuitively, for assortative models $\bBM_m(\Qmt) \ge \bBM_m(\Qm) $ implies that $\bBM_m(\Qmt)$ is easier than $\bBM_m(\Qm)$ for cluster recovery.   This will be made precise  in Corollary~\ref{cor:BM:ordering} in Section~\ref{SEC:SDP:RESPECTS:ORDERING}.   For example, consider a strongly assortative model $\bBM(\Qm)$ where 
\begin{align}\label{eq:pminus:qplus}
  \pminus := \min_{k} p_k \; > \; \max_{(k,\ell): \; k \neq \ell} q_{k,\ell} =: \qplus.
\end{align}
Then, it is easy to see that $ \bBM(\Qm) \ge \bsymBM(p^{-},q^{+})$,  roughly meaning that fitting $\bBM(\Qm)$ is not harder than fitting $\bsymBM(p^{-},q^{+})$.

In order to study consistency, we always condition on the true cluster matrix, which is taken to be $\Xtru := I_K \otimes \onem_m$ without loss of generality.  Let $\Stru := \supp(\Xtru)$ be the index set of non-zero elements of $\Xtru$.   We write $\SDP(A)$ for the solution set of the SDP for input $A$, where 
SDP is any of SDP-1, SDP-2 or SDP-3, which will be clear from the context.   In this notation, $\SDP(A) = \{\Xtru\}$ means that $\Xtru$ is the unique solution of the SDP, in which case we say that the SDP is \emph{strongly consistent} for cluster matrices. 
\begin{rem}
Our notion of consistency here is stronger than what is commonly called strong consistency in the literature~\cite{Zhaoetal2012, Bickel2009, Amini.et.al.2013, Mossel2014}.   Strong consistency for an algorithm that outputs a set of community labels usually means exact recovery of labels up to a permutation of communities, with high probability.  Viewing SDPs as algorithms that output the cluster matrix $\Xh$, here by strong consistency we mean the exact recovery of $\Xtru$, which immediately implies exact label recovery. We note, however, that in some regimes one can recover labels exactly even when the output $\Xh$ of an SDP is not exact. For example, one can run a community detection algorithm on $\Xh$, say spectral clustering;  see Algorithm~\ref{alg:graphon} in Section~\ref{sec:graphon}.   However, if the labels are inferred directly from $\Xh$  (if $\Xh$ corresponds to a graph with $K$ disjoint connected components, then output the labels implied by the components; otherwise, output random labels), then our notion of strong consistency matches the standard one in the literature.   
\end{rem}

\medskip
 A key piece in our results will be the following matrix concentration inequality noted recently by many authors; see, for example~\cite{Lei&Rinaldo2013,Chen2014,Tomozei2011} and the references therein. Results of this type are often based on the refined discretization argument of~\cite{Feige2005}.
\begin{prop}\label{prop:key:adj:concent}
Let $A = (A_{ij}) \in \{0,1\}^{n \times n}$ be a symmetric binary matrix, with independent lower triangle and zero diagonal. There are universal positive constant $(C,C',c,r)$ such that 
if 
\begin{align*}
  \max_{ij}\var(A_{ij}) \le \sigma^2, \quad\text{for}\quad
  n \sigma^2 \ge C' \log n
\end{align*}
 then with probability at least $1 - c\, n^{-r}$,
\begin{align*}
  \mnorm{ A - \ex A } \le C \sigma \sqrt{n}.
\end{align*}
\end{prop}
In what follows $(C,C',c,r)$ will always refer to the constants in this proposition.
Our first result establishes consistency of SDP-$2'$ for the balanced planted partition models. We will work with two rescaled version of $p$, namely 
\begin{align}\label{eq:p:defs}
\pb := p m= p \frac{n}{K},  \quad   \pub := \frac{\pb}{\log n} , 
\end{align}
and similarly for $\qub$, $\qb$ and $q$.

\begin{thm}[Consistency of SDP-$2'$]\label{thm:consist:SDP2}
  Let  $A$ be drawn from $\bsymBM_{m,K}(p,q)$.
  For any $c_1,c_2 > 0$, let $C_1 := C' \vee \frac{4}{9}(c_1+1)$ and $C_2 := C + (\sqrt{4(c_1+1)} \vee 3 \sqrt{4 (c_2+1)})$. Assume $\pub \ge C_1 $. Then, if 
  \begin{align}\label{eq:SDP2:consist:cond}
  \pub - \qub > C_2 ( \sqrt{\pub} + \sqrt{  \qub K}),
  \end{align}
  SDP-$2'$ is strongly consistent with probability at least $1 - c(Km^{-r} + n^{-r}) -n^{-c_1} - 2m^{-1}n^{-c_2}$.
\end{thm}

As a consequence, we get consistency for a strongly assortative block model. More precisely, Theorem~\ref{thm:consist:SDP2} combined with Corollary~\ref{cor:BM:ordering} in Section~\ref{SEC:SDP:RESPECTS:ORDERING} gives the following:
\begin{cor}[Consistency of SDP-$2'$ for the strongly assortative case]\label{cor:SDP2:consist:strong:assort}
  Let $A$ be drawn from a strongly assortative $\bBM_m(\Qm)$. Then, the conclusion of Theorem~\ref{thm:consist:SDP2} holds with $(p,q)$ replaced with $(\pminus,\qplus)$ as defined in~\eqref{eq:pminus:qplus}.
\end{cor}

Note that Theorem~\ref{thm:consist:SDP2} and its corollary automatically apply to SDP-1, because it is a tighter relaxation of the MLE than SDP-$2'$. However, SDP-1 succeeds for the much larger class of \emph{weakly assortative} block models, as reflected in our main result, Theorem~\ref{thm:consist:SDP1} below. Recall the notation $\qs_k$ defined in~\eqref{eq:qs:def} and write $\qsmax = \max_k \qs_k = \max_{k \neq \ell} q_{k\ell}$. The scaled versions $\qsub_k$, $\qsb_k$, and  $\qsubmax, \qsbmax$ are defined based on $\qs_k$ and $\qsmax$ as in~\eqref{eq:p:defs}.
\begin{thm} [Consistency of SDP-1]\label{thm:consist:SDP1}
  Let $A$ be drawn from a weakly assortative $\bBM_m(\Qm)$.
  For any $c_1,c_2 > 0$, let $C_1 := C' \vee \frac{4}{9}(c_1+1)$ and $C_2 := (\sqrt{4(c_1+1)} + C) \vee (6 \sqrt{2 (c_2+1)})$. Assume $\min_k \pub_k \ge C_1 $. Then, if 
  \begin{align}\label{eq:SDP1:consist:cond}
  \min_k \Big[  (\pub_k - \qsub_k) - C_2 \big(\sqrt{\pub_k } + \sqrt{\qsub_k }\big) 
    \Big] &> C \sqrt{\frac{\qsubmax K}{\log n}},
  \end{align}
  SDP-1 is strongly consistent with probability at least $1 - c(Km^{-r} + n^{-r}) -n^{-c_1} - 2m^{-1}n^{-c_2}$.
  \end{thm}
 
 Note that for any weakly assortative $\bBM_m(\Qm)$ with fixed $K$ and constant entries of $\Qm$, condition~\eqref{eq:SDP1:consist:cond} holds for large $n$ and hence SDP-1 is strongly consistent. We show in Section~\ref{sec:failure:SDP-2} that SDP-$2'$ fails in general outside the class of strongly assortative block models.


\begin{rem}\label{rem:SDP2:result:strength}
Our result for SDP-2 can be slightly strengthened by stating~\eqref{eq:SDP2:consist:cond} as in~\eqref{eq:SDP1:consist:cond} with $\pub_k \equiv \pub$ and $\qub_{k,\ell} \equiv \qub$. This gives a better threshold in the case where $\qub \to \infty$ but $\qub/\log n \to 0$.
\end{rem}

\begin{rem}
  One can define strong and weak disassortativity by replacing $p_k$ and $q_{k\ell}$ with $-p_k$ and $-q_{k \ell}$, respectively, in Definition~\ref{def:BM:ordering}. The results then hold if one applies the SDPs to $-A$ in the disassortative case.
\end{rem}

\begin{rem} 
 Another way to express conditions of Theorem~\ref{thm:consist:SDP2} is in terms of the alternative parametrization $(\avgd,\beta)$ where $\avgd := \pb + (K-1) \qb$ is the expected node degree, and $\beta := q/p = \qb/\pb$ is the out-in-ratio. A slight weakening of condition \eqref{eq:SDP2:consist:cond}, using $(a+b)^2 \le 2(a^2 + b^2)$, gives
  \begin{align}\label{eq:beta:lambda:cond}
    (\pb - \qb)^2 \, \gtrsim\, (\pb + \qb K) \log n \;\;\iff\;\; 
      \avgd \,\gtrsim\, \Big( \frac{1+K \beta}{1-\beta}\Big)^2 \log n.
  \end{align}
  where we have used $\avgd \asymp \pb + K \qb$.  We also need $\pb \gtrsim \log n$ which translates to $\avgd \gtrsim (1+K\beta) \log n$, which is implied by~\eqref{eq:beta:lambda:cond}. In particular, for fixed $\beta$, it is enough to have $\avgd = \Omega(K^2 \log n)$ for SDP-$2'$ (and hence SDP-1) to be strongly consistent.
\end{rem}

The proof of Theorem~\ref{thm:consist:SDP1} appears in Section~\ref{sec:proof:consist:SDP1} with some of the more technical details deferred to the appendices.
The proof of Theorem~\ref{thm:consist:SDP2} is similar and appears in Appendix~\ref{sec:proof:consist:SDP2}.

\subsection{Comparison with other consistency results}
Rigorous results about the phase transition in the so-called reconstruction problem for the 2-block balanced PP model, i.e., recovering a labeling positively correlated with the truth, in the sparse regime where $\avgd = O(1)$, have appeared in~\cite{Mossel.et.al.2012,Mossel2013,Massoulie2013} after originally conjectured by~\cite{Decelle2011}. For $K=2$, the problem of exact recovery in PP has recently been studied in~\cite{Abbe2014,Mossel2014}, where the exact recovery threshold is obtained when $\avgd = \Omega(\log n)$, the minimal degree growth required for exact recovery. \cite{Mossel2014} also discusses exact thresholds for weak consistency, i.e., fraction of misclassified labels going to zero. \cite{Abbe2014} also analyzed the MAXCUT SDP showing a consistency threshold within constant factor of the optimal. Since the earlier draft of our manuscript, more refined analyses of SDPs for balanced PP have appeared in~\cite{Hajek2014,Hajek2015}, as well as~\cite{Abbe2015} which obtains the exact threshold for a general SBM, by a two-stage approach with no SDP involved. In~\cite{Hajek2014}, the argument in~\cite{Abbe2014} is refined to show that MAXCUT SDP achieves the threshold of exact recovery with optimal constant, for the case $K=2$. In~\cite{Hajek2015}, the analysis is extended to the general $K$, for an SDP which interestingly is equivalent to what we have called SDP-1, showing that it achieves optimal exact recovery threshold. This threshold is equivalent, up to constants, to that obtained in~\cite{Chen2014}, and hence to~\eqref{eq:beta:lambda:cond} as will be discussed below. The analysis in \cite{Hajek2015} also provides the exact constant and an extension to the unbalanced case. 

For the PP model with general $K$,~\cite{Chen2014} provides sufficient conditions for strong consistency of their nuclear norm relaxation of the MLE. These conditions automatically apply to SDP-$2'$ and SDP-1 since they are tighter relaxations. More
precisely, their model, in the zero outliers case, coincides with $\bsymBM(p,q)$ and their sufficient conditions translate to $(p-q)^2 (n/K)^2 \gtrsim p (n/K)\log n + q n$. A slightly weaker version, obtained by replacing $q$ with $q \log n$, reads $(\pb - \qb)^2 \, \gtrsim\, (\pb + \qb K) \log n $ which is the one we have obtained in~\eqref{eq:beta:lambda:cond} as a consequence of Theorem~\ref{thm:consist:SDP2}.  The stronger version also follows from our proof -- see Remark~\ref{rem:SDP2:result:strength}.  Interestingly, exactly the same condition \eqref{eq:beta:lambda:cond}, is established in~\cite{Cai2014} for
SDP-3, when specialized to $\bsymBM(p,q)$, the case with zero outliers. In other words, results of the form predicted by Theorem~\ref{thm:consist:SDP2} already exist for SDP relaxations of the block model, albeit using different proof techniques. On the other hand, we are not aware of any results like Theorem~\ref{thm:consist:SDP1}, which guarantees success of SDP-1 for weakly assortative block models.  
A somewhat different condition
amounting to $(\pb -\qb)^2 \gtrsim m K^2 \pb$ and $np = K\pb \ge \log
n$ is implied by the results of~\cite{Lei&Rinaldo2013} for spectral clustering
based on the adjacency matrix, which we have called eigenvalue
truncation (EVT). We note that the dependence on $K$ is worse than in the SDP results, among other things.
 This is corroborated empirically in Section~\ref{sec:numeric}, which shows SDPs outperform EVT for larger values of $K$.

We should point out that there is a somewhat parallel line of work regarding relaxations for clustering problems. For example, a variant of SDP-1 (with $\diag(X) = \onev_n$ replaced with $\tr(X) = n$) has been proposed as a relaxation of the $K$-means or normalized $K$-cut problems~\cite{Xing2003,Peng2007}. However, theoretical analysis of SDPs in the clustering context have only recently began. See for example~\cite{Awasthi2014} for a recent analysis, using a probabilistic model of clusters. An earlier line of work reformulates the clustering problem as instances of the planted partition model and analyzes an SDP relaxation for cluster recovery~\cite{Ames2010,Mathieu2010}. The planted $K$-disjoint clique model in~\cite{Ames2010} and the fully random model of~\cite{Mathieu2010} both can be considered as special case of the planted partition model. The analysis in~\cite{Mathieu2010} is in particular interesting for analyzing an SDP with triangle-inequality type constraints and providing approximation bounds relative to the optimal combinatorial solution.

Recently,  a very interesting paper~\cite{Guedon2014} analyzed the performance of SDP relaxations in the sparse regime where $\avgd = O(1)$.   They showed that as long as the feasible region is contained in the so-called Grothendieck set $\{X \succeq 0,\; \diag(X) \le 1\}$, the SDPs can achieve arbitrary accuracy, with high probability, assuming that $(\pb - \qb)^2 / (\pb + \qb K)$ is sufficiently large. These results are complementary to ours and show that all the SDPs in Table~\ref{SDP-table} are capable of approximate recovery in the sparse regime. 


\subsection{Some useful general results}\label{sec:aux:results}
Here we collect some general observations on solutions of SDPs which will be useful in proving Theorems~\ref{thm:consist:SDP2} and~\ref{thm:consist:SDP1}.  Let $\subb{k}$ be the indices of the $k$th community. We have $|S_k| = m$. Let  $X_{\subb{k}\subb{j}}$ be the submatrix of $X$ on indices $\subb{k} \times \subb{j}$, and $X_{\subb{k}}:= X_{\subb{k}\subb{k}}$. Let $\onev_{\subb{k}} \in \reals^n$ be the indicator vector of $\subb{k}$, equal to one on $S_k$ and zero elsewhere. $\onem_{\Stru} \in \{0,1\}^{n \times n}$ denotes the indicator matrix of $\Stru \subset [n]^2$. Let $\basisv{k}{n}$, or simply $e_k$, be $k$th unit vector of $\reals^n$. Let $\Span\{\onev_{\subb{k}}\}$ and $\Span\{\onev_{\subb{k}}\}^\perp$  denote the subspace spanned by $\{\onev_{\subb{2}},\onev_{\subb{2}}, \dots, \onev_{\subb{K}}\}$ and its orthogonal complement.
Let $\dg{}{k} \in \reals^n$ be the vector of node degrees relative to the subgraph induced by $\subb{k}$, $\dg{}{k} = A \onev_{\subb{k}} = A_{\subb{k}} \onev_m$. Note that $[\dg{}{k}]_{\subb{k}} \in \reals^m$ is the subvector of $\dg{}{k}$ on indices $\subb{k}$. 
    


\subsubsection{SDPs respect SA-ordering}\label{SEC:SDP:RESPECTS:ORDERING}
The following lemma formalizes an intuitive fact on how SDPs interact with the SA-ordering of Definition~\ref{def:BM:ordering}.  The proof is given in Appendix~\ref{sec:proof:SDP:respects:ordering}. 
\begin{lem}\label{lem:deterministic:nesting}
  Let $\At \in \Symm{n}$ be obtained from $A$ by setting some elements off $\Stru$ to zero and some elements on $\Stru$ to one. Then, for either of SDP-1 or \SDPp{2},
  \begin{align*}
    \SDP(A) = \{\Xtru\} \quad \implies \quad \SDP(\At) = \{\Xtru\}.
  \end{align*}
\end{lem}

The lemma generalizes to any optimization problem that maximizes $X \mapsto \ip{A,X}$, and has its feasible region included in $\{X: 0 \le X \le 1\}$.
An immediate consequence is the following probabilistic version for \SBM s, stated conditionally on the true cluster matrix $\Xtru$. 
\begin{cor}\label{cor:BM:ordering}
  Assume $\bBM_m(\Qmt) \ge \bBM_m(\Qm)$, and let $\At \sim \bBM_m(\Qmt)$ and $A \sim \bBM_m(\Qm)$. Then, for either of SDP-1 or \SDPp{2},
  \begin{align*}
    \pr\Big( \SDP(\At) = \{\Xtru \} \Big ) \ge \pr\Big( \SDP(A) = \{\Xtru \} \Big ).
  \end{align*}
\end{cor}

This corollary allows us to transfer consistency results for SDPs regarding a particular \SBM\  to any \SBM\  that dominates it. It also allows us to inflate off-diagonal entries of $\Qm$ for a general $\bBM(\Qm)$ without loss of generality. More precisely, we will assume in the course of the proof that off-diagonal entries of $\Qm$ satisfy certain lower bounds to ensure concentration. These lower bounds can then be safely discarded at the end by Corollary~\ref{cor:BM:ordering}.


\subsubsection{Optimality conditions}\label{sec:optim:cond}
Consider the following general SDP:
\begin{align}\label{eq:gen:SDP}
  \def\arraystretch{1.3}
  \begin{array}{ll}
    \max & \ip{A,X} \\
    \text{s.t.} &  \diag(X) = \onev_n, \; \Lc_2(X) = b_2 \\
    & X \succeq 0, \, X \ge 0 
  \end{array}
\end{align}
where $\Lc_2$ is a linear map from $\Symm{n}$ to $\reals^s$ for some integer $s$, and $b_2 \in \reals^s$. This is a blueprint for both SDP-1 and  SDP-$2'$. Let $\Lc_1(X) := \diag(X)$ and $b_1 = \onev_n$. Then, 
  $\Lc(X) := (\Lc_1(X),\Lc_2(X)) = (b_1,b_2) =: b$
summarizes the linear constraints for the SDP. The dual problem is
\begin{align*}
  \def\arraystretch{1.3}
  \begin{array}{ll}
    \min & \ip{\mu,b_2} + \sum_i \nu_i \\
    \text{s.t.} 
    & \Lc_2^*(\mu) + \diag^*(\nu) \succeq A + \Gamma, \quad \Gamma \ge 0,
  \end{array}
\end{align*}
where $\mu \in \reals^s$, $\nu \in \reals^n$ and $\Gamma \in \Symm{n}$, and the minimization is over the triple $\dualtri$ of dual variables. $\Lc_2^*$ is the adjoint of $\Lc_2$ and $\diag^*$ is the adjoint of $\diag$.
Letting 
\begin{align}\label{eq:Lambda:def:gen}
\Lambda := \Lambda(\mu,\nu,\Gamma)  := \Lc_2^*(\mu) + \diag^*(\nu) -A -\Gamma,
\end{align}
the (KKT) optimality conditions are 
\begin{align*}
  \begin{array}{lll}
    \text{Primal Feas. } & X \succeq 0, \; X \ge 0,\; \Lc(X) = b, \\
    \text{Dual Feas.} & \Lambda \succeq 0, \; \Gamma \ge 0,\\
    \text{Comp. Slackness (a)} & \Gamma_{ij} X_{ij} = 0, \quad \forall i,j, & \text{(CSa)}\\
    \text{Comp. Slackness (b)} &\ip{\Lambda,X} = 0. & \text{(CSb)}
  \end{array}
\end{align*}
Another way to state (CSa) is to write $\Gamma \circ X = 0$ where $\circ$ denotes the Schur (element-wise) product of matrices. 

The primal-dual witness approach that we will use in the proofs is based on finding a pair of primal and dual solutions that simultaneously satisfy the KKT conditions. The pair then witnesses strong duality between the primal and dual problems implying that it is an optimal pair.


\subsubsection{Sufficient conditions for exact recovery}

We would like to obtain sufficient conditions under which the true cluster matrix $\Xtru =  I_K \otimes \onem_m$
is the unique solution of the primal SDP. Complementary slackness (a), or (CSa), implies that we need $\Gamma_{\subb{k} } = 0$ for all $k$, while we are free to choose $\Gamma_{\subb{k}\subb{j}}$ for $j\neq k$,  using the submatrix notation.

Since both $\Xtru$ and $\Lambda$ are PSD, (CSb) is equivalent to $\Lambda \Xtru = 0$, which is in turn equivalent to $\rangeS(\Xtru) \subset \ker(\Lambda)$. Note that $\Xtru$ has $K$ nonzero eigenvalues, all equal to $m$, corresponding to eigenvectors $\{\onev_{\subb{k}}\}_{k=1}^K$, where $\onev_{\subb{k}} \in \reals^n$ is the indicator vector of $\subb{k}$. Hence, $\rangeS(\Xtru) = \Span\{\onev_{\subb{k}}\}$, and (CSb) for $\Xtru$ is equivalent to 
\begin{align*}
  \Span\{\onev_{\subb{k}}\} \subset \ker(\Lambda)
\end{align*}
The following lemma, proved in Appendix~\ref{sec:proof:suff:cond:exact:recov}, gives conditions for $\Xtru$ to be the unique optimal solution.

\begin{lem}\label{LEM:SUFF:COND:EXACT:RECOV}
  Assume that $\Gamma$ is dual feasible (i.e., $\Gamma \ge 0$), and for some $\mu \in \reals$ and $\nu \in \reals^n$,
  \begin{itemize}
    \item[(A1)] $\ker\big(\Lambda(\mu,\nu,\Gamma)\big) = \Span\{\onev_{\subb{k}}\}$, \; and \; $\Lambda(\mu,\nu,\Gamma) \succeq 0$,
    \item[(A2)] $\Gamma_{\subb{k}} = 0,\;\forall k$, 
    \item[(A3)] Each $\Gamma_{\subb{k}\subb{\ell}}$, $k \neq \ell$ has at least one nonzero element.
  \end{itemize}
  Then $\Xtru$ is the unique primal optimal solution, and $\dualtri$ is dual optimal.
\end{lem}

Note that condition (A1) is satisfied if for some $\eps > 0$,
\begin{align}
  \label{eq:A1:equiv:1}  \Lambda \onev_{\subb{k}} &= 0, \quad \forall k\\
  \label{eq:A1:equiv:2} u^T \Lambda u &\ge \eps \|u\|_2^2, \quad \forall u \in \Span\{\onev_{\subb{k}}\}^\perp.
\end{align}


\subsection{Proof of Theorem~\ref{thm:consist:SDP1}: primal-dual witness for SDP-1}\label{sec:proof:consist:SDP1}
Let $\Star_i = \bv{i} \onev_n^T + \onev_n \bv{i}^T \in \Symm{n}$ where $\bv{i} = \basisv{i}{n}$ is the $i$th standard basis vector in $\reals^n$. We note that $\ip{X,\Star_i} = \tr(X\Star_i)= 2(X \onev_n)_i$. Thus, SDP-1 is an instance of~\eqref{eq:gen:SDP}, with $\Lc_2(X) = \big( \ip{X,\Star_i} \big)_{i=1}^n$  and $b_2 = 2 m \onev_n$.
The corresponding adjoint operator is
  $\Lc_2^*(\mu) = \sum_{i=1}^n \mu_i \Star_i  = \mu \onev_n^T + \onev_n \mu^T$.
Thus (cf.~\eqref{eq:Lambda:def:gen}),
\begin{align}
  \Lambda = \Lambda(\mu,\nu,\Gamma) = (\mu \onev_n^T + \onev_n \mu^T) + \diag^*(\nu) -A - \Gamma.
\end{align}
The following summarizes our primal-dual construction in this case:
\begin{align}
\nu_{\subb{k}} &= [\dg{}{k}]_{\subb{k}} - \mubar_k m \onev_m, \label{eq:nu:def:SDP1} 
\quad \mu_{\subb{k}} := \frac12 \mubar_k \onev_m, \\ 
\begin{split}
  \label{eq:Gamma:def:SDP1}
  \Gamma_{\subb{k}} &:= 0,  \\
  \Gamma_{\subb{k}\subb{\ell}} &:= 
     \mu_{\subb{k}} \onev_m^T + \onev_m \mu_{\subb{\ell}}^T  + \projonep A_{\subb{k}\subb{\ell}} \projonep - A_{\subb{k}\subb{\ell}},  \\
     &= \frac12(\mubar_k + \mubar_\ell) \onem_m   + \projonep A_{\subb{k}\subb{\ell}} \projonep - A_{\subb{k}\subb{\ell}},
      \quad k \neq \ell
\end{split}
\end{align}
for some numbers $\{\mubar_k\}_{k=1}^K$ to be determined later. Note that $\mu$ is chosen to be constant over blocks, but these constants can vary between blocks. We have the following analogue of Lemma~\ref{LEM:VALID:GAMMA:SDP2}. Recall that $\onem_{\Stru^c}$ is the indicator matrix of $\Stru^c$ where $\Stru$ is the support of $\Xtru$.

\begin{lem}\label{LEM:VALID:GAMMA:SDP1}
  Let $(\mu,\nu,\Gamma)$ be as defined in~\crefrange{eq:nu:def:SDP1}{eq:Gamma:def:SDP1}.  Then, $\Gamma$ verifies (A2) and~\eqref{eq:A1:equiv:1} holds. In addition, 
  \begin{itemize}
  \setlength\itemsep{0em}
    \item[(a)] $\Gamma$ is dual feasible, i.e. $\Gamma \ge 0$, if for all $i\in \subb{k}, j \in \subb{\ell}, \ell \neq k$, 
    \begin{align}\label{eq:mu:low:bound:SDP1}
       \frac12(\mubar_k + \mubar_\ell) m \ge \dg{i}{\ell}  + \dg{j}{k} - \dav{k}\ell.
    \end{align}
    and satisfies (A3) if at least one inequality is strict for each pair $k \neq \ell$.
    \item[(b)] $\Gamma$ verifies~\eqref{eq:A1:equiv:2} if for $\rho_k := \min_{i \in \subb{k}} \dg{i}{k} /m$,
    \begin{equation}\label{eq:Delta:up:bound:SDP1}
      \min_k \big[ (\rho_k - \mubar_k) m - \mnorm{\Delta_k}\big] > \mnorm{\onem_{\Stru^c} \circ \Delta}.
    \end{equation}
  \end{itemize}
\end{lem}
This lemma amounts to a set of deterministic conditions for the success of SDP-1.   To complete the proof of Theorem~\ref{thm:consist:SDP1}, we develop a probabilistic analogue by choosing $\mubar_k \approx \qsb_k$ and using the key inequality $ \qb_{k \ell} \le \frac12 (\qsb_k + \qsb_\ell)$. See Appendix~\ref{sec:prob:cond:bm} for details.


\section{Failure of \SDPp{2} in the absence of strong assortativity}
\label{sec:failure:SDP-2}
We now show that strong assortativity is a necessary condition for exact recovery in SDP-2$'$. For this purpose, it is enough to focus on the noiseless case, i.e., when the input to the SDP is the mean matrix of the block model. If SDP-2$'$ fails on exact recovery of the true population mean, there is no hope of recovering its noisy version, i.e., the adjacency matrix. The following result is deterministic and non-asymptotic. In particular, it holds without any constraints on the expected degrees (besides those imposed by assortativity assumptions).  We will state it in a slightly more general form than is needed here, including the case of general block sizes.   Keeping consistency with earlier notation, we let $\onem_{S_k S_\ell} \in \{0,1\}^{n \times n}$ be the indicator matrix of the set $S_k \times S_\ell$, and $\onem_{S_k} := \onem_{S_k S_k}$. Similalry, $I_{S_k}$ is the ${n \times n}$ identity matrix with elements outside $S_k \times S_k$ set to zero, i.e., $\onem_{S_k S_{\ell}}$ is not a submatrix of $\onem_n$, but a masked version of it.

\begin{prop}\label{prop:failure:SDP-2}
	Let $\ex[A]$ be the mean matrix of a weakly assorative block model. Assume that the blocks are indexed by $S_k \subset [n]$ where $|S_k| = \bs_k$, for $k=1,\dots,K$.   For some $I \subset [n]  = \{1, \dots, n\}$, to be determined, consider a solution of the form
	\begin{align}\label{eq:blk:diag:solution:1}
		X = \sum_{k \in I} \sum_{\ell \in I} \alpha_{k \ell} \onem_{S_k S_\ell}
		+ \sum_{ k \notin I} \big[ \beta_k \onem_{S_k} +  (1-\beta_k) I_{S_k} \big], \quad \alpha_{k \ell} = \alpha_{\ell k},
	\end{align} 
	with $\alpha_{kk} = 1, \; k \in I$ and $\beta_k \in [0,1)$ for $k \notin I$. 
	%
	Then the following holds:
	\begin{itemize}\itemsep=1.5ex
		\item[(a)]  Assume that $\argmax_{k \neq \ell} q_{k\ell} = \{(k_0,\ell_0)\}$ and let $I := \{ k:\; p_k \ge q_{k_0 \ell_0}\}$. Furthermore, let $m := \min_k \bs_k$,~$\xi_k := n_k/m$ and
		\begin{align}\label{eq:alpha:k0:ell0}
			\alpha^*_{k_0 \ell_0} := \frac{1}{2 \xi_{k_0} \xi_{\ell_0}} \Big[
			\Big(1- \frac1m\Big) \sum_{k \notin I} \xi_k - \sum_{k \in I} \xi_k(\xi_k-1) \Big].
		\end{align}
		If $\alpha^*_{k_0 \ell_0} \in [0,1]$, then SDP-2$'$, applied with $A = \ex[A]$ and $m = \min_k \bs_k$ has~\eqref{eq:blk:diag:solution:1} as solution,  with $\alpha_{k \ell} = \alpha_{k_0 \ell_0}^* 1\{ \{k,\ell\} = \{k_0,\ell_0\}\}$ and $\beta_k = 0$ for all $k \notin I$.

		\item[(b)] Assume that the given block model is balanced and let $I^c := [K] \setminus I = \{ k:\; p_k < q_{k_0 \ell_0}\}$ where $I$ and $(k_0,\ell_0)$ are defined in part~(a). If $|I^c| \le 2$, then the conclusion of part~(a) holds with $\alpha^*_{k_0 \ell_0} = \frac12(1-1/m) | I^c|$.

		\item[(c)] Assume that the given block model is balanced and weakly but not strongly assortative. Let $\text{SDP}_{sol}(\cdot)$ be the solution set of SDP-2$'$. Then, $\text{SDP}_{sol}(\ex [A]) \neq \{X_0\}$.

	\end{itemize}
\end{prop}

We note that part~(c) establishes the failure of SDP-2$'$ once strong assortativity is violated. We prove the proposition in Appendix~\ref{sec:proof:prop:failure:SDP-2}. The conclusions of both parts~(a) and~(b) hold even in the strongly assortative case. However, in that case, the set $I^c$ will be empty and the conditions in part~(a) cannot be met, whereas part~(b) gives the expected result of $X = X_0$. The interesting case occurs when strong assorativity is violated, which gives a nonempty set $I^c$. Since $\beta_k = 0$ for $k \in I^c$, this shows that SDP-2$'$ fails to recover those blocks. The condition $|I^c| \le 2$, in the balanced case, might seem restrictive, but it is enough for our purpose of establishing part~(c). In general, i.e., with no assumption on $|I^c|$, SDP-2$'$ still misses the blocks violating strong associativity, though the nonzero-block portion of $X$, namely, $(\alpha_{k \ell})_{k,\ell \in I}$ takes a more complicated form. In parts~(a) and (b), at most one non-diagonal element of $(\alpha_{k \ell})$ is non-zero whereas in general several such elements will be nonzero. These ideas are illustrated in Figure~\ref{fig:SDP13:and:everythin:else}, with more detailed discussion in Appendix~\ref{prop:sdp13}, in particular, with an application of part~(a), in the unbalanced case with $|I^c| > 2$.



\section{Extensions to the unbalanced case}
\label{sec:unbalanced:extension}
Let us discuss how our results can be extended to the unbalanced case. Recall that, in general, $\nb
= (\bs_1,\dots,\bs_K)$ denotes the vector of block sizes. One could argue that as long as $(\min_k n_k) / n\ge C$ for some constant $C > 0$, i.e., $\{n_k/n\}_k$ is bounded away from zero, the problem of block model recovery is not inherently more difficult than that of the balanced case. To simplify our discussion, we focus on the noiseless case from which the results can be extended to the aforementioned bounded block-size regime.
We will show that in a W.A. block model, SDP-1 applied with $m = \min_k n_k$ recovers all the blocks, albeit some imperfectly. We also consider the following mixture of SDP-1 and SDP-3, which we will call \textbf{SDP-13},
\begin{align}\label{eq:sdp:13}
	\begin{split}
		\renewcommand{\arraystretch}{1.3}
			\begin{array}{ll}
			\max\limits_{X} & \ip{A,X} - \mu \ip{\onem_n,X}, \\
			\text{s.t.} & \diag(X) = \onev_n, \;\; X \onev_n \ge m \onev_n, \\
			&X \succeq 0, \; X \ge 0,
		\end{array}
	\end{split}
\end{align}
and show that when applied with $m \le \min_k n_k$ and appropriate choice of $\mu$, it too recovers the blocks with improvements over SDP-1.
Without loss of generality, let us sort the blocks so that $p_1 \ge p_2 \ge \dots \ge p_K$.
\begin{prop}\label{prop:sdp13}
	Let $\ex[A]$ be the mean matrix of a weakly assorative block model, with blocks indexed by $S_k \subset[n]$ where $|S_k| = \bs_k$, for $k=1,\dots,K$.
	Consider a solution of the form
	\begin{align}\label{eq:blk:diag:solution:2}
	X = \sum_{k=1}^K \alpha_k \onem_{S_k} + (1-\alpha_k) I_{S_k}, \quad \alpha_k \in (0,1].
	\end{align} 
	The following holds:
	\begin{itemize}\itemsep=1.5ex
		\item[(a)]  SDP-1 applied with $A = \ex[A]$ and $m \le \min_k \bs_k$ has~\eqref{eq:blk:diag:solution:2} as a solution with $\alpha_k = (m-1)/(\bs_k-1)$ for all $k$. 
		\item[(b)]   Consider $I := \{k: \bs_k > m\}$ and $I_1(k) := \{r \in I: r \le k\}$.
		 Let $J_k := \bigcap_{\,r =1}^k [\qs_r,p_r]$. Define $k_0 := \max\{k:\;  J_k \neq \emptyset\}$. Then, SDP-13, applied with $A = \ex[A]$, $m \le \min_k \bs_k$ and 
		\begin{align*}
			\mu \in J_{k_0} \cap [p_{k_0+1},1] , \quad (p_{K+1} := 0),
		\end{align*}
		 has~\eqref{eq:blk:diag:solution:2} as a solution with
		\begin{align*}
			\alpha_k = 
			\begin{cases}
				1, & k \in I_1(k_0),\\
				(m-1)/(\bs_k-1), &\text{otherwise}.
			\end{cases}
		\end{align*}

	\end{itemize}
\end{prop}

\begin{figure}[t]
	\begin{tabular}{cccccc}
	\includegraphics[width=2cm]{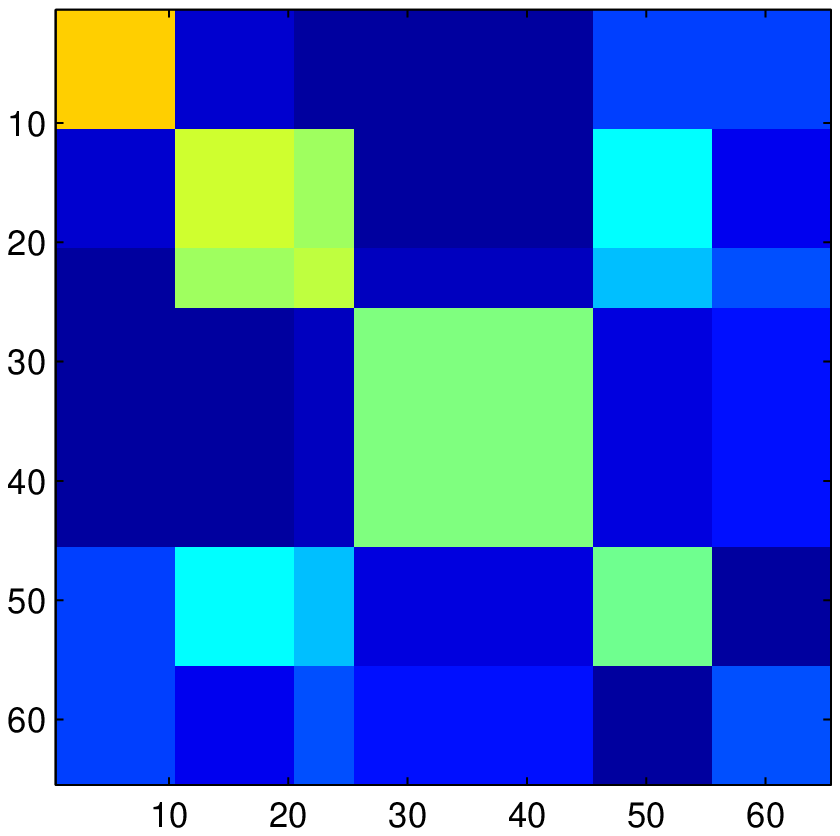} &
	\includegraphics[width=2cm]{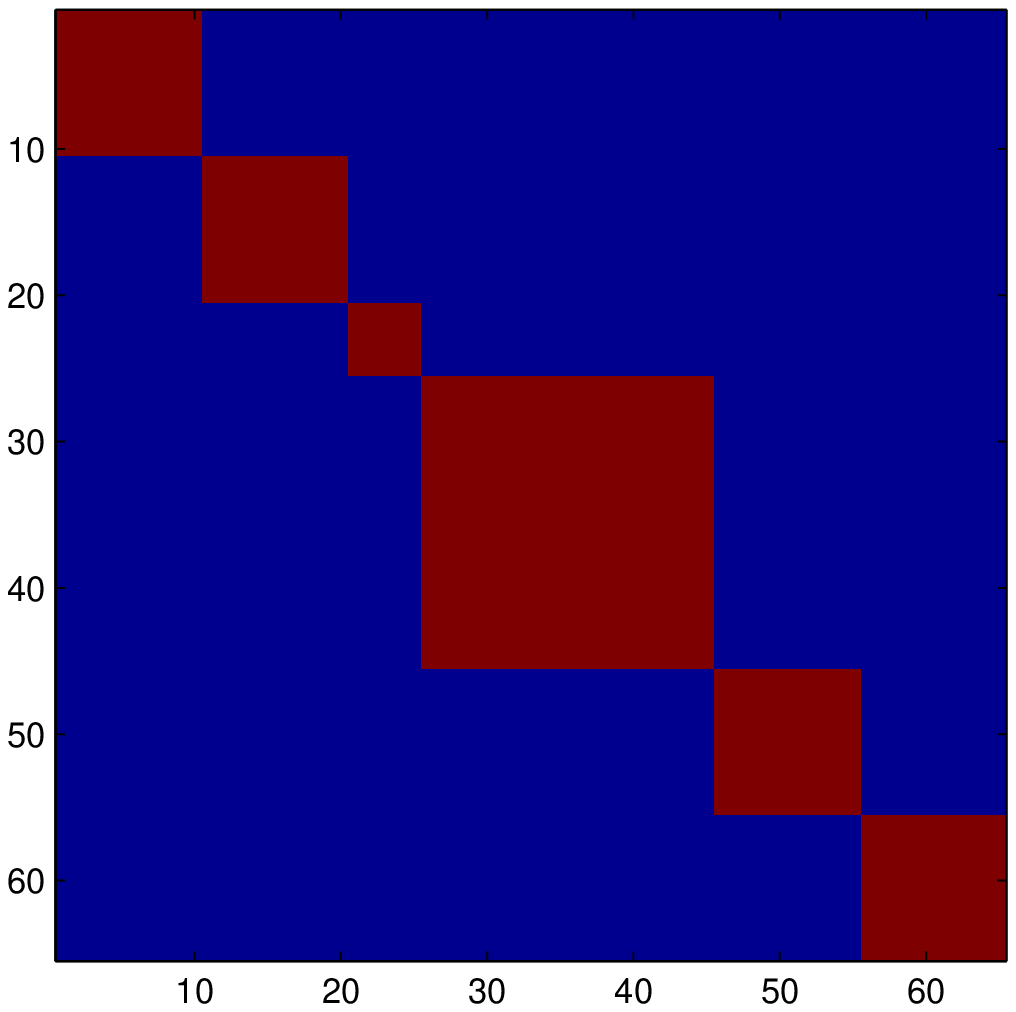} &
	\includegraphics[width=2cm]{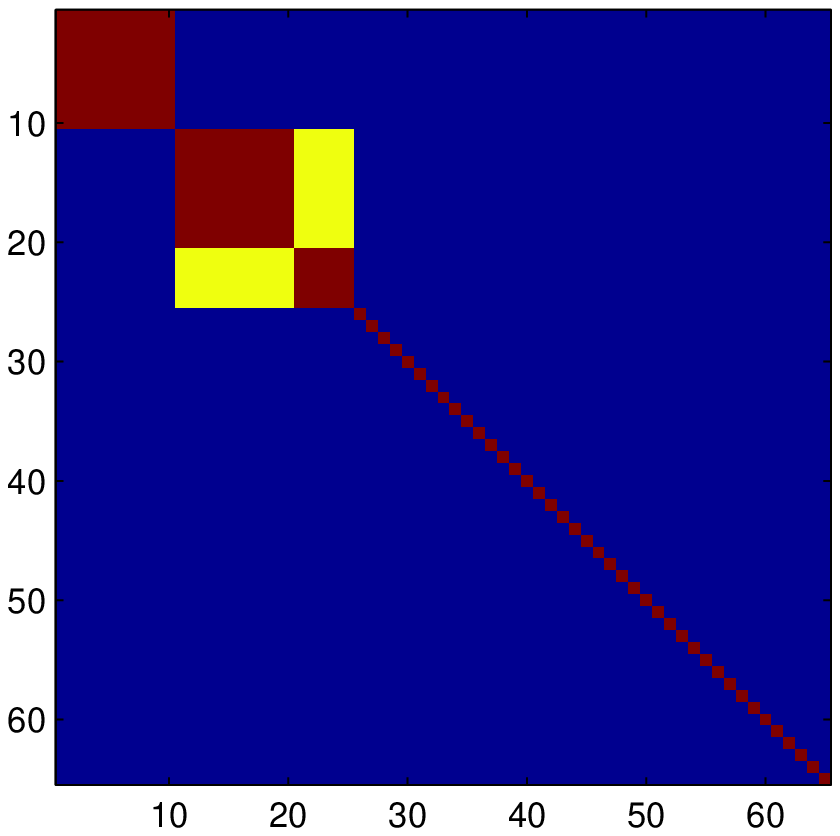} &
	\includegraphics[width=2cm]{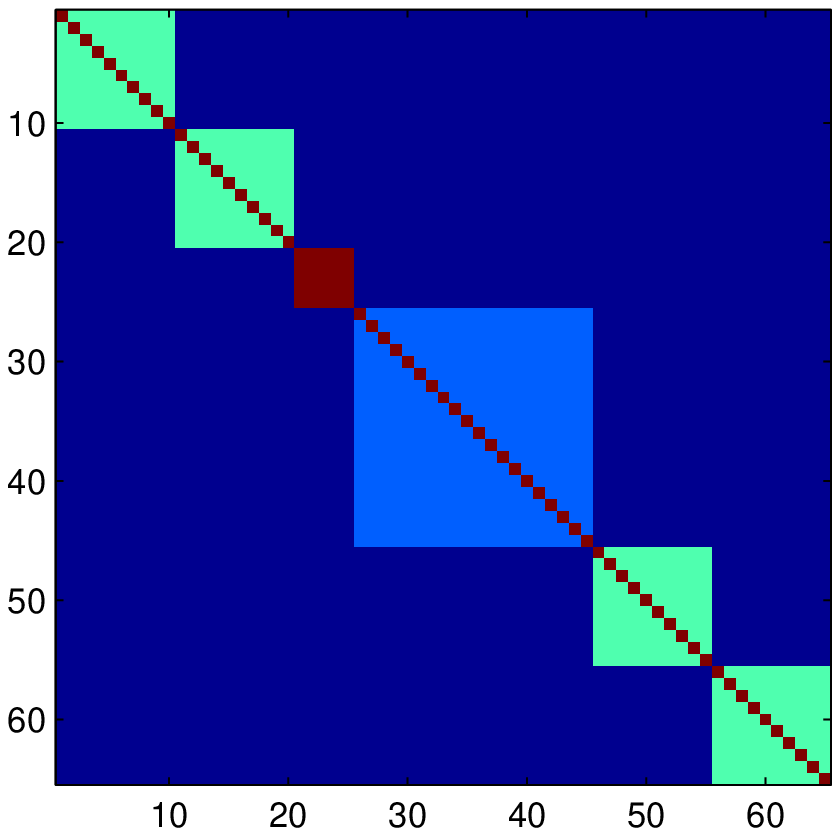} &
	\includegraphics[width=2cm]{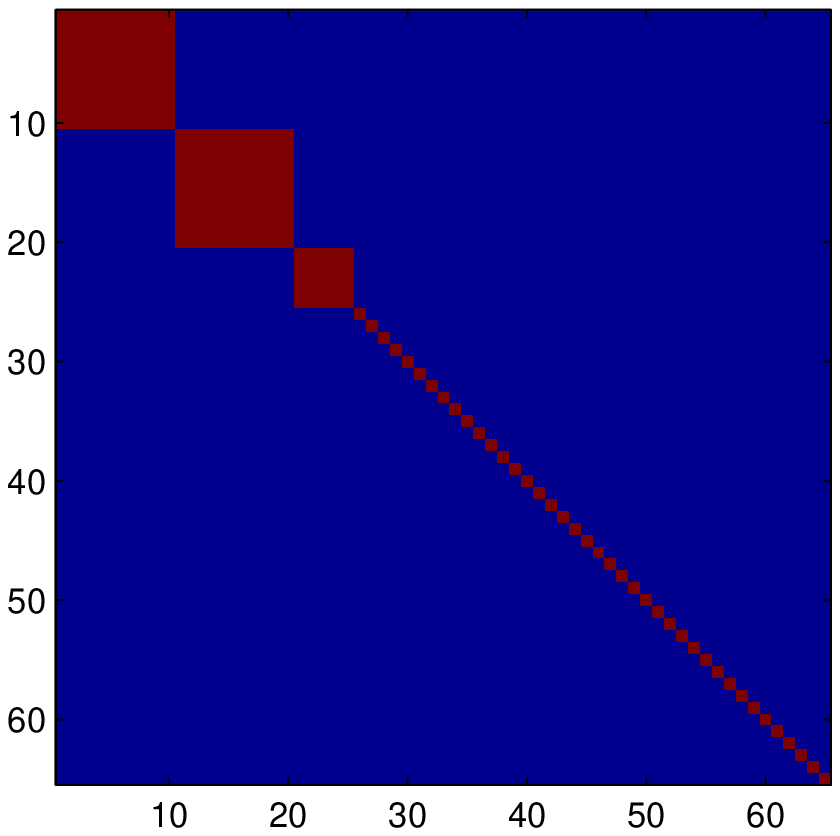} &
	\includegraphics[width=2cm]{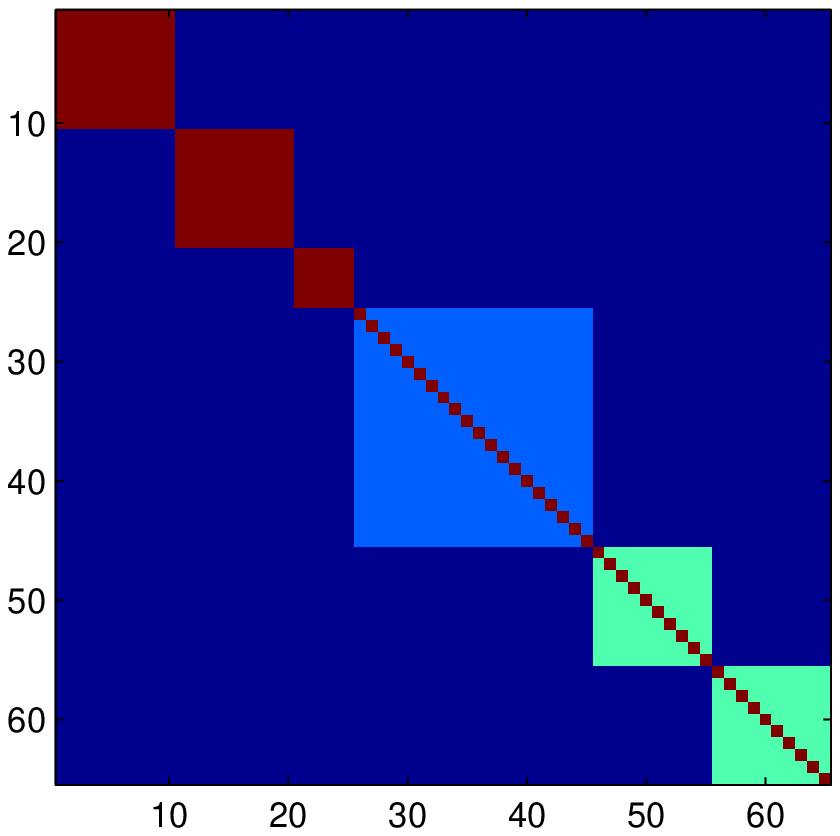}\\
	$\ex[A]$ &‌ Ideal & SDP-2$'$ & SDP-1 & SDP-3 & SDP-13
	\end{tabular}
	
	\centering
	\caption{Illustration of Propositions~\ref{prop:sdp13} and~\ref{prop:failure:SDP-2}. This block model is weakly but not strongly assortative, and has unequal block sizes $\nb = (10,10,5,20,10,10)$. The leftmost column is the population mean, and the rest of the columns are the results of various SDPs, with $m = \min_k n_k$, and equal regularization parameters in the case of SDP-3 and SDP-13 ($\lambda = \mu$). The ideal cluster matrix is also shown for comparison. See Appendix~\ref{sec:SDP13:fig:details} for more details. }	
	 \label{fig:SDP13:and:everythin:else}
\end{figure}

The key difference between the solution presented in Proposition~\ref{prop:sdp13} and that of Proposition~\ref{prop:failure:SDP-2} is that in the former, all $\alpha_k$ are guaranteed to be nonzero, whereas in the latter, $\alpha_k \equiv \beta_k$ corresponding to blocks violating weak assortativity are zero. Let us call the blocks in~\eqref{eq:blk:diag:solution:2} for which $\alpha_k \in (0,1)$, as imperfectly-recovered, while those with $\alpha_k = 1$ as \emph{perfectly recovered}. The result of Proposition~\ref{prop:sdp13} can be summarized as follows: Both SDP-1 and SDP-13, with properly set parameters, recover all the blocks at least imperfectly, while SDP-13 has the potential to recover more blocks perfectly. In particular, we always have $k_0 \ge 1$ in part~(b), implying that SDP-13 recovers at least one more block perfectly relative to SDP-1. In the special case of a strongly assortative block model, we have $\emptyset \neq [\max_k \qs_k, \min_k p_k] \subset \bigcap_{k=1}^K [\qs_r,p_r]$, hence $k_0 = K$ and SDP-13 recovers all the blocks perfectly. It is also interesting to note that both SDP-1 and SDP-13 recover the smallest blocks (i.e., those in $\{k :\bs_k = m\}$) perfectly, when we set $m = \min_k n_k$ (which is the optimal choice if the minimum is known). These observations are illustrated in Figure~\ref{fig:SDP13:and:everythin:else}. The proof of Proposition~\ref{prop:sdp13} appears Appendices, along more details on Figure~\ref{fig:SDP13:and:everythin:else}.


\section{Application to network histograms}\label{sec:graphon}
A balanced block model is ideally suited for computing network histograms as defined by \cite{Olhede2013}, which have been proposed as nonparametric estimators of graphons.  They have been shown to do well empirically and recent results of~\cite[Section~2.4]{Klopp2015} suggest rate-optimality of the balanced models for reasonably sparse graphs;  see also~\cite{Gao2015}.
A graphon is a bivariate symmetric function $f : [0,1]^2 \to [0,1]$. The corresponding network model can be written as $\ex[A | \xi] = f(\xi_i,\xi_j)$ where $\xi = (\xi_1,\dots,\xi_n) \in [0,1]^n$ are (unobserved) latent node positions. Without loss of generality, $(\xi_i)$ can be assumed to be i.i.d. uniform on $[0,1]$. The goal is to recover (a version of) $f$ given $A$. In general, $f$ is identifiable up to a measure-preserving transformation $\sigma$ of $[0,1]$ onto itself, since $f^{\sigma} = f\big(\sigma(\cdot),\sigma(\cdot)\big)$ produces the same network model as $f$.

\newcommand{\ft}{\widetilde{f}}
 Let $\{I_1,I_2,\dots,I_K\}$ be a partition of $[0,1]$ into equal-sized blocks, i.e., $|I_k| = 1/K$ for $k \in [K]$. We associate to each node a label $z_i$, by letting $z_i := k$ if $\xi_i \in I_k$. With some abuse of notation, we identify $z_i$ with an element $(z_{ik})_k$ of $\{0,1\}^K$ as before, and let $Z = (z_{ik})_{ik}$. Then, $M_Z := \ex[A|Z]$ follows a block model as in~\eqref{eq:bm:mean:def} with $[\Psi]_{kk} = |I_k|^{-1} \int_{I_k} f(\xi,\xi) d\xi$ and $[\Psi]_{k \ell} = (|I_k||I_\ell|)^{-1}\int_{I_k} \int_{I_\ell} f(\xi,\xi') d\xi d\xi'$, for $k \neq \ell$. Asymptotically, as $n \to \infty$, this block model is very close to being balanced. It provides an approximation of $f$, via the mapping that sends $\Psi$ to a block constant graphon $\ft$, defined as $\ft(\xi,\xi') = [\Psi]_{k \ell}$ if $\xi \in I_k, \xi' \in I_\ell$. One can show that under regularity assumptions (e.g. smoothness) on $f$, as $K \to \infty$, $\ft$ approximates $f$, for example in the quotient norm: $\inf_\sigma \vnorm{f - \ft^{\sigma}}_{L^2}$. Alternatively, one can consider the mean matrix $M_f := (f(\xi_i,\xi_j))_{ij} \in [0,1]^{n \times n}$ of the graphon model as an empirical version of $f$. In which case, the mean matrix $M_Z$ of the aforementioned block model serves as an approximation to $M_f$, for example in the quotient norm: $\inf_{P} \mnorm{M_f - P M_Z P^T}_F$, where $P$ runs through permutation matrices. This is the approach we take here and, with some abuse of terminology, call $M_f$ the ``graphon''.

  Graphon estimation via block model approximation requires estimating the mean matrix $M_Z$, which is fairly straightforward once we have a good estimate of the cluster matrix $X$. Algorithm~\ref{alg:graphon} details the procedure based on eigenvalue truncation and $K$-means (that is, spectral clustering), leading to estimate $\Mh_{\Zh}$ of $M_f$. We call $\Mh_{\Zh}$ a \emph{network histogram} or a \emph{graphon estimator}, and note that it can be computed from any estimate 
  of $\Xh$. However, in practice, SDP-1 has advantages over other ways of estimating $\Xh$ in this context.  The likelihood-based estimators have no way of enforcing equal number of nodes in each block, whereas our empirical results in Section~\ref{sec:numeric} show that SDP-1 has a high tendency to form equal-sized blocks, more so than SDP-2, making it an ideal choice for histograms.  SDP-3 is not well suited for this task since it does not enforce either a particular number of blocks or a particular block size.  
It is more flexible due to the tuning parameter $\lambda$, but that flexibility is a disadvantage when the goal is to construct a histogram.

%
%

\begin{algorithm}[t]
\caption{Graphon estimation by fitting $\bsymBM(p,q)$}
\label{alg:graphon}
{\small
\begin{algorithmic}[1]
   \Require Estimated cluster matrix $\Xh$, and number of blocks $K$.
   \Ensure Graphon estimator $\Mh_{\Zh}$.
   \State Compute the eigendecomposition  $\Xh = \Uh \widehat{\Lambda} \Uh^T$ and
    set $\Uh^{K} = \Uh(:,\text{1:$K$})$.
    \State Apply \texttt{K-means} to rows of $\Uh^K$ to get a label vector $e \in [K]^n$. 
      \quad Set $\Zh(i,e(i)) = 1$, otherwise 0.
   \State Set $\Qmh_{rk} = \frac{1}{n^2} \sum_{e_i = r, e_j=k} {A_{ij}}$ for $r\neq k$ and 
   $\frac1{n(n-1)} \sum_{e_i=e_j=r} A_{ij}$ otherwise.
   \State Change $\Qmh$ to $Q \Qmh Q^T$ so that its diagonal is decreasing. 
    and update $\Zh$ to $\Zh Q^T$.
   \State Change $\Zh $ to $ P \Zh$ so that corresponding labels are in increasing order.
   \State Set $\Mh_{\Zh} = \Zh \Qmh \Zh^T$.

   
\end{algorithmic}
}
\end{algorithm}

\section{Numerical Results}\label{sec:numeric}\label{sec:numeric:res}

\newcommand{\ppwidth}{.19\linewidth}
\newcommand{\ppscale}{.26}

In this section we present some experimental results comparing SDP-1 with 
SDP-2, SDP-3, and EVT, which amounts to spectral clustering on the
adjacency matrix $A$.   We chose EVT rather than a version of spectral clustering based on
the graph Laplacian because SDPs also all operate on $A$ itself.    For SDP-3, in simulations we set the tuning
parameter $\lambda$ to the optimal value given in~\eqref{eq:log:like:Z-form:1}; a data-driven choice is given in \cite{Cai2014}.   

We first 
consider the balanced symmetric model $\bsymBM(p,q)$,
reparametrized in terms of the average expected degree $\avgd =
p(\frac{n}{K}-1) + q \frac{n}{K}(K-1)$ and the out-in-ratio $\beta =
q/p < 1$.    Estimation becomes harder when $\avgd$ decreases (fewer
edges) and when $\beta$ increases (communities are not well separated). As $K$ increases, however, estimation becomes harder to a certain point and then becomes relatively easier in some settings. 
\newcommand{\nmiscale}{.29}
%
\begin{figure}[t]
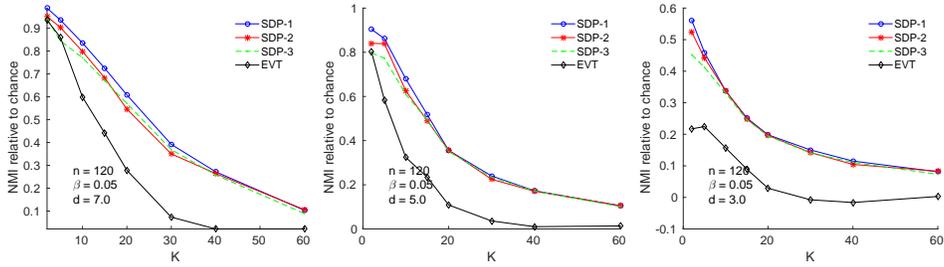

  \centering
  \includegraphics[scale=\nmiscale]{figs/Kvar/{nmi_n120_l7.0_Kvar_oir0.05_G2}.eps}
  \includegraphics[scale=\nmiscale]{figs/Kvar/{nmi_n120_l5.0_Kvar_oir0.05_E2}.eps}
  \includegraphics[scale=\nmiscale]{figs/Kvar/{nmi_n120_l3.0_Kvar_oir0.05_F2}.eps}
  \caption{Bias-corrected NMI vs. $K$ in a balanced planted partition~model, for various values of average degree $\avgd$, with $n=120$ and $\beta = 0.05$. }
  \label{fig:nmi}
\end{figure}
Figure~\ref{fig:nmi} shows the agreement of estimated labels with the
truth, as measured by the normalized mutual
information (NMI), versus the number of communities $K$,  averaged over 25 Monte Carlo replications.   NMI takes values between 0 and 1, with
higher values representing a better match.
%
%
The labels are estimated from $\Xh$ by Algorithm~\ref{alg:graphon}. 
As expected, the SDPs rank according to the tightness of
relaxation, with SDP-1 dominating the other two, and all SDPs
outperforing EVT. In Figure~\ref{fig:nmi}, the NMI is bias-adjusted, so that random guessing maps to NMI = $0$. Without the adjustment, the NMI‌ of random guessing increases as $K$ approaches $n$, leading to a ``dip'' in the plots.  See Figure~\ref{fig:nmi:adj} in~Appendices and the discussion that follows for more details.

\renewcommand{\ppwidth}{.23\linewidth}
\renewcommand{\ppscale}{.24}
\begin{figure}[t]
\centering
  \begin{subfigure}[t]{\ppwidth}
  \includegraphics[scale=\ppscale]{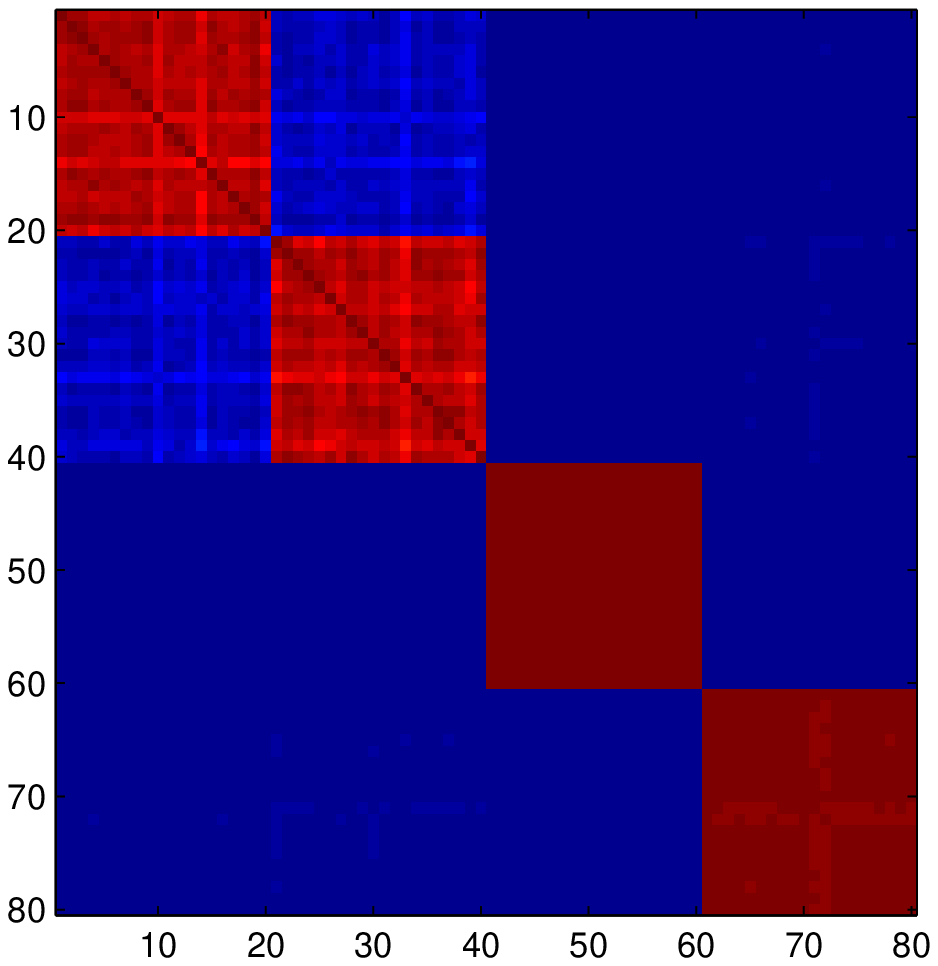}
  \caption[t]{{\tiny SDP-1, $p_3=0.7$}} 
  \end{subfigure}
  \begin{subfigure}[t]{\ppwidth}
  \includegraphics[scale=\ppscale]{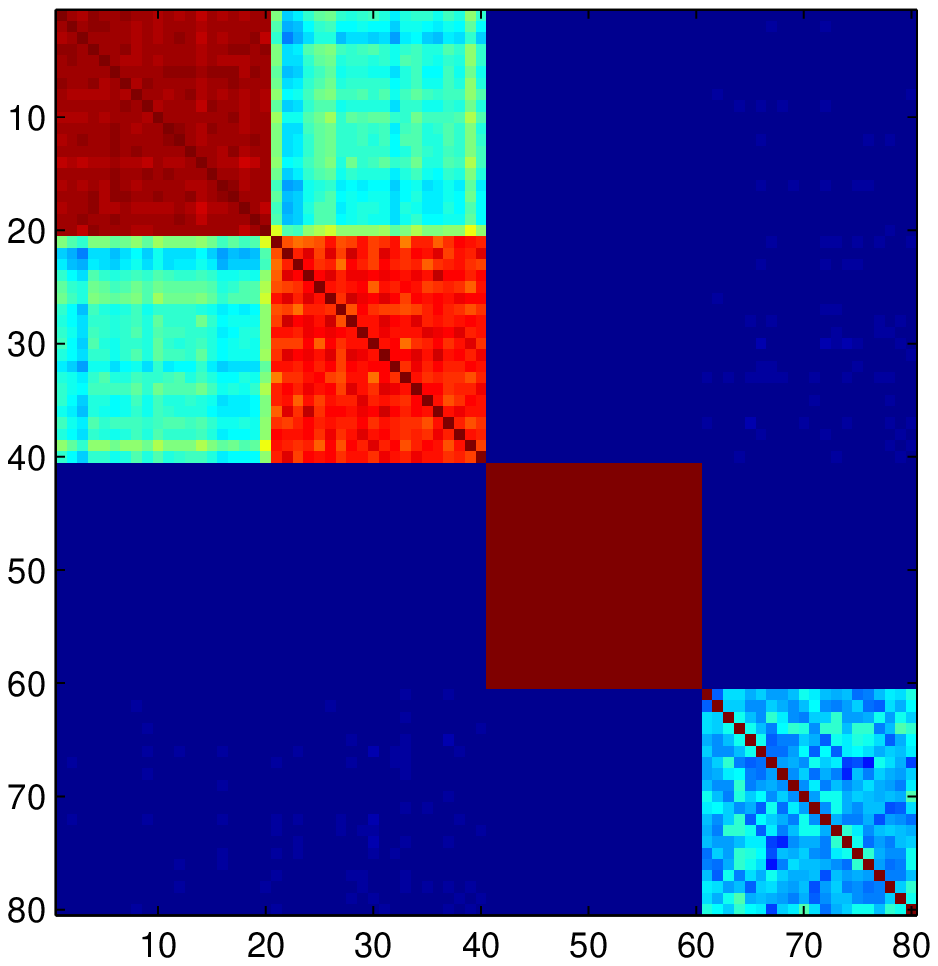}
  \caption[t]{\tiny SDP-2, $p_3=0.7$}  
  \end{subfigure}
  \begin{subfigure}[t]{\ppwidth}
  \includegraphics[scale=\ppscale]{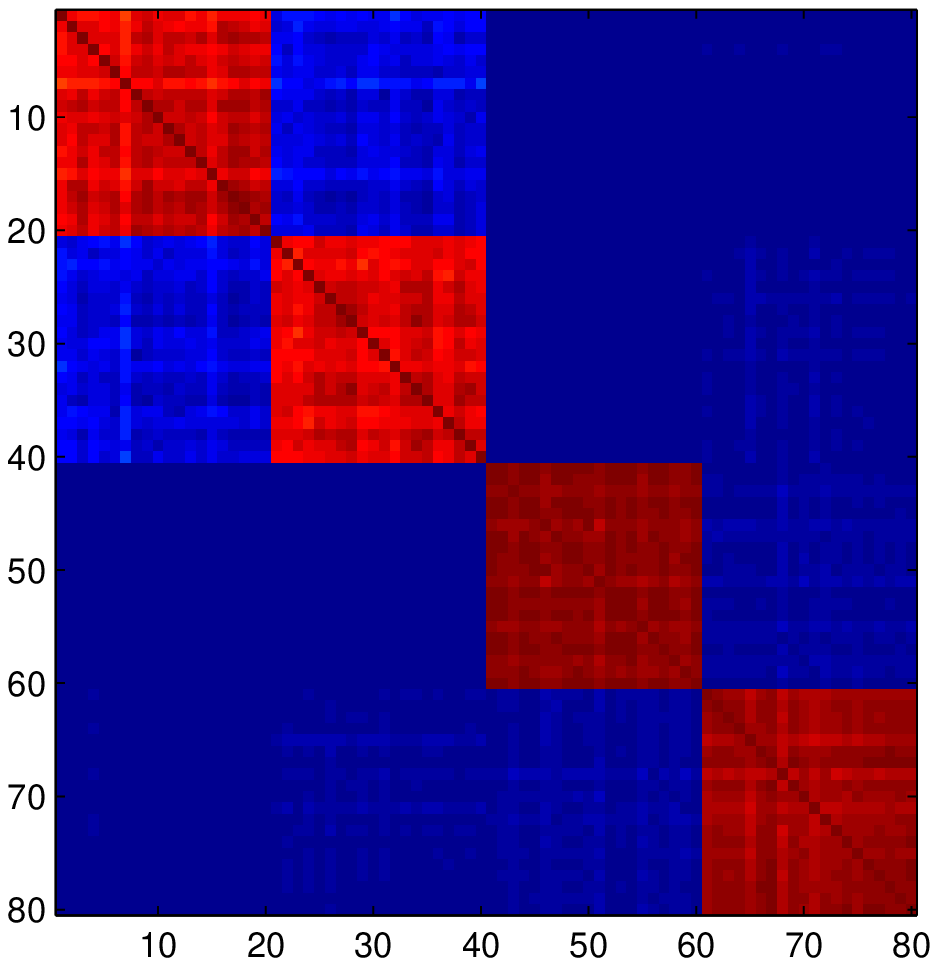}
  \caption[t]{{\tiny SDP-1, $p_3=0.05$}} 
  \end{subfigure}
  \begin{subfigure}[t]{\ppwidth}
  \includegraphics[scale=\ppscale]{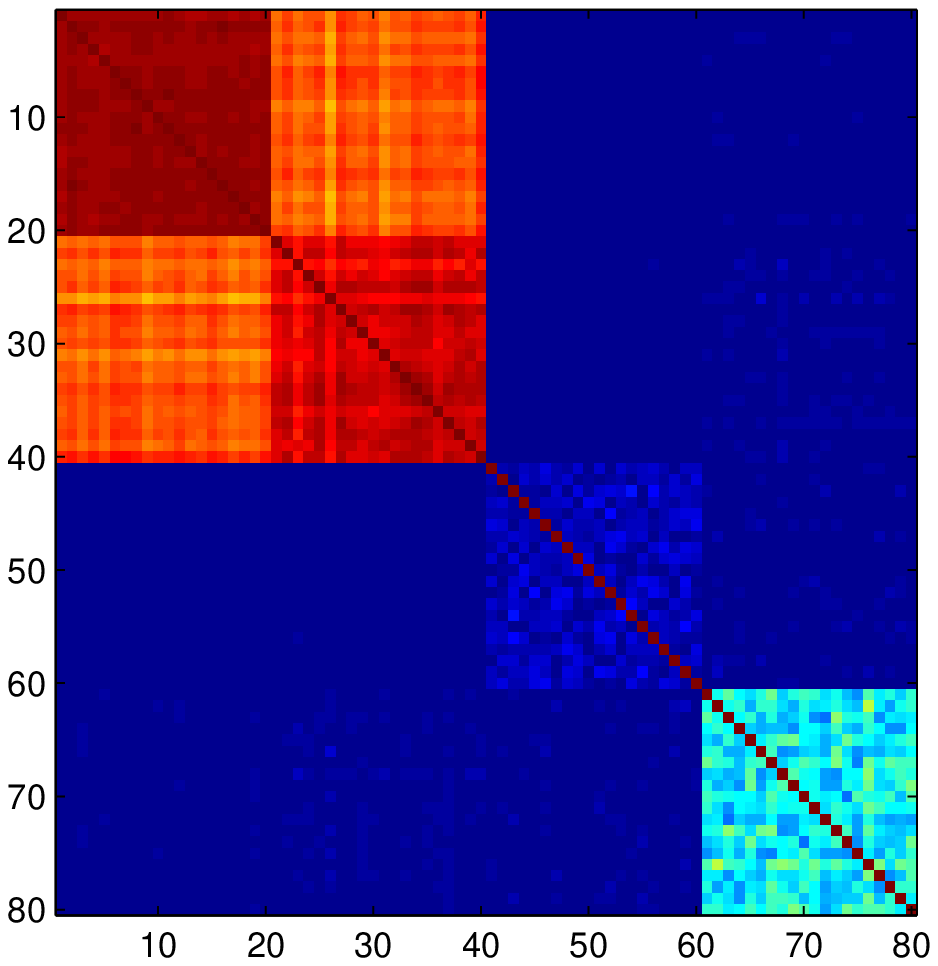}
  \caption[t]{{\tiny SDP-2, $p_3=0.05$}}  
  \end{subfigure}
\caption{Mean estimated cluster matrices, $\Xh$, for SDP-1 and SDP-2, for the weakly but not strongly assortative model ~\eqref{eq:K4:Psi:exam} with $p_3=0.7$ and $p_3 =0.05$. SDP-2 fails to recover one block at $p_3 = 0.7$ and two blocks at $p_3 = 0.05$.}
\label{fig:p3:cluster:mats}
\end{figure}

Next, we consider a more general balanced block model $\bBM(\Qm)$, with $K = 4$ to investigate the predictions of the theorems of Section~\ref{sec:consist:results}. We consider the probability matrix
\begin{align}\label{eq:K4:Psi:exam}
  \Qm = 
  \begin{pmatrix}
     .7  & .4  & .05 &.2 \\
     .4  & .6  & .05 &.2 \\
     .05 & .05 & p_3  &.05 \\
     .2  & .2  & .05 &.4  
  \end{pmatrix}
\end{align}
and we vary $p_3$ from $0.7$ down to $0.05$.   This model never satisfies the strong assortativity assumption over the range of $p_3$, because of the last row. However, it is at the boundary of strong assortativity if $p_3 > 0.4$, since $\Psi_{44} = \max_{k \neq \ell} \Psi_{k \ell}$ and $\Psi_{jj} > \max_{k \neq \ell} \Psi_{k \ell}$ for $j \neq 4$. Its deviation from strong assortativity increases once $p_3$ falls below $0.4$, and again once it crosses below $0.2$. However, except for the boundary value of $p_3 = 0.05$, the model always remains weakly assortative. Figure~\ref{fig:p3:cluster:mats} shows the results of Monte Carlo simulations  with 25 replications, for SDP-1 and SDP-2. Mean cluster matrices $\Xh$ obtained for the two SDPs are shown at the boundary points $p_3 = 0.05, 0.7$.   SDP-2 has difficulty recovering the fourth block in both cases, and completely fails to recover the third block when $p_3 = 0.05$. The performance of SDP-1, however, remains more or less the same, surprisingly even at $p_3=0.05$. This can be clearly seen in Figure~\ref{fig:assort:nmi:err}, which shows  the relative errors $\mnorm{\Xh-\Xtru}_F/\mnorm{\Xtru}_F$ for cluster matrices and the NMI for the  labels reconstructed by Algorithm~\ref{alg:graphon}. Note how SDP-2 degrades as $p_3$ decreases to $0.05$, with a sharp drop around $0.2$, while SDP-1 behaves more or less the same.   Note that for larger values of $p_3$, while SDP-2 does not reconstruct $\Xtru$ sexactly as seen from the relative error plot, the resulting labels are nearly always exactly the truth as seen from the NMI plot. This may be due to the EVT truncation on $\Xh$ implicit in Algorithm~\ref{alg:graphon}.
\begin{figure}[t]
\centering
  \includegraphics[scale=.42]{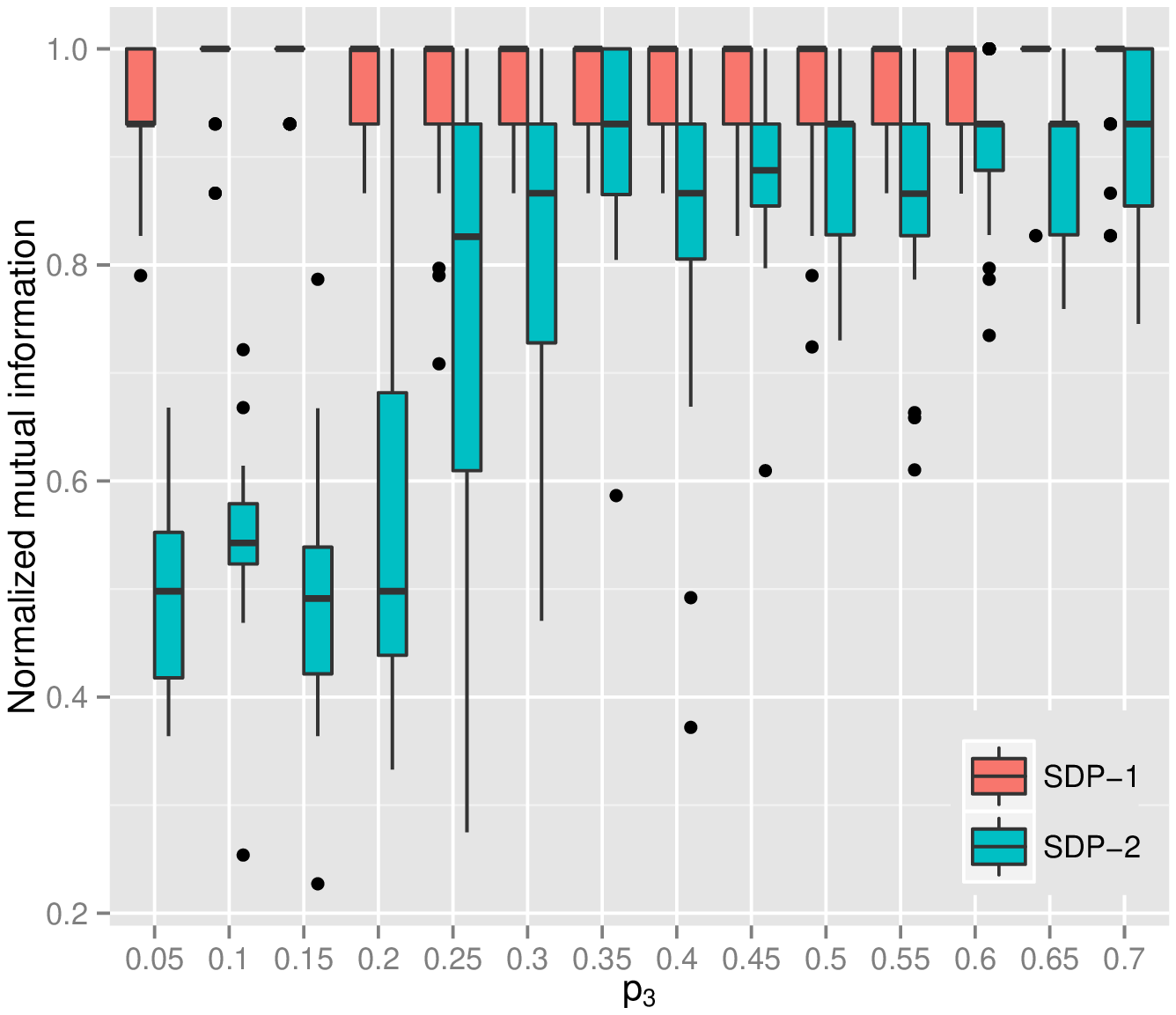}\hspace{2ex}
  \includegraphics[scale=.42]{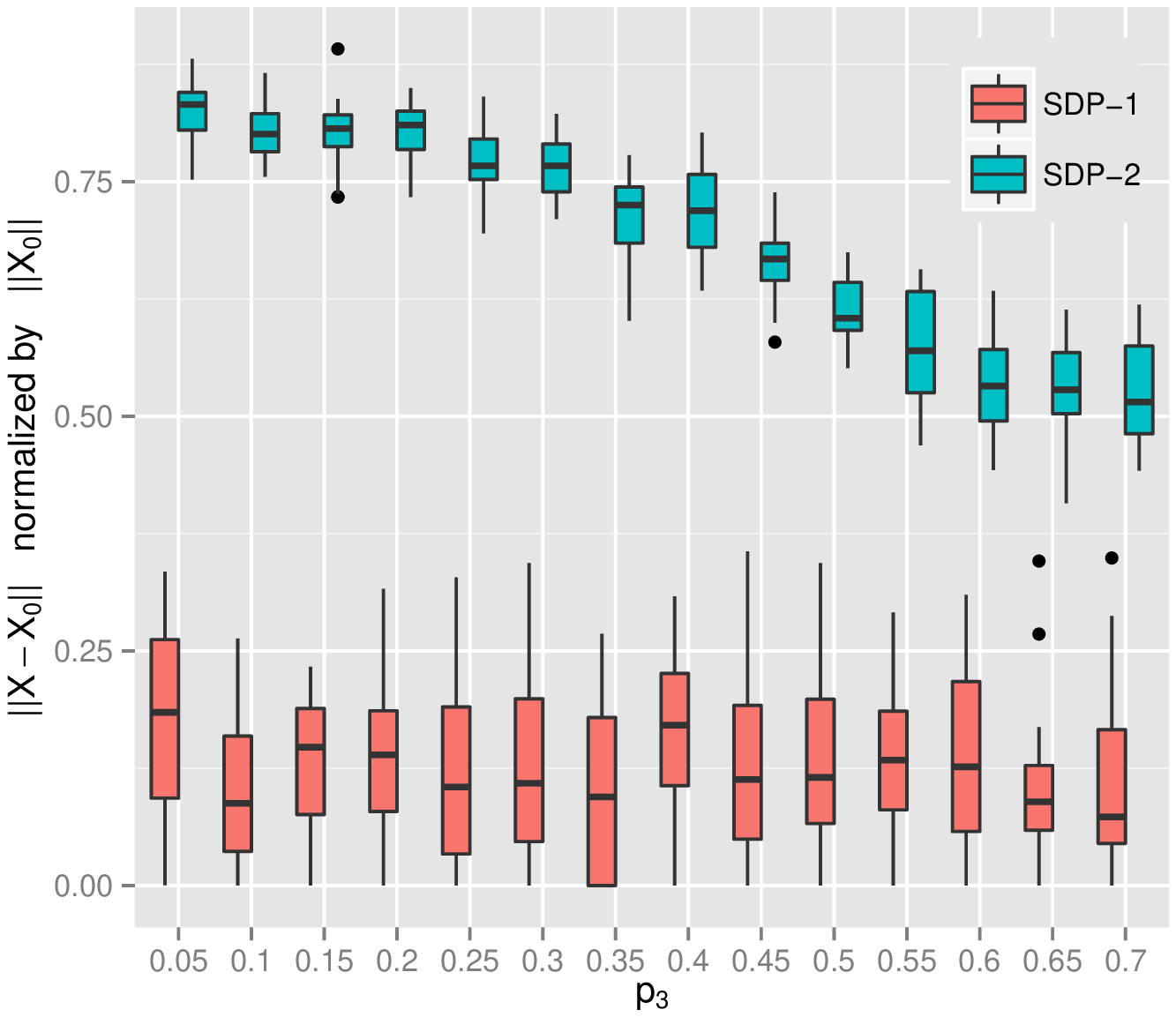}
  \caption{NMI and relative error of $\Xh$ versus $p_3$ for the model with probability matrix~\eqref{eq:K4:Psi:exam}.}
  \label{fig:assort:nmi:err}
\end{figure}

Finally, we apply various SDPs to graphon estimation for the dolphins network~\cite{Lusseau&Newman2004}, with
$n=62$ nodes. Figure~\ref{fig:dolphins} shows the results for the three SDPs
and the EVT with $K= 3$ and $K =10$.   For SDP-3, we used the median connectivity to set $\lambda$
as suggested in \cite{Cai2014}.   The adjacency matrices in the first row and the graphon estimators
in the second row are both permuted according to the ordering from
Algorithm~\ref{alg:graphon}.   The SDPs again provide a much cleaner
picture of the communities in the data than the EVT.  The blocks found by SDP-1 are similar in size and well separated from each other compared to the other two SDPs. 
We also applied the algorithms with $K=2$ to compare to the the partition suggested
by Fig. 1(b) in \cite{Lusseau&Newman2004}, which can be considered the ground truth for a two-community structure.  SDP-1, SDP-2, SDP-3, and EVT
misclassify  7, 1, 4, and 11 nodes, respectively, out of 62.  Since this 
partition into two blocks has 
unbalanced blocks (20 and 42), we expect SDP-1 to not match it as well.  
However,  if we replace equality constraints with
the inequality ones as discussed in Remark~\ref{rem:outliers}, SDP-1 misclassifies only 2
nodes. It is worth noting that the ground truth in this case is only one possible way to describe the network, taken from one scientific paper focused on the dolphins split, and there may well be more communities than two in the data.  The nine strong (and one weak) clusters found by SDP-1 may be of interest for further understanding of this network.



\newcommand{\dolphins}[3]{{figs/dolphins/dolphins_#1_K#2_#3.eps}}
\newcommand{\dolphwidth}{.118\linewidth}
\newcommand{\dolphscale}{.18}
\begin{figure}[t]
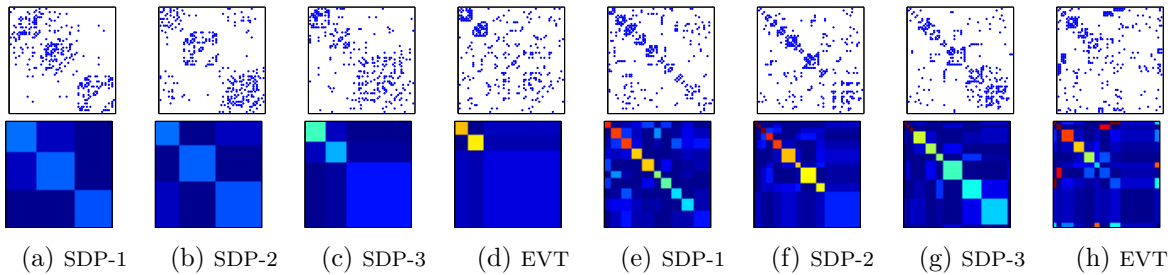

	\centering
	\begin{subfigure}[t]{\dolphwidth}
	\includegraphics[scale=\dolphscale]{\dolphins{adj}{3}{1}}
	\includegraphics[scale=\dolphscale]{\dolphins{Grphn}{3}{1}}
	\subcaption{\scriptsize SDP-1}
	\end{subfigure}
	\begin{subfigure}[t]{\dolphwidth}
	\includegraphics[scale=\dolphscale]{\dolphins{adj}{3}{2}}
	\includegraphics[scale=\dolphscale]{\dolphins{Grphn}{3}{2}}
	\subcaption{\scriptsize SDP-2}
	\end{subfigure}
	\begin{subfigure}[t]{\dolphwidth}
	\includegraphics[scale=\dolphscale]{\dolphins{adj}{3}{3}}
	\includegraphics[scale=\dolphscale]{\dolphins{Grphn}{3}{3}}
	\subcaption{\scriptsize SDP-3}
	\end{subfigure}
	\begin{subfigure}[t]{\dolphwidth}
	\includegraphics[scale=\dolphscale]{\dolphins{adj}{3}{4}}
	\includegraphics[scale=\dolphscale]{\dolphins{Grphn}{3}{4}}
	\subcaption{\scriptsize EVT}
	\end{subfigure}
	\begin{subfigure}[t]{\dolphwidth}
	\includegraphics[scale=\dolphscale]{\dolphins{adj}{10}{1}}
	\includegraphics[scale=\dolphscale]{\dolphins{Grphn}{10}{1}}
	\subcaption{\scriptsize SDP-1}
	\end{subfigure}
	\begin{subfigure}[t]{\dolphwidth}
	\includegraphics[scale=\dolphscale]{\dolphins{adj}{10}{2}}
	\includegraphics[scale=\dolphscale]{\dolphins{Grphn}{10}{2}}
	\subcaption{{\scriptsize SDP-2}}
	\end{subfigure}
	\begin{subfigure}[t]{\dolphwidth}
	\includegraphics[scale=\dolphscale]{\dolphins{adj}{10}{3}}
	\includegraphics[scale=\dolphscale]{\dolphins{Grphn}{10}{3}}
	\subcaption{\scriptsize SDP-3}
	\end{subfigure}
	\begin{subfigure}[t]{\dolphwidth}
	\includegraphics[scale=\dolphscale]{\dolphins{adj}{10}{4}}
	\includegraphics[scale=\dolphscale]{\dolphins{Grphn}{10}{4}}
	\subcaption{\scriptsize EVT}
	\end{subfigure}
	
	\caption{Results for the dolphins network for $K=3$ (a--d) and
          $K=10$ (e--h). Row 1: adjacency matrix sorted
  according to the permutation of Algorithm~\ref{alg:graphon}. Row 2:
  Graphon estimator $\Mh_{\Zh}$ of Algorithm~\ref{alg:graphon}. 
  }
	\label{fig:dolphins}
\end{figure}
%


\vspace{-1ex} 
\section{Discussion} \label{sec:discuss}
In this paper, we have put several SDP relaxations of the MLE
into a unified framework (Table~\ref
{SDP-table}) by treating them as relaxations over different parameter spaces.  
SDP-1, the tighter relaxation we proposed, was shown to empirically dominate previous relaxations, and while all the SDPs we considered are strongly consistent on the strongly assortative class of block models, we showed that SDP-1 is strongly consistent  on the much larger class of weakly assortative models, while  SDP-2 fails outside the strongly assortative class. We proposed a mixture of SDP-1 and SDP-3 which combines the flexibilities of both, namely, consistency in weakly assortative and unbalanced models.
It remains an open question whether a SDP relaxation can work for mixed networks with both assortative and dissortative communities. There are some indications that one can tackle mixed networks by applying SDPs to $|A| = \sqrt{A^2}$, the positive square-root of $A$.


 We also note that SDP-3 is harder to compare directly to SDP-1 or SDP-2 because it depends on a 
tuning parameter $\lambda$.  However,  Lagrange duality implies that
for every $A$ and $K$ there exists a $\lambda$ that makes SDP-3
equivalent to SDP-2.  In general, SDP-3 is more flexible than SDP-2 because of the continuous parameter $\lambda$, but this also makes it unsuitable for certain tasks such as histogram estimation.   
Empirically, the SDPs outperformed adjacency-based spectral clustering (EVT), especially for a large number of communities $K$. This is reflected in current theoretical guarantees, where the conditions for the SDPs have better dependence on $K$ that those available for the EVT. In addition, SDP formulation of EVT shows it to be a looser relaxation than, say, SDP-1 for balanced planted partition model. The three SDPs also seem to be inherently more robust to noise than the EVT, perhaps due to the implicit regularization effect of the doubly nonnegative cone.


\appendix



\section{Proof of Lemma~\ref{LEM:VALID:GAMMA:SDP1}}\label{sec:proof:valid:Gamma:SDP1}
As in the proof of Theorem~\ref{thm:consist:SDP2}, and in accordance with condition (A2), we set $\Gamma_{\subb{k}} := 0$.
Then, 
\begin{align}
  \label{eq:Lambda:def:SDP1}
  \begin{split}
    \Lambda_{\subb{k}} &= \mu_{\subb{k}}  \onev_m^T + \onev_m \mu_{\subb{k}}^T + \diag^*(\nu_{\subb{k}})- A_{\subb{k}} \\
    \Lambda_{\subb{k}^c \subb{k}} &= \mu_{\subb{k}^c}  \onev_m^T + \onev_{n-m} \mu_{\subb{k}}^T - (A+\Gamma)_{\subb{k}^c \subb{k}} 
  \end{split}
\end{align}
Recalling that $[\dg{}{k}]_{\subb{k}} = A_{\subb{k}} \onev_m$, we can rewrite~\eqref{eq:A1:equiv:1} as 
\begin{alignat}{4}
 \label{eq:SDP1:A1:equiv:1a}
  \Lambda_{\subb{k}} \onev_m = 0  &\iff 
    && \mu_{\subb{k}} m + \onev_m \mu_{\subb{k}}^T \onev_m + \nu_{\subb{k}} - [\dg{}{k}]_{\subb{k}} &&= 0 \\
 \label{eq:SDP1:A1:equiv:1b}
  \Lambda_{\subb{\ell}\subb{k}} \onev_m = 0 &\iff 
    && \big[\mu_{\subb{\ell}} \onev_m^T + \onev_m \mu_{\subb{k}}^T - (A+\Gamma)_{\subb{\ell} \subb{k}}\big] \onev_m &&= 0 , \quad k \neq \ell
 \end{alignat}
As in the case of SDP-$2'$ (cf.~Appendices), \eqref{eq:SDP1:A1:equiv:1b} is equivalent to
\begin{align*}
  \mu_{\subb{\ell}} \onev_m^T + \onev_m \mu_{\subb{k}}^T - (A+\Gamma)_{\subb{\ell} \subb{k}} = - B_{\subb{\ell}\subb{k}}
\end{align*}
for some $B_{\subb{\ell}\subb{k}}$ acting on $\Span\{\onev_m\}^\perp$. As before, we set $B_{\subb{\ell}\subb{k}}:=\projonep A_{\subb{\ell}\subb{k}} \projonep$, and note that
\begin{align}
  \label{eq:SDP1:Delta:def}
  \Delta := A - \ex A, \quad  [\ex A]_{\subb{k}} = p_k \onem_m, \quad [\ex A]_{\subb{k}\subb{\ell}} = q_{k \ell} \onem_m, \; k \neq \ell,
\end{align}
so that $B_{\subb{\ell}\subb{k}} = \projonep \Delta_{\subb{\ell}\subb{k}} \projonep$. 
 Now, take $u \in \Span\{\onev_{S_k}\}^\perp$. Then, $u = \sum_k u_{\subb{k}} = \sum_k e_k \otimes u_k$, for some $\{u_k\} \subset \Span\{\onev_m\}^\perp$. We will work with expansion of $u^T \Lambda u$ obtained in~Appendices (Eq.~\eqref{eq:u:Lambda:u:expan}).
  Using $A_{\subb{k}} = p_k \onem_m + \Delta_{\subb{k}}$ and~\eqref{eq:Lambda:def:SDP1}, we have
  \begin{align*}
    u_k^T \Lambda_{\subb{k}} u_k 
      &= u_k^T\big[\mu_{\subb{k}}  \onev_m^T + \onev_m \mu_{\subb{k}}^T  - p_k \onem_m + \diag^*(\nu_{\subb{k}}) - \Delta_{\subb{k}} \big] u_k \\
      &= u_k^T\big( \diag^*(\nu_{\subb{k}}) - \Delta_{\subb{k}} \big) u_k 
  \end{align*}
  using $\onev_m^T u_k = 0$. Let us now choose $\mu$ to be constant over blocks, that is,
  $\mu_{\subb{k}} := \frac12 \mubar_k \onev_m, \, \forall k$ 
  for some numbers $\{\mubar_k\}$ to be determined later. Note that~\eqref{eq:SDP1:A1:equiv:1a} reads
   $ \mubar_k m \onev_m + \nu_{\subb{k}} - [\dg{}{k}]_{\subb{k}} = 0$
  or equivalently 
  \begin{align}\label{eq:nu:def:SDP1:b}
    \diag^*(\nu_{\subb{k}}) = \diag^*([d(\subb{k})]_{\subb{k}}) - \mubar_k m I_m.
  \end{align}
  On the other hand, for $k \neq \ell$, we have $u_k^T \Lambda_{\subb{k}\subb{\ell}} u_\ell = -u_k^T B_{\subb{k}\subb{\ell}} u_\ell = -u_k^T \Delta_{\subb{k}\subb{\ell}}  u_\ell$ 
  since $\{u_k\} \subset \Span\{\onev_m\}^\perp$. We arrive at
  \begin{align}\label{eq:u:Lambda:u:SDP1}
    u^T \Lambda u = \sum_k u_k^T\big( \diag^*([d(\subb{k})]_{\subb{k}}) - \mub_k m I_m - \Delta_{\subb{k}} \big) u_k 
    -\sum_{k \neq \ell} u_k^T \Delta_{\subb{k} \subb{\ell}} u_\ell
  \end{align}


\begin{proof}[Proof of (a) and (b)]
  To verify dual feasibility, recall that $\projonep e_j = e_j - \frac1m \onev_m$. Then,
  \begin{align*}
    [\Gamma_{\subb{k}\subb{\ell}}]_{ij} 
    = e_i^T \Gamma_{\subb{k}\subb{\ell}} e_j 
    &=  \frac12(\mubar_k + \mubar_\ell) + (e_i - \frac1m \onev_m)^T A_{\subb{k}\subb{\ell}} (e_j - \frac1m \onev_m) - A_{ij}  \\
    &=  \frac12(\mubar_k + \mubar_\ell)   - \frac1{m} \big[\dg{i}{\ell}  + \dg{j}{k} - \dav{k}\ell \big] \ge 0.
  \end{align*}
    
  To verify~\eqref{eq:A1:equiv:2}, we recall representation~\eqref{eq:u:Lambda:u:SDP1}. By assumption $ \diag^*([d(\subb{k})]_{\subb{k}}) \succeq \rho_k m I_m$ for all $k$. From~\eqref{eq:nu:def:SDP1:b} and~\eqref{eq:u:Lambda:u:SDP1} it follows that for $u \in \Span\{\onev_{S_k}\}^\perp$
  \begin{align*}
    u^T \Lambda u &\ge 
    \sum_k u_k^T\big(\rho_k m I_m - \mubar_k m I_m - \Delta_{\subb{k}} \big) u_k 
    -\sum_{k \neq \ell} u_k^T \Delta_{\subb{k} \subb{\ell}} u_\ell \\
     &\ge \sum_k \big[ (\rho_k - \mubar_k)m - \mnorm{\Delta_{\subb{k}}} \big]\vnorm{u_k}^2 
     -  u^T (\onem_{\Stru^c} \circ \Delta) u\\ 
      &\ge \min_k \big[(\rho_k - \mubar_k) m - \mnorm{\Delta_{\subb{k}}}\ \big] \,\|u\|^2 - \mnorm{\onem_{\Stru^c} \circ \Delta} \,\|u\|^2.
  \end{align*}
\vspace{-4ex}

  \end{proof}


\section{Probabilistic bounds for $\bBM$}\label{sec:prob:cond:bm}
  We will complete the construction of $(\mu,\nu,\Gamma)$ in~\crefrange{eq:nu:def:SDP1}{eq:Gamma:def:SDP1} for $\bBM_m(\Qm)$, by specifying $\{\mubar_k\}$ and finishing the the proof of Theorem~\ref{thm:consist:SDP1}. The following is the analogue of Lemma~\ref{LEM:DEG:CONCENT:SDP2} in~Appendices, for $\bBM_m(\Qm)$. The proof is similar and is omitted.
  \begin{lem}\label{lem:deg:concent:SDP1}
  Let $\gamma_k := \sqrt{ (4c_1 \log n)/ \pb_k}$ and $\zeta_{k \ell} := \sqrt{ (4c_2 \log n)/ \qb_{k\ell}}$. Assume $\gamma_k,\zeta_{k \ell} \in [0,3]$.~Then,
  \begin{alignat*}{4}
    \dg{i}{k} &\ge \pb_k(1-\gamma_{k \ell}), 
      &&\; i \in \subb{k}, \forall k
      && \quad \text{w.p. at least $1 - n^{-(c_1-1)}$, and} \\
    \big| \dg{i}{\ell} - \qb_{k\ell} \big| &\le \zeta_{k\ell} \, \qb_{k\ell}, 
      &&\; i \in \subb{k}, \forall (k \neq \ell),   
      &&\quad \text{w.p. at least $1-2 m^{-1}n^{-(c_2-2)}$}.
  \end{alignat*}
  \end{lem}

  We also have the following corollary of Proposition~\ref{prop:key:adj:concent} for $\bBM_m(\Qm)$. Recall the chain of definitions and equivalences:
   $ \qsbmax := \max_k \,\qsb_k = \max_{k \neq \ell} \,\qb_{k \ell} = m (\max_{k \neq \ell} \, q_{k \ell} ) =: m \qmax$.

  \begin{cor}\label{cor:adj:concent:SDP1}
  Let $A \in \{0,1\}^{n \times n}$ be distributed as $\bBM_m(\Qm)$ and $\Delta := A - \ex A$. Assume that $p_k \ge (C' \log m) / m$ for all $k$ and $\qmax \ge (C' \log n)/ n$. Then, 
  \begin{itemize}
    \item $\max_k \mnorm{\Delta_{\subb{k}}} \le C \sqrt{\pb_k}$, w.p. at least $1- c K m^{-r}$. 
    \item $\mnorm{\onem_{\Stru^c} \circ \Delta} \le C \sqrt{\qsbmax K}$ w.p. at least $1-c n^{-r}$.
  \end{itemize}
  \end{cor}
  \begin{proof}
    The assertion about diagonal blocks follows as in Corollary~\ref{cor:adj:concent:SDP2} in Appendices. 
    For the second assertion, we note that $\onem_{\Stru^c} \circ \Delta$
    is an $n \times n$ matrix whose entries have variance $\le \max_{k,\ell} \,(q_{k,\ell}) = \max_{k,\ell} (\qb_{k,\ell}/m) =  \qsbmax / m$, hence  w.p. at least $1-cn^{-r}$,  $\mnorm{\onem_{\Stru^c} \circ \Delta} \le C \sqrt{(\qsbmax/m) n}$. 
  \end{proof}

  According to Lemma~\ref{lem:deg:concent:SDP1}, for sufficiently small $\gamma_k$ and $\zeta_{k,\ell}$, we have w.h.p. that
  $\dav{k}{\ell}$  also lies in $[\qb_{k \ell}(1-\zeta_{k\ell}),\qb_{k \ell} (1+\zeta_{k \ell})]$, for $k \neq \ell$, so that 
  \begin{align*}
    \dg{i}{\ell}  + \dg{j}{k} - \dav{k}\ell &\le \qb_{k \ell}(1+3\zeta_{k \ell}) \\
     &\le \qb_{k \ell} + 3 \sqrt{4c_2 \, \qb_{k \ell} \log n}, \quad (i,j) \in \subb{k}\times \subb{\ell}.
  \end{align*}
  Note that right-hand side is increasing in $\qb_{k \ell}$. We also note the following key inequality
    $\qb_{k \ell} \le \frac12 (\qsb_k + \qsb_\ell)$
  obtained by summing the following two inequalities
  \begin{align*}
    \frac12 \qb_{k \ell} \le \frac12 \max_{r = k, s \neq k} \qb_{r s}, \quad 
     \frac12 \qb_{k \ell}  \le \frac12 \max_{r \neq  \ell, s = \ell} \qb_{r s}
  \end{align*}
  which hold for $k \neq \ell$. Hence,
  \begin{align*}
    \qb_{k \ell} + 3 \sqrt{4c_2 \, \qb_{k \ell} \log n} 
      &\le \frac12(\qsb_k + \qsb_\ell) + 3 \sqrt{2c_2 \, (\qsb_k+\qsb_\ell) \log n} \\
      &\le \frac12\Big(\qsb_k + 6\sqrt{2c_2 \, \qsb_k \log n}\Big) +
       \frac12 \Big(\qsb_\ell + 6 \sqrt{2c_2 \, \qsb_\ell \log n}\Big) 
  \end{align*}
  where we have used $\sqrt{x+y} \le \sqrt{x} + \sqrt{y}$ for $x,y \ge 0$. Thus, taking
  \begin{align*}
    \mubar_k := \frac1m \phib_k, \quad \phib_k:= \qsb_k + 6\sqrt{2c_2 \, \qsb_k \log n}
  \end{align*}
  satisfies~\eqref{eq:mu:low:bound:SDP1}. We also have 
  $
    m \rho_k := \min_{i \in \subb{k}} \dg{i}{k} \ge \pb_k - \sqrt{4c_1\, \pb_k \log n}
  $,
  and from Corollary~\ref{cor:adj:concent:SDP1}, $\mnorm{\Delta_{\subb{k}}} \le C \sqrt{\pb_k}$ for all $k$. It follows that
  \begin{align*}
    (\rho_k - \mubar_k) m - \mnorm{\Delta_k}
      &\ge \pb_k - \sqrt{4c_1\, \pb_k \log n}
        - \Big(\qsb_k + 6\sqrt{2c_2 \, \qsb_k \log n}\Big) - C \sqrt{\pb_k} \\
      &\ge (\pb_k 
        - \qsb_k) - (C + \sqrt{4c_1}) \sqrt{\pb_k \log n} - 6\sqrt{2c_2 \, \qsb_k \log n} 
  \end{align*}
   By Corollary~\ref{cor:adj:concent:SDP1}, $\mnorm{\onem_{\Stru^c} \circ \Delta} \le C \sqrt{\qsbmax K}$. Thus, to satisfy~\eqref{eq:Delta:up:bound:SDP1}, it is enough to have
  \begin{align*}
    \min_k \Big[  (\pb_k - \qsb_k) - (C + \sqrt{4c_1}) \sqrt{\pb_k \log n} - 6\sqrt{2c_2 \, \qsb_k \log n} 
    \Big] &> C \sqrt{\qsbmax K}.
  \end{align*}
  which is implied by
  \begin{align}\label{eq:SDP1:log:cond}
  \min_k \Big[  (\pb_k - \qsb_k) - C_2 \big(\sqrt{\pb_k \log n} + \sqrt{\qsb_k \log n}\big) 
    \Big] &> C \sqrt{\qsbmax K}.
  \end{align}

  Auxiliary conditions we needed on $\pb_k$ and $\qb_{k \ell}$ were $\pb_k \ge (4c_1/9) \log n$ and $\qb_{k \ell} \ge (4c_2/9) \log n$ from Lemma~\ref{lem:deg:concent:SDP1} and $\pb_k \ge C' \log m$ and $n \qmax > C' \log n$. As before, we can drop the lower bounds on $\{q_{k\ell}\}_{k \neq \ell}$ due to~Corollary~\ref{cor:BM:ordering}. The lower bounds on $\pb_k$ are implied by $\pb_k \ge (C' \vee (4c_1/9)) \log n$. This completes the proof. To get to the form in which the theorem is stated, replace $c_1$ with $c_1 + 1$ and $c_2$ with $c_2 + 2$, and divide~\eqref{eq:SDP1:log:cond} by $\log n$.





	    
	       
	             







The following appendices contain proofs of the remaining results, a detailed description of the implementation of an ADMM solver for SDP 1, and additional details on simulations.

\section{Proofs of Section~\ref{SEC:SDP:RESPECTS:ORDERING}}\label{sec:proof:SDP:respects:ordering}
\begin{proof}[Proof of Lemma~\ref{lem:deterministic:nesting}]
  Let $\Stru := \supp(\Xtru)$. We proceed in two steps, first setting elements on $\Stru$ to one, and then setting elements on $\Stru^c$ to zero. 
  More precisely, let $\At_1 = \At$ on $\Stru$ (meaning that $[\At_1]_{ij} = \At_{ij}$ for $(i,j) \in \Stru$) and $\At_1 = A$ on $\Stru^c$.
  Let $X$ be any feasible solution other than $\Xtru$, so that $0 \le X\le 1$. We will use the notation $\ip{A,X}_{\Stru} := \sum_{(i,j)\in \Stru} A_{ij} X_{ij}$. By (unique) optimality of $\Xtru$ for $A$, we have $\ip{A,X} < \ip{A,\Xtru}$. Then,
  \begin{align*}
       \ip{\At_1,X-\Xtru}_{\Stru^c}  = \ip{A,X-\Xtru}_{\Stru^c} < \ip{A,\Xtru-X}_{\Stru} \le \ip{\At_1,\Xtru-X}_{\Stru}
  \end{align*}
  where the first equality is by assumption and the last inequality follows from $A \le \At_1$ on $\Stru$, and that $\Xtru - X \ge 0$ on $\Stru$. (Note that $\Xtru = \onem_m$ on $\Stru$ and $X \le 1$ everywhere.) Hence, the conclusion of the lemma follows for $\At_1$.  Now, we can write
  \begin{align*}
    \ip{\At,X} \le \ip{\At_1,X} < \ip{\At_1,\Xtru} = \ip{\At,\Xtru}.
  \end{align*}
  The first inequality is by nonnegativity of $X$ and $\At \le \At_1$ everywhere. The second inequality is by (unique) optimality of $\Xtru$ for $\At_1$. The last equality is by $\At = \At_1$ on $\Stru$.
\end{proof}

\begin{proof}[Proof of Corollary~\ref{cor:BM:ordering}]
   We construct a coupling between $A$ and $\At$. Recall that $\subb{k}$ denotes the indices of nodes in community $k$. Draw $A \sim \bBM_m(\Qm)$, and draw
   \begin{alignat*}{3}
    R_{ij} &\sim \bern\Big(\frac{\pt_k - p_k}{1- p_k}\Big), &&\quad (i,j)\in \subb{k}, \forall k\\
    R_{ij} &\sim \bern(\qt_{k \ell}/q_{k \ell}), &&\quad (i,j)\in \subb{k} \times \subb{\ell}, \forall k < \ell
   \end{alignat*}
   independently from $A$. Extend $R$ symmetrically, by setting $R_{\subb{k}\subb{\ell}} = R_{\subb{k}\subb{\ell}}^T$ for $k > \ell$. Let 
   \begin{alignat*}{3}
    \At_{ij} &:= 1-(1-A_{ij})(1-R_{ij}), &&\quad (i,j)\in \subb{k}, \forall k\\
    \At_{ij} &:= A_{ij}R_{ij}, &&\quad (i,j)\in \subb{k}\times \subb{\ell}, \forall k < \ell
   \end{alignat*}
   and extend symmetrically. It is easy to verify that $\At$ has distribution $\bBM_m(\Qmt)$. Moreover, by construction $\At \ge A$ on $\supp(\Xtru)$ and $\At \le A$ on $\supp(\Xtru)^c$. The result now follows from Lemma~\ref{lem:deterministic:nesting}.
\end{proof}

\section{Proof of Lemma~\ref{LEM:SUFF:COND:EXACT:RECOV}}\label{sec:proof:suff:cond:exact:recov}
To prove the lemma, we need the following intermediate result.
\begin{lem}\label{LEM:BLOCK:CONSTANCY}
  Let $X\in\Symm{n}$ with $\rangeS(X) \subset \Span\{\onev_{\subb{k}}\}$. Then $X = B \otimes \onem_m$ for some $B \in \Symm{K}$, that is, $X$ is block-constant.
\end{lem}
\begin{proof}
   Note that $\onev_{\subb{k}} = \bv{k} \otimes \onev_m$ where $\bv{k} = \basisv{k}{K}$ is the $k$-th basis vector of $\reals^K$.
  An eigenvector $v_j$ of $X$ will be of the form $v_j = \sum_k \alpha_k^j \onev_{\subb{k}} = (\sum \alpha_k^j \bv{k}) \otimes \onev_m = u_j \otimes \onev_m$  for some $u_j \in \reals^K$. Then,
  \begin{align*}
    X = \sum_j \beta_j v_j v_j^T 
      &= \sum_j \beta_j (u_j \otimes \onev_m) (u_j \otimes \onev_m)^T \\
      &= \sum_j \beta_j (u_j u_j^T) \otimes(\onev_m \onev_m^T) = \big(\sum_j \beta_j u_j u_j^T\big) \otimes \onem_m.
   \end{align*} 

\end{proof}

\begin{proof}[Proof of Lemma~\ref{LEM:SUFF:COND:EXACT:RECOV}.]
  Conditions (A1) and (A2) together satisfy (CSa) and (CSb) for $\Xtru$ and $\dualtri$, in addition to dual feasibility. Hence, $\Xtru$ is an optimal solution of the primal problem. To show uniqueness, let $X$ be any optimal primal solution. Then $X$ and the specific triple $\dualtri$ assumed in the statement of the lemma should together satisfy optimality conditions. (CSb) for $X$ (and the triple) implies 
  \begin{align*}
    \rangeS(X) \subset \ker\big(\Lambda(\mu,\nu,\Gamma)\big) = \Span\{\onev_{\subb{k}}\}
  \end{align*}
  by (A1), which then implies $X = B \otimes \onem_m$ for some $B = (b_{k \ell})\in \Symm{K}$ by Lemma~\ref{LEM:BLOCK:CONSTANCY}. 
  Note that this means $X_{\subb{k} \subb{\ell}} = b_{k \ell} \onem_m$.
  Now, (CSa) for $X$ implies
  \begin{align*}
    0 = X_{\subb{k}\subb{\ell}}  \circ \Gamma_{\subb{k}\subb{\ell}} = b_{k \ell} \Gamma_{\subb{k}\subb{\ell}}, \quad \text{for}\; k \neq \ell
  \end{align*}
  using $\onem_m \circ D = D$, for any $D$. But since $\Gamma_{\subb{k}\subb{\ell}}$ is not identically zero by (A3), we should have $b_{k\ell} = 0$, for $k\neq \ell$. One the other hand, primal feasibility of $X$, in particular, $X_{ii} = 1$ implies $b_{kk} = 1$. That is, $B = I_K$, hence $X = I_K \otimes \onem_m = \Xtru$.
\end{proof}


\subsection{Proof of Theorem~\ref{thm:consist:SDP2}: primal-dual witness for SDP-$2'$}\label{sec:proof:consist:SDP2}
For SDP-$2'$, the linear condition $\Lc_2(X) = b_2$ is just the scalar equation $\ip{\onem_n,X} = n^2/K = n m$. The dual variable $\mu$ is a scalar in this case, and we have $b_2 = m n$ and $\Lc_2^*(\mu) = \mu \onem_n$. Hence, 
we have (cf.~\eqref{eq:Lambda:def:gen})
\begin{align}\label{eq:Lambda:def:SDP2}
  \Lambda = \Lambda(\mu,\nu,\Gamma) = \mu \onem_n + \diag^*(\nu) -A - \Gamma .
\end{align}

Let $\dg{}{k} = A \onev_{\subb{k}}$ be the vector of node degrees relative to subgraph $\subb{k}$. We denote its $i$th element by $\dg{i}{k} = \sum_{j \in \subb{k}} A_{ij}$.  Let $\projonep := I_m - \frac1m \onem_m$ be projection onto $\Span\{\onev_m\}^\perp$. The following summarizes our primal-dual construction, modulo the choice of $\mu$:
\begin{align}\label{eq:nu:def}
  \nu_i &:=   \dg{i}{k} - \mu m, \quad \text{for} \;i \in S_k,\\
  \begin{split}
  \label{eq:Gamma:def}
  \Gamma_{\subb{k}} &:= 0, \quad  \forall k \\
  \Gamma_{\subb{k}\subb{\ell}} &:= 
    \mu \onem_m + \projonep A_{\subb{k}\subb{\ell}} \projonep - A_{\subb{k}\subb{\ell}},
\end{split}
\end{align}
for all $k \neq \ell$.  Note that $\Gamma$ is symmetric. Let
\begin{align}\label{eq:dav:def}
    \Delta := A - \ex[A], \quad \dav{k}\ell := \frac1m \sum_{i \in \subb{k}} \dg{i}{\ell} =
       \frac1m \sum_{j \in \subb{\ell}} \dg{j}{k}.
\end{align}
 The following lemma, proved in~Section~\ref{sec:proof:valid:Gamma:SDP2}
 , verifies the validity of this construction.

\begin{lem}\label{LEM:VALID:GAMMA:SDP2}
  Let $(\mu,\nu,\Gamma)$ be as in~ \cref{eq:nu:def,eq:Gamma:def}. Then, $\Gamma$ verifies (A2) and~\eqref{eq:A1:equiv:1} holds for all $\mu$. In addition,
  \begin{itemize}
  \setlength\itemsep{0em}
    \item[(a)] $\Gamma$ is dual feasible, i.e. $\Gamma \ge 0$, if for all $i\in \subb{k}, j \in \subb{\ell}, k \neq \ell$, 
    \begin{align}\label{eq:mu:low:bound}
      \mu m \ge \dg{i}{\ell}  + \dg{j}{k} - \dav{k}\ell,
    \end{align}
    and satisfies (A3) if at least one inequality is strict for each pair $k \neq \ell$.
    \item[(b)] $\Gamma$ verifies~\eqref{eq:A1:equiv:2} if
    \begin{align}\label{eq:Delta:up:bound:SDP2}
      (\rho - \mu) m > \mnorm{\Delta}, \quad \text{where} \quad \rho := \min_k \min_{i \in \subb{k}} \dg{i}{k} /m.
    \end{align}

  \end{itemize}
\end{lem}

    We note that choosing $\mu m$ to be the maximum of the RHS of~\eqref{eq:mu:low:bound}, i.e.,
    \begin{align*}
      \max_{k,\ell} \max_{i \in \subb{k},\, j \in \subb{\ell}} \big[\dg{i}{\ell}  + \dg{j}{k} - \dav{k}\ell\big]
    \end{align*}
    together with~\eqref{eq:Delta:up:bound:SDP2} gives a deterministic condition for the success of SDP-$2'$. 
    In~Section~\ref{sec:prob:cond:ppbal}, we give a probabilistic version of this condition which completes the proof of Theorem~\ref{thm:consist:SDP2}.



\section{Proof of Lemma~\ref{LEM:VALID:GAMMA:SDP2}}\label{sec:proof:valid:Gamma:SDP2}
We start by seeing how far the KKT conditions determine the dual variables and how much freedom in choosing them is left. In accordance with condition (A2), we set $\Gamma_{\subb{k}} := 0$.
Then, \eqref{eq:A1:equiv:1} holds if and only if $\Lambda_{\subb{k}} \onev_m = 0$ and $ \;\Lambda_{\subb{k}^c\subb{k}} \onev_m = 0$, or equivalently,
\begin{alignat}{4}
\label{eq:A1:equiv:1a}
  \Lambda_{\subb{k}} \onev_m 
    &= (\mu  \onem_m + \diag^*(\nu_{\subb{k}})- A_{\subb{k}} ) \onev_m 
    &&= \mu m \onev_m + \nu_{\subb{k}} - A_{S_k} \onev_m  
    &&= 0\\
\label{eq:A1:equiv:1b}
  \Lambda_{\subb{k}^c\subb{k}} \onev_m 
    &= [\mu \onem_{n-m,m} - (A+\Gamma)_{\subb{k}^c\subb{k}}] \onev_m 
    &&= \mu m \onev_{n-m} - (A+\Gamma)_{\subb{k}^c\subb{k}} \onev_m 
    &&= 0
\end{alignat}
Let $\dg{}{k} = A \onev_{\subb{k}}$ be the vector of node degrees relative to community/subgraph $\subb{k}$. We denote its $i$th element as $\dg{i}{k} = \sum_{j \in \subb{k}} A_{ij}$. Note also that $A_{S_k}\onev_m = [\dg{}{k}]_{\subb{k}}$.
Then, setting $\nu_i :=   \dg{i}{k} - \mu m$ for $i \in \subb{k}$ 
verifies~\eqref{eq:A1:equiv:1a}. To verify~\eqref{eq:A1:equiv:1b}, we need to have
\begin{align}
  (A+\Gamma)_{\subb{k}\subb{\ell}} \onev_m = \mu m \onev_m, \quad \text{for all}\;\ell \neq k.
  \label{eq:Gamma:row:sum:cond}
\end{align}
Note that the same holds for $(A+\Gamma)_{\subb{\ell}\subb{k}}$. That is, every row and column of $(A+\Gamma)_{\subb{k}\subb{\ell}}$, $k\neq \ell$ should sum to a constant ($= \mu m$). 
In other words, $\onev_m$ is a right and left eigenvector of $(A+\Gamma)_{\subb{k}\subb{\ell}}$ associated with eigenvalue $\mu m$. By spectral theorem (i.e., SVD), we should have
\begin{align}\label{eq:B:def:SDP2}
  (A+\Gamma)_{\subb{k}\subb{\ell}} = \mu \onem_m + B_{\subb{k}\subb{\ell}} 
\end{align}
where $B_{\subb{k}\subb{\ell}}$ acts on $\Span\{\onev_m\}^\perp$. 
To satisfy~\eqref{eq:A1:equiv:2}, we first note that
\begin{align*}
  \Span\{\onev_{S_k}\}^\perp = 
  \Big\{ u = \sum_k e_k \otimes u_k:\; u_k \in \reals^m, \, \onev_m^T u_k = 0, \, \forall k\Big\}, 
  \, \text{where}\; e_k = e_k^{(m)}.
\end{align*}
In other words, $\Span\{\onev_{S_k}\}^\perp$ is the set of vectors $u$ such that each sub-vector $u_{S_k}$ sums to zero. Now, take $u \in \Span\{\onev_{S_k}\}^\perp$. Then, $u = \sum_k u_{\subb{k}} = \sum_k e_k \otimes u_k$, for some $\{u_k\} \subset \Span\{\onev_m\}^\perp$, and we have
  \begin{align}\label{eq:u:Lambda:u:expan}
    u^T \Lambda u = \sum_{k,\ell} u_{\subb{k}}^T \Lambda \,u_{\subb{\ell}} = 
    \sum_{k,\ell} u_k^T \Lambda_{\subb{k}\subb{\ell}} \,u_\ell =\; 
    \sum_{k} u_k^T \Lambda_{\subb{k}} u_k + 
    \sum_{k\neq \ell} u_k^T \Lambda_{\subb{k}\subb{\ell}} u_\ell.
  \end{align}
Recall that $\Delta := A - \ex[A]$ where $[\ex A]_{\subb{k}} = p \onem_m$ and $[\ex A]_{\subb{k}\subb{\ell}} = q\onem_m, \; k \neq \ell$.
  Then, $A_{\subb{k}} = p \onem_m + \Delta_{\subb{k}}$ and from~\eqref{eq:Lambda:def:SDP2}, we have $\Lambda_{\subb{k}} =  \mu \onem_m + \diag^*(\nu_{\subb{k}}) - A_{\subb{k}}$. It follows that
  \begin{align*}
    u_k^T \Lambda_{\subb{k}} u_k 
      = u_k^T\big[(\mu -p) \onem_m + \diag^*(\nu_{\subb{k}}) - \Delta_{\subb{k}} \big] u_k 
      = u_k^T\big( \diag^*(\nu_{\subb{k}}) - \Delta_{\subb{k}} \big) u_k 
  \end{align*}
  using the fact that $u_k^T \onem_m u_k = (\onev_m^T u_k)^2 = 0$. We also note from~\eqref{eq:nu:def} that
  \begin{align*}
    \diag^*(\nu_{\subb{k}}) = \diag^*([d(\subb{k})]_{\subb{k}}) - \mu m I_m.
  \end{align*}

  On the other hand, for $k \neq \ell$, we have $\Lambda_{\subb{k}\subb{\ell}} = \mu \onem_m - (A+\Gamma)_{\subb{k}\subb{\ell}} = - B_{\subb{k} \subb{\ell}}$ from~\eqref{eq:Lambda:def:SDP2} and~\eqref{eq:B:def:SDP2}. To summarize, 
  \begin{align}\label{eq:u:Lambda:u:SDP2}
    u^T \Lambda u = \sum_k u_k^T\big( \diag^*([d(\subb{k})]_{\subb{k}}) - \Delta_{\subb{k}} \big) u_k 
    - \mu m \sum_{k} \|u_k\|^2 - \sum_{k \neq \ell} u_k^T B_{\subb{k} \subb{\ell}} u_\ell.
  \end{align}
  To satisfy~\eqref{eq:A1:equiv:2}, we want $u^T \Lambda u$ to be big, which is the case if both $\mu$ and $\{B_{\subb{k} \subb{\ell}}\}$ are small. We are free to choose them subject to dual feasibility constraint $\Gamma \ge 0$, which translates to
  $\mu \onem_m + B_{\subb{k}\subb{\ell}} \ge A_{\subb{k}\subb{\ell}}$. Our construction of $\Gamma_{\subb{k}\subb{\ell}}$ in~\eqref{eq:Gamma:def} corresponds to $B_{\subb{k}\subb{\ell}} := \projonep A_{\subb{k}\subb{\ell}} \projonep $. See also Remark~\ref{rem:trade-off:mu:B} for a discussion of the trade-off involved.

\begin{proof}[Proof of part~(a)]
  To verify dual feasibility, we use  $\projonep e_j = e_j - \frac1m \onev_m$ to write
  \begin{align*}
    [\Gamma_{\subb{k}\subb{\ell}}]_{ij} 
    = e_i^T \Gamma_{\subb{k}\subb{\ell}} e_j 
    &= \mu + (e_i - \frac1m \onev_m)^T A_{\subb{k}\subb{\ell}} (e_j - \frac1m \onev_m) - A_{ij}  \\
    &= \mu - \frac1m e_i^T A_{\subb{k}\subb{\ell}} \onev_m - \frac1{m} \onev_m^T A_{\subb{k}\subb{\ell}} e_j 
      + \frac1{m^2}  \onev_m^T A_{\subb{k}\subb{\ell}} \onev_m\\
    &= \mu - \frac1{m} \big[\dg{i}{\ell}  + \dg{j}{k} - \dav{k}\ell \big] \ge 0
  \end{align*}
  (A3) holds if $ [\Gamma_{\subb{k}\subb{\ell}}]_{ij} > 0$ for at least one $(i,j) \in \subb{k} \times \subb{\ell}$, for each pair $k \neq \ell$, which is equivalent to the stated condition.
\end{proof}
\begin{proof}[Proof of part~(b)]    
  To verify~\eqref{eq:A1:equiv:2}, we recall representation~~\eqref{eq:u:Lambda:u:SDP2}. By assumption $ \diag^*([d(\subb{k})]_{\subb{k}}) \succeq \rho m I_m$ for all $k$. Also, using 
  $B_{\subb{k}\subb{\ell}} = \projonep A_{\subb{k}\subb{\ell}} \projonep$ we have
  \begin{align*}
    u_k^T B_{\subb{k}\subb{\ell}} u_\ell = u_k^T \projonep (q\onem_m + \Delta_{\subb{k}\subb{\ell}}) \projonep u_\ell = 
    u_k^T \projonep \Delta_{\subb{k}\subb{\ell}} \projonep u_\ell = u_k^T \Delta_{\subb{k}\subb{\ell}}  u_\ell
  \end{align*}
  for $\{ u_k \} \subset \Span\{\onev_m\}^\perp$.
  From~\eqref{eq:u:Lambda:u:SDP2} it follows that
  \begin{align*}
    u^T \Lambda u &\ge \rho m \sum_{k} \vnorm{u_k}^2 - \sum_k u_k^T \Delta_{\subb{k}} u_k - \mu m \sum_{k} \vnorm{u_k}^2   
      -\sum_{k \neq \ell}  u_k^T \Delta_{\subb{k}\subb{\ell}}  u_\ell \\
      &= (\rho - \mu) m \|u\|^2 - u^T \Delta u \\
      &\ge \big[(\rho- \mu)m - \mnorm{\Delta}\big] \,\|u\|^2.
  \end{align*}
\end{proof}

\begin{rem}\label{rem:trade-off:mu:B}
    The trade-off in choosing $\mu$ and $B_{\subb{k}\subb{\ell}}$ can be abstracted away in the following subproblem:
    \begin{align*}
      h(\mu) := \min \big\{ \| \Bt\|: \; \mu \onem_m + \Bt \ge \At, \;  \rangeS(\Bt) \subset \Span\{\onev_m\}^\perp \big\}  
     \end{align*}
     where $\At \in \{0,1\}^{m \times m}$ is a non-symmetric adjacency matrix (say, of a directed Erdos-Renyi graph with connection probability $q$). If $\mu = 1$, one can take $\Bt = 0$, hence $h(1) = 0$. As one decreases $\mu$ from $1$, the feasible set of the problem shrinks until the problem becomes infeasible for some $\mu_0 \in (0,1)$, if $\At \neq 0$. We have chosen $\Bt = \projonep \At \projonep$, essentially the largest choice, to make $\mu$ as small as possible. This might not in general be optimal. It would be interesting to study $h(\mu)$ more carefully. For example, another choice is $\Bt = P_V \At P_V$ where $V$ is a proper subspace of $\Span\{\onev_m\}^\perp$ of low dimension. This increases $\mu$, but decreases $h(\mu)$, helping us to better control the contributions of off-diagonal blocks in~\eqref{eq:u:Lambda:u:SDP2}.
  \end{rem}

\section{Probabilistic conditions for $\bsymBM$}\label{sec:prob:cond:ppbal}
We will show when the construction of $(\mu,\nu,\Gamma)$ in~\eqref{eq:nu:def} and~\eqref{eq:Gamma:def} works for the balanced planted partition model,  completing the proof of Theorem~\ref{thm:consist:SDP2}. We start we a consequence of Proposition~\ref{prop:key:adj:concent}.

\begin{cor}\label{cor:adj:concent:SDP2}
  Let $A = (A_{ij}) \in \{0,1\}^{n \times n}$ be drawn from $\bsymBM(p,q)$ with $p \ge (C' \log m)/m $ and $q \ge (C' \log n)/ n$. Then, w.p. at least $1- c (K m^{-r}+ n^{-r})$, 
  \begin{align*}
    \mnorm{A-\ex A} \le C (\sqrt{p\, m} + \sqrt{q n}).
  \end{align*}
\end{cor}
\begin{proof}
  Let $\Delta:= A - \ex A$ and decompose it into its diagonal and off-diagonal blocks. In particular, let $\Stru := \supp(\Xtru) = \bigcup_k \subb{k} \times \subb{k}$ and let $\Stru^c$ be its complement. Then,
  \begin{align*}
    \mnorm{\Delta} \le  \mnorm{\onem_{\Stru} \circ \Delta} +  \mnorm{\onem_{\Stru^c} \circ \Delta} = \max_k \mnorm{\Delta_{\subb{k}}} +  \mnorm{\onem_{\Stru^c} \circ \Delta}
  \end{align*}
  $\onem_{\Stru^c} \circ \Delta$ is an $n \times n$ matrix whose entries have variance $\le q$, hence $\mnorm{\onem_{\Stru^c} \circ \Delta} \le C \sqrt{q n}$ w.p. at least $1-cn^{-r}$. Each $\Delta_k$ is an $m \times m$ matrix whose entries have variance bounded by $p$, hence $\mnorm{\Delta_k} \le C \sqrt{pm}$ w.p. at least $1-c m^{-r}$, for each $k$. The result follows from union bound.
\end{proof}

The following consequence of Bernstein's inequality summarizes the concentration of $\dg{}{k}$ around their mean. For simplicity, we will assume that the diagonal of $A$ is also filled with $\bern(p)$ variates. This has no effect on the optimal primal solution due to the diagonal conditions $X_{ii} = 1$. Recall that
\begin{align*}
  m K = n, \quad \pb := p m, \quad \qb := q m.
\end{align*}

\begin{lem}\label{LEM:DEG:CONCENT:SDP2}
  Let $\gamma := \sqrt{ (4c_1 \log n)/ \pb}$ and $\zeta := \sqrt{ (4c_2 \log n)/ \qb}$ and assume $\gamma,\zeta \in [0,3]$. Then,
  \begin{alignat*}{4}
    \dg{i}{k} &\ge \pb(1-\gamma), 
      &&\; i \in \subb{k}, \forall k 
      &&\quad \text{w.p. at least $1 - n^{-(c_1-1)}$, and} \\
      \big| \dg{i}{\ell} - \qb \big| &\le \zeta \qb, 
      &&\; i \in \subb{k}, \forall (k \neq \ell),
      &&\quad \text{w.p. at least $1-2 m^{-1}n^{-(c_2-2)}$}.
  \end{alignat*}
\end{lem}
The proof is deferred to the Appendix~\ref{sec:proof:deg:concent:SDP2}.
Assume now that the conditions of Lemma~\ref{LEM:DEG:CONCENT:SDP2} (on $\gamma$ and $\zeta$) are met.
Then w.h.p., $\dav{k}{\ell}$ is also in $[\qb (1-\zeta),\qb (1+\zeta)]$, for $k \neq \ell$, so that 
\begin{align*}
  \dg{i}{\ell}  + \dg{j}{k} - \dav{k}\ell \le 2\qb(1+\zeta)-\qb(1-\zeta) = \qb(1+3\zeta).
\end{align*}
Thus, to satisfy~\eqref{eq:mu:low:bound}, it is enough to have $\mu m \ge \qb (1+3\zeta)$. On the other hand, Lemma~\ref{LEM:DEG:CONCENT:SDP2} implies that $\displaystyle m \rho := \min_k \min_{i \in \subb{k}} \dg{i}{k} \ge \pb - \pb \gamma$. Then,
\begin{align*}
  (\rho-\mu) m &\ge \pb - \pb\gamma - \qb - 3\qb \gamma \\
         &\ge \pb - \qb - \sqrt{4c_1\pb \log n} - 3 \sqrt{4c_2 \qb \log n}  
\end{align*}
By Corollary~\ref{cor:adj:concent:SDP2}, w.h.p. $\mnorm{\Delta} \le C(\sqrt{\pb} + \sqrt{\qb K})$, where we have used $q n = \qb K$. Then, to satisfy~\eqref{eq:Delta:up:bound:SDP2}, it is enough to have
\begin{align*}
   \pb - \qb - \sqrt{4c_1\pb \log n} - 3 \sqrt{4c_2 \qb \log n} > C(\sqrt{\pb} + \sqrt{\qb K})
\end{align*}
which is implied by
\begin{align*}
  \pb - \qb> (C+\sqrt{4c_1})\sqrt{\pb \log n} + (C+3\sqrt{4c_2})\sqrt{\qb K \log n}
\end{align*}
in turn implied by
\begin{align}\label{eq:SDP2:log:cond}
  \pb - \qb > C_2 ( \sqrt{\pb \log n} + \sqrt{ \qb K \log n}).
\end{align}

Auxiliary conditions we needed on $\pb$ and $\qb$ were $\pb \ge (4c_1/9) \log n$ and $\qb \ge (4c_2/9) \log n$ from Lemma~\ref{LEM:DEG:CONCENT:SDP2} and $\pb \ge C' \log m$ and $n q > C' \log n$ from Corollary~\ref{cor:adj:concent:SDP2}. We can drop the lower bounds on $q$ due to Corollary~\ref{cor:BM:ordering}. The lower bounds on $\pb$ are implied by $\pb \ge (C' \vee (4c_1/9)) \log n$. This completes the proof. To get to the form in which the theorem is stated, replace $c_1$ with $c_1 + 1$ and $c_2$ with $c_2 + 2$, and divide~\eqref{eq:SDP2:log:cond} by $\log n$.

\section{Proof of Lemma~\ref{LEM:DEG:CONCENT:SDP2}}\label{sec:proof:deg:concent:SDP2}
We recall the following version of Bernstein inequality.
\begin{prop}[Bernstein]
  Let $\{X_i\}$ be independent zero-mean RVs, with $|X_i| \le 1$ almost surely, and let $v := \sum_i \ex[X_i^2]$, then
  \begin{align*}
    \pr \Big( \sum_{i=1}^n X_i > v t\Big) \le \exp [-v \phi(t)], \; t >0, 
      \quad \text{where} \; \phi(t):= \frac{t^2}{2(1 + t/3)}.
  \end{align*}
\end{prop}
   For the first assertion, note that for $i \in \subb{k}$, $\dg{i}{k} = \sum_{j \in \subb{k}} A_{ij}$ is a binomial random variable with mean $mp$ and variance $mp(1-p) \le mp$. then applying Bernstein's with $v = mp$ and $t = \gamma$, we have
  $
    \pr [\dg{i}{k} - mp < - mp\gamma ] \le \exp( -mp \, \phi(\gamma)).
  $
  It follows from union bound that 
  \begin{align*}
    \pr \Big(\min_{k} \min_{i \in \subb{k}}\, \dg{i}{k} \ge \pb (1-\gamma) \Big) \ge 1- m K \exp(-\pb \phi(\gamma))
  \end{align*}
  For $\gamma \in [0,3]$, we have $\phi(\gamma) \ge \gamma^2/4$. It follows that
  $mK \exp(-\pb \phi(\gamma)) \le n \exp(-\pb \gamma^2/4) \le n n^{-c_1}$, proving the first assersion. The second assersion follows similarly, by noting that $\dg{i}{\ell}$ is binomial with mean $\qb = qm$ for $i \in \subb{k}$, $k \neq \ell$. It follows from two-sided Bernstein and union bound that
  \begin{align*}
    \pr \Big(\max_{k\neq\ell} \max_{i \in \subb{k}}\, \big| \dg{i}{\ell} - \qb \big| \le \zeta \qb \Big) 
    &\ge 1-  \big[2\tbinom{K}{2} m\big] 2 \exp(-\qb \phi(\zeta)) \\
    &\ge 1 - (m K)^2 m^{-1} 2 \exp(-\qb \zeta^2/4).
  \end{align*}
  The rest of the argument follows as before.

\section{Proof of Proposition~\ref{prop:failure:SDP-2}}\label{sec:proof:prop:failure:SDP-2}
  The implication $(a)\implies(b)$ follows since in the balanced case $\xi_k = 1, \forall k$.  Let us explain how part~(b) implies part~(c): Let $M := \ex[A]$ for a weakly but not strongly assortative block model. Let $\Mt$ be the matrix obtained from $M$ by setting all the diagonal blocks identically equal to one, except for one of the blocks that violates strong assortativity. Part~(b) applies to $\Mt$ with $|I^c| = 1$ and hence $\text{SDP}_{sol}(\Mt) \neq \{X_0\}$. It then follows from Lemma~\ref{lem:deterministic:nesting} that $\text{SDP}_{sol}(M) \neq \{X_0\}$ which is the desired result.

  The remainder of this section is devoted to proving part~(a).
    We take dual variable $\Lambda$ be of the following form
  \begin{align}\label{eq:Lam:gen:form:2}
    \Lambda &= \sum_{k \in I} \lambda_k (- \onem_{S_k} + \bs_k I_{S_k}),\quad \text{with}\; \lambda_k \ge 0.
  \end{align}
  In order to satisfy $\Lambda X = 0$, it is enough to have $\Lambda \onev_{S_r} = 0, \; r \in I$. This holds for the form given in~\eqref{eq:Lam:gen:form:2}, namely, $\Lambda \onev_{S_r} = \lambda_r (- \onem_{S_r} + \bs_r I_r) \onev_{S_r} = \lambda_r(-\bs_r I_{S_r} + \bs_r I_{S_r}) = 0$. We also assume the following form for $\Gamma$,
  \begin{align}\label{eq:Gam:546}
    \Gamma &= \sum_{k \neq \ell } \rho_{k \ell} \onem_{S_k S_\ell}  + 
    \sum_{k \notin I} \big[ \gamma_k \onem_{S_k} - \gamma_k I_{S_k}\big].
  \end{align}
  For $k \in I$, we have $\alpha_{kk} \neq 0$, and (CSa) implies $\Gamma_{S_k} = 0$. We also have $X_{ii} = 1, \,\forall i$, hence $\Gamma_{ii} = 0, \,\forall i$. The form given in~\eqref{eq:Gam:546} respects these conditions.

  \medskip
  Let $M := \ex[A]$ and recall that $\Lambda = \mu \onem_n + \diag^*(\nu)  - (M+\Gamma)$, hence~\eqref{eq:Lam:gen:form:2} is equivalent to 
  \begin{align}
    \mu \onem_{\bs_k \bs_\ell} - (M+\Gamma)_{S_k S_\ell} &= 0, \quad k \neq \ell \label{eq:cond:167}\\
    \mu \onem_{\bs_k} + \diag^*(\nu_{S_k}) - M_{S_k} &= -\lambda_k \onem_{\bs_k} + \lambda_k \bs_k I_{\bs_k}, \quad k \in I \label{eq:cond:267}\\
    \mu \onem_{\bs_k} + \diag^*(\nu_{S_k}) - (M+\Gamma)_{S_k} &= 0, \quad k \notin I \label{eq:cond:367}.
  \end{align}

  Recall that $M_{S_k S_\ell} = q_{k\ell} \onem_{\bs_k \bs_\ell}$ for $k \neq \ell$, and $M_{S_k} = p_k \onem_{\bs_k}$. Hence,~\eqref{eq:cond:167} is equivalent to $\forall (k\neq \ell)\; \rho_{k\ell} = \mu - q_{k\ell}$. Let us now simplify condition~\eqref{eq:cond:267}. 
  Looking at the diagonal, we have $\mu + \nu_i - p_k = (\bs_k-1)\lambda_k, \; i \in S_k, \, k \in I$. Looking at the off-diagonal, we have $\lambda_k = p_k - \mu$. Hence, $\nu_i = \bs_k \lambda_k$ for $i \in S_k, k \in I$. Now consider~\eqref{eq:cond:367}. For $k \notin I$, the diagonal gives $\mu + \nu_i - p_k  = 0$ since $A_{ii} = p_k$ and $\Gamma_{ii} = 0$ (because of $X_{ii} = 1$.) Hence, $\nu_i = p_k-\mu$ for $i \in S_k, k \notin I$. The off-diagonal gives $\gamma_k = \mu - p_k$. The following table summarizes these relationships:
  \begin{align*}
    \begin{array}{l|l|l}
      k \in I & \; k \notin I & \forall (k \neq \ell)  \\
      \hline
      \lambda_k = p_k - \mu & & \rho_{k\ell} = \mu - q_{k \ell}\\
      \nu_i = n_k \lambda_k, \, i \in S_k & \nu_i = p_k - \mu,\, i \in S_k\\
      & \gamma_k = \mu - p_k& 
    \end{array}
  \end{align*}
  (CSa) implies $(\forall k \notin I)\, \beta_k \gamma_k = 0$, and  $(\forall k,\ell \in I, \, k\neq\ell),\; \rho_{k \ell} \,\alpha_{k \ell} = 0$. Together with dual feasibility, namely $\lambda_k \ge 0$, $\gamma_k \ge 0$,  and $\rho_{k\ell} \ge 0$, we obtain the following restrictions on $\mu$,
  \begin{align}\label{eq:table:325}
    \begin{array}{c|c|c}
      k \in I & \; k \notin I & \forall (k \neq \ell)  \\
      \hline
      \mu < p_k &  \mu \ge p_k &  \mu \ge q_{k \ell}\\
      & \beta_k(\mu - p_k) = 0& \alpha_{k \ell} (\mu - q_{k\ell}) = 0,\; k, \ell \in I
    \end{array}
  \end{align}

  It is interesting to note that when $\{q_{k \ell}, k \neq \ell\}$ are distinct (which is not assumed here), at most one of $\{\alpha_{k \ell}, k < \ell,\; k, \ell \in I\}$ is nonzero. This is enforced by the condition in the last row and column of~\eqref{eq:table:325}, since $\mu$ can be equal to at most one of $q_{k\ell}$.

  Recall that $(k_0,\ell_0) := \argmax_{k \neq \ell} q_{k\ell}$ which is assumed to be unique, and $I := \{k :\; p_k \ge \max_{k \neq \ell} q_{k \ell}\}$. By the assumption of weak assortativity, $\{k_0,\ell_0\} \subset I$.
  Let $\mu = q_{k_0 \ell_0}$. Then, the condition in the last row and column of~\eqref{eq:table:325} implies that $\alpha_{k \ell} = 0, \forall k,\ell \in I\setminus\{ k_0,\ell_0\},\; k \neq \ell$, that is, only $\alpha_{k_0 \ell_0} = \alpha_{\ell_0 k_0}$ could be nonzero. Also, note that with this choice of $\mu$, the first row of~\eqref{eq:table:325} is satisfied.

  Since $\mu > p_k$ for $k \in I^c := [K] \setminus I = \{ k: p_k < q_{k_0\ell_0}\}$, the condition in the second row and column of~\eqref{eq:table:325}, implies that $\beta_k = 0$ for $k \notin I$. It remains to verify primal feasibility, namely, $\ip{X,\onem_n} = m n$. Recall that $\bs_k = \xi_k m$ with $\xi_k \ge 1$ (since $m = \min_k \bs_k$). We need to have
  \begin{align*}
    \ip{X,\onem_n} = \sum_{k \in I} \bs_k^2 + 2 \alpha_{k_0 \ell_0} \,\bs_{k_0} \bs_{\ell_0} + \sum_{k \notin I} \bs_k = n m
  \end{align*}
  Writing $n = \sum_{k} \bs_k$ and dividing by $m$, this is equivalent to
  \begin{align*}
    \sum_{k \in I} \xi_k^2 + 2 \alpha_{k_0 \ell_0} \xi_{k_0} \xi_{\ell_0} + \frac1 m \sum_{k \notin I} \xi_k = \sum_k \xi_k
  \end{align*}
  Some algebra gives the expression $\alpha_{k_0 \ell_0} = \alpha^*_{k_0 \ell_0}$ where the latter is given in~\eqref{eq:alpha:k0:ell0}.
  Under the stated assumption, $\alpha_{k_0 \ell_0} \in [0,1]$, the constructed solution is primal feasible and the proof is complete.

\section{Proof of Proposition~\ref{prop:sdp13}}

  We start by proving part~(b). Part~(a) then follows by simple modifications to the argument. Throughout, we mainly have the case $m = \min_k n_k$ in mind, which adds to the complexity in the construction of the primal-dual witness. When $m < \min_k n_k$, the set $I^c$ that appears below will be empty and the argument simplifies.

  Let $\Lc_2(X) = X \onev_n$ and $b_2 = m\onev_n$. The dual to problem~\eqref{eq:sdp:13} is
  \begin{align*}
    \renewcommand{\arraystretch}{1.3}
    \begin{array}{ll}
      \max\limits_{\Gamma,\,\rho,\,\nu,\,\mu} & - \ip{\rho,b_2} + \ip{\nu,\onev_n} \\
      \text{s.t.} & \Lambda :=  \diag^*(\nu) - \Lc_2^*(\rho) - (A-\mu \onem_n +\Gamma) \succeq 0 \\
      & \Gamma \ge 0, \; \rho \ge 0
    \end{array}
  \end{align*}
  Besides primal and dual feasibility we have the following complementary slackness conditions
  \begin{align*}
    \renewcommand{\arraystretch}{1.3}
    \begin{array}{c|c|c}
      \text{(CSa)} & \text{(CSb)} & \text{(CSc)}\\
      \hline
      \Gamma_{ij} X_{ij} = 0, \;\forall i,j & \; \Lambda X = 0 \; &\rho_i [\Lc_2(X) - b_2 ]_i = 0, \;\forall i
    \end{array}
  \end{align*}
  Consider the potential primal solution given in~\eqref{eq:blk:diag:solution:2} and note that it always satisfies $X_{ii} = 1$, $X \succeq 0$ and $X \ge 0$. We can use (CSc) to make a reasonable choice of $\{\alpha_k\}$.
  For $i \in S_k$, $[\Lc_2(X)]_i = (X \onev_n)_i =  \alpha_k \bs_k + (1-\alpha_k)$, hence $\rho_i [\Lc_2(X) - b_2 ]_i = 0$, together with the corresponding dual feasibility, translate to 
  \begin{align}\label{SDP13:CS:phi:alpha}
    \phi_k [ (\alpha_k \bs_k  + 1-\alpha_k) - m] = 0, \\
      \alpha_k \bs_k  + 1-\alpha_k \ge m
  \end{align}
  Note that if $\bs_k > m$, then setting $\alpha_k = 1$ forces $\phi_k = 0$, hence we lose the flexibility associated with $\phi_k$. This suggests that we should avoid setting $\alpha_k = 1$, as much as possible, unless $\bs_k = m$. For simplicity, let $I_1 := I_1(k_0)$, and recall that $I := \{ k: \bs_k > m\}$ and $I_1 \subset I$. Let $I_2 := I \setminus I_1$ and note that $\alpha_k$ given in part~(b) of the proposition can be written as 
  \begin{align*}
    \alpha_k := 
    \begin{cases}
      \frac{m - 1}{\bs_k-1} < 1, & k \in I_2 \\
      1, & k \in I^c \cup I_1
    \end{cases}
  \end{align*}
where $I^c := [n] \setminus I = \{k: \bs_k = m\}$. Note that this choice frees $\phi_k, \forall k$ to be any nonnegative number, except for $k \in I_1$ where we need $\phi_k = 0$.

\medskip
We now turn to the dual variables. Let us take $\Lambda$ to be of the form
\begin{align*}
  \Lambda = \sum_{k=1}^K \lambda_k (-\onem_{S_k} + \bs_k I_{S_k}), \quad \lambda_k \ge 0.
\end{align*}
$\Lambda$ is block diagonal, and the $k$th block has eigenvalues $\lambda_k (0,\bs_k,\bs_k,\dots,\bs_k)$. We will choose $\rho_{S_k} = \frac12 \phi_k 1_{S_k}$. 
Note that $\Lc^*_2(\rho) = \rho \onev_n^T + \onev_n \rho^T$, hence $[\Lc^*_2(\rho)]_{S_k S_\ell} = \frac12 (\phi_k + \phi_\ell) \onem_{\bs_l,\bs_\ell}$ for all $k,\ell$. With $M:= \ex[A]$, the following has to hold
\begin{align}
  \mu \onem_{\bs_k,\bs_\ell} -  \frac12 (\phi_k + \phi_\ell) \onem_{\bs_k,\bs_\ell} - (M+\Gamma)_{S_k S_\ell} &
    = 0, \quad k \neq \ell \notag \\
  \mu \onem_{\bs_k} + \diag^*(\nu_{S_k}) - \phi_k \onem_{\bs_k} - M_{S_k} &= \lambda_k (-\onem_{\bs_k} + \bs_k I_{\bs_k}).
  \label{eq:temp:156}
\end{align}
In deriving~\eqref{eq:temp:156}, we have used $\Gamma_{S_k} = 0, \, \forall k$ which follows from the particular choice of $X$ in~\eqref{eq:blk:diag:solution:2} and (CSa). 
Let $\psi_k := \mu - \phi_k$. Using $M_{S_k S_\ell} = q_{k\ell} \onem_{\bs_k \bs_\ell}, k \neq \ell$ and $M_{S_k} = p_k \onem_{\bs_k} $, we arrive at
\begin{align*}
  \Gamma_{S_k S_\ell} &= \Big[\frac12(\psi_k + \psi_\ell) - q_{k\ell} \Big] \onem_{\bs_k,\bs_\ell}, \\
  \psi_k + \nu_i  - p_k &= (\bs_k - 1)\lambda_k, \quad i \in S_k \\
  \psi_k - p_k &= - \lambda_k
\end{align*}
where the last two equalities are obtained by considering the diagonal and off-diagonal elements in~\eqref{eq:temp:156}. It follows that $\lambda_k = p_k - \psi_k$ and $\nu_i = \bs_k \lambda_k, \, i \in S_k$.

Dual feasibility implies
\begin{align}
  \phi_k = \mu - \psi_k &\ge 0,\, \forall k \label{eq:temp:945}\\
  \frac12(\psi_k + \psi_\ell) - q_{k\ell} &\ge 0, \quad \forall k \neq \ell \label{eq:temp:475}\\
  \lambda_k = p_k - \psi_k &\ge 0,\quad  \forall k \label{eq:temp:348}
\end{align}
(CSb), namely, $\Lambda X = 0$, translates to $\lambda_k (1-\alpha_k) = 0$, since $\onem_{S_k} (-\onem_{S_k} + m_k I_{S_k}) = 0$ implies
$\Lambda X = \sum_k \lambda_k(1-\alpha_k) (-\onem_{S_k} + \bs_k I_{S_k})$.
In particular, for $k \in I_2$, we have $\alpha_k < 1$, hence $\lambda_k = 0$; otherwise $\lambda_k$ is free to be any nonnegative number. To summarize, (CSb) and (CSc) impose the following restrictions on the dual variables
\begin{align}
  (\forall k \in I_1) \, \phi_k = 0, \;\; (\forall k \in I_2)\, \lambda_k = 0.
  \label{eq:CS:summary}
\end{align}
Recall the inequality $q_{k\ell} \le \frac{1}{2} (\qs_k+\qs_\ell), \; \forall k\neq \ell$. It follows that by choosing $\psi_k \ge \qs_k,\, \forall k$, we can satisfy~\eqref{eq:temp:475}. To satisfy~\eqref{eq:temp:945},~\eqref{eq:temp:348} and~\eqref{eq:CS:summary}, we need
\begin{align}
  (\forall k \in I_1) \,\psi_k = \mu \le p_k, \quad (\forall k \in I_2) \psi_k = p_k \le \mu, \quad (\forall k \in I^c) \,\psi_k \le \min\{p_k, \mu\} 
\end{align}
where we note that $I_1,I_2,I^c$ form a partition of $[K]$. Thus, it is enough to have
\begin{align}\label{eq:temp:596}
  (\forall k \in I_1) \mu \in [\qs_k,p_k], \quad (\forall k \in I_2) \mu \ge p_k, \quad (\forall k \in I^c) \,\qs_k \le \min\{p_k, \mu\} 
\end{align}
Since $\mu \in J_{k_0} \subset \bigcap_{k \in I_1(k_0)} [\qs_k,p_k]$, we have $\mu \in [\qs_k,p_k]$ for all $k \in I_1 = I_1(k_0)$. Since we have taken $\mu \ge p_{k_0+1}$, we have $\mu \ge p_k$ for all $k \ge k_0+1$ due to th assumed ordering of $\{p_k\}$. In particular, $\mu \ge p_k$ for all $k \in I_2$. For $k \in I^c$, either $ k \le k_0$, in which case $\mu \in [\qs_k,p_k]$, i.e. $\qs_k \le \mu = \min\{p_k,\mu\}$, or we have $k \ge k_0+1$ in which case $\mu \ge p_k$, hence $\qs_k \le p_k = \min\{\mu,p_k\}$. Thus, all the conditions in~\eqref{eq:temp:596} are met and the proof is complete.

\bigskip
\emph{Proof of part~(a).} The argument here is similar to that of part~(b). In addition to setting $\mu = 0$, the main difference is that (CSc) and dual feasibility condition $\rho_i \ge 0, \, \forall i$ is replaced by the single primal feasibility condition $\Lc_2(X) - b_2 = 0$. Note that there is no nonnegativity assumption on $\rho$ anymore. The argument goes true if we take $I_1 = \emptyset$ and $I_2 = I$, which ensures that $X \onev_n - m \onev_n = 0$. We now have $\psi_k = - \phi_k$ and the dual feasibility conditions reduce to~\eqref{eq:temp:475} and~\eqref{eq:temp:348}. Furthurmore,~\eqref{eq:CS:summary} is simplified to $(\forall k \in I) \lambda_k = 0$, since only (CSb) is present.  Thus, it is enough to have $\psi_k \ge \qs_k$ for all $k$ and 
\begin{align}
  (\forall k \in I)\, \psi_k = p_k, \quad (\forall k \in I^c) \psi_k \le p_k.
\end{align}
Since $\qs_k \le p_k, \forall k$, by assumption, it is clearly possible to choose $\psi_k$ to satisfy these conditions.

\newcommand{\Lct}{\widetilde{\Lc}}
\newcommand{\bt}{\widetilde{b}}
\newcommand{\mut}{\widetilde{\mu}}
\section{Implementation of SDP-1}\label{sec:implement}
It is straightforward to adapt a first order method to solve the SDP-1
problem~\eqref{eq:our:sdp}. 
We briefly discuss the implementation of an ADMM solver~\cite{Boyd2010}. We start by rewriting the problem~as 
\begin{align*}
    \inf_X \big\{ -\ip{A,X} +\delta_{\{\Lct(X) \,= \,\bt\}}
      + \delta_{\{Z \,\ge\, 0\}} 
      + \delta_{\{Y \,\succeq\, 0\}} \big\} \; \quad \text{s.t.} \; X = Z, \;   X = Y,
\end{align*}
where $\delta_S$ is the indicator of set $S$ defined by $\delta_S(x) = 0$ if $x \in S$ and $= \infty$ otherwise, and $\Lct: \reals^{n \times n} \to \reals^{2 n}$ is a linear
operator such that $\Lct(X) = \bt$ collects the affine constraints
in~\eqref{eq:our:sdp}. More precisely, for $i=1,\dots,n$, we  take
$[\Lct(X)]_i = \ip{X,H_i}$ and $[\Lct(X)]_{i+n} = \ip{X,F_i}$. Here,
$H_i$ is a symmetric matrix with 1 in the off-diagonal elements of the
$i$-th column and row, and 0 everywhere else.  $F_i$ is a matrix with
element $(i,i)$ equal to 1 and 0 everywhere else.   Finally, $\bt_i = 2((n/K)-1)$ for $i=1,\dots,n$ and $b_i = 1$ otherwise. ($\Lct$ is a variation of $\Lc$ that appears in Section~\ref{sec:optim:cond}. It is chosen so that $H_i$ is orthogonal to $F_j$ for all $i,j$. However, $\{\Lct(X) =\bt\}$ and $\{\Lc(X) = b\}$ describe the same affine subspace.)

The only real work in deriving ADMM updates is to find the projection operator $\proj_{\Ac}$ for $ \Ac := \{X: \Lct(X) = \bt\}$. For any $Y$, this projection is given~by 
\begin{align}\label{eq:proj:L:Y}
  \proj_{\Ac}(Y) := Y - \Lct^* (\Lct \Lct^*)^{-1}[ \Lct(Y) - \bt] .
\end{align}
Note that $\ip{H_i,F_j} = 0$ for all $i,j = 1, \dots, n$. Hence, $\Lct \Lct^*$ is block diagonal with two blocks $(\ip{H_i,H_j}) = 2[(n-2) I_n + \onev_n \onev_n^T]$ and 
$(\ip{F_i,F_j}) = I_n$. It follows that
\begin{align*}
  (\Lct\Lct^*)^{-1} = \diag\Big(\frac{1}{2(n-2)} 
    \big[ I_n - \frac{\onev_n \onev_n^T}{2n-2} \big],\; I_n \Big) \ . 
\end{align*}
We also have $\Lct^*(\mut,\nu) = \sum_i \mut_i E_i + \sum_i \nu_i F_i =
(\mut_i + \mut_j)_{i \neq j} + \diag(\nu)$, which gives a complete
recipe to compute $\proj_{\Ac}(Y)$. Note that due to the simplicity of
$(\Lct \Lct^*)^{-1}$ and $\Lct^*$, implementing this projection has
essentially the same computational cost as projecting onto an affine set with two constraints $\{X: \tr(X) = n,\; \ip{\onem_n,X} = n^2/K\}$, which is needed for implementing SDP-2.

The ADMM updates are easily derived to be
\begin{align*}
  X^{k+1} &= \proj_{\Ac}\big(\tfrac12
  ( Z^k - U^k + Y^k - V^k +\tfrac1\rho A) \big) , \\
  Z^{k+1}
  &= 
  \max\{0,X^{k+1} + U^k\}, \quad 
  Y^{k+1} = \proj_{\psd{n}} \big(X^{k+1} + V^k\big), \\
  U^{k+1} &= U^{k} + X^{k+1}-Z^{k+1}, \quad 
  V^{k+1} = V^{k} + X^{k+1}-Y^{k+1}.
\end{align*}
where $\proj_{\psd{n}}$ is the projection onto the PSD cone $\Sc_+^n$,
which can be done by truncating to nonnegative eigenvalues. The ADMM
updates for SDP-2 and SDP-3 can be derived similarly, as in~\cite{Cai2014}.





\section{Details on Figure~\ref{fig:SDP13:and:everythin:else}:  comparing the theoretical predictions with empirical results}
\label{sec:SDP13:fig:details}
Figure~\ref{fig:SDP13:and:everythin:else} provides an illustration of the contents of Propositions~\ref{prop:failure:SDP-2} and~\ref{prop:sdp13}. The results are obtained by numerically solving the SDPs. Here, we explain how they match with our theoretical results. The leftmost panel corresponds to the mean matrix $M = \ex[A]$ of a weakly assortative block model, randomly picked among such models. The specific edge probability matrix is as follows: 
\begin{align*}
  \begin{array}{llllll}
    0.670 &   0.072 &  0.020 &   0.023 &   0.186 &  0.187 \\
    0.072 &   0.570 &  0.521 &   0.016 &   0.360 &  0.107 \\
    0.020 &   0.521 &  0.555 &   0.048 &   0.311 &  0.188 \\
    0.023 &   0.016 &  0.048 &   0.494 &   0.081 &  0.137 \\
    0.186 &   0.360 &  0.311 &   0.081 &   0.475 &  0.031 \\
    0.187 &   0.107 &  0.188 &   0.137 &   0.031 &  0.195
  \end{array}.
\end{align*}
There are six blocks of sizes $\nb = (10,10,5,20,10,10)$. The parameters $p_k$ and $ \qs_k = \max_{\ell \neq k} q_{\ell:\, k \ell}$, for each of the $K=6$ blocks are as follows
\begin{align}\label{eq:p:qs:table}
  \begin{array}{lllllll}
      \qs_k & 0.187 & 0.521 & 0.521 & 0.137 & 0.360 & 0.188  \\
      p_k & 0.670 & 0.570 & 0.555 & 0.494 & 0.475 & 0.195
  \end{array}.
\end{align}
where the overall maximum of the off-diagonal entries is $\max_k \qs_k = 0.521$. It is clear that the last three blocks violate strong associativity.

 We can use part~(a) of Proposition~\ref{prop:failure:SDP-2} to predict the behavior of \SDPp{2}. Note that for this example, $(k_0,\ell_0) := \argmax_{k < \ell} q_{k \ell} = (2,3)$. We have $m = \min_k \bs_k = 5$, hence $(\xi_k)=(2,2,1,4,2,2)$ where $\xi_k = \bs_k /m$, $I = \{k: p_k \ge q_{k_0\ell_0} = 0.521\} = \{1,2,3\}$, $\xi_{k_0} = 2$ and $\xi_{\ell_0} = 1$. It then follows that
\begin{align*}
    \alpha^*_{k_0 \ell_0} = \frac{1}{2 (2 \cdot 1)} \Big[ \Big(1-\frac15\Big)(4+2+2) - (2\cdot1 + 2\cdot1 + 1\cdot 0)\Big] = 0.6.
\end{align*}
Since $\alpha^*_{k_0 \ell_0} \in [0,1]$, the conditions of part~(a) of Proposition~\ref{prop:failure:SDP-2} are met and the solution is of the form~\eqref{eq:alpha:k0:ell0} with $\alpha_1=\alpha_2=\alpha_3 = 1$, $\alpha_{23} = \alpha_{32} = 0.6$ and $\beta_k = 0$ for $ k=4,5,6$.

Let us now apply Proposition~\ref{prop:sdp13} to predict the behavior of SDP-13. Recall that $J_k = \bigcap_{r=1}^k [\qs_r,p_r]$ which in this case gives 
\begin{align*}
  \begin{array}{c|cccccc}
      k & 1 & 2 & 3 & 4 & 5 & 6 \\
      \hline
      \multirow{2}{*}{$J_k$} & 0.187 & 0.521 & 0.521 & 
        \multirow{2}{*}{$\emptyset$} & \multirow{2}{*}{$\emptyset$} & \multirow{2}{*}{$\emptyset$} \\
          & 0.670 & 0.570 & 0.555 & 
  \end{array}.
\end{align*}
We have $k_0 = \max\{k: \; J_k \neq \emptyset\} = 3$ and we can apply SDP-13 with any $\mu \in J_3 \cap [p_4,1] = [0.521,0.555] \cap [0.494,1] = [0.521,0.555]$. The images in Figure~\ref{fig:SDP13:and:everythin:else} are generated with $\mu = 0.55$ and $m = \min_k \bs_k = 5$. We have $I = \{ k : n_k > m\} = \{k: \xi_k > 1\} = \{ 1,2,4,5,6\}$, hence $I_1(k_0) = I_1(3)=\{1,2\}$. Proposition~\ref{prop:sdp13}(b) now applies and we have a solution of the form~\eqref{eq:blk:diag:solution:2} with
\begin{align*}
  \begin{array}{c|cccccc}
      k & 1 & 2 & 3 & 4 & 5 & 6 \\
      \hline
      \alpha_k & 1 & 1 & 1 & 0.211 & 0.444 & 0.444.
  \end{array}
\end{align*}
Note that we have three perfectly recovered blocks in the sense discussed after Proposition~\ref{prop:sdp13}. 

Finally, part~(a) of Proposition~\ref{prop:sdp13} predicts that SDP-1, applied with $m = 5$, has a solution of the form~\eqref{eq:blk:diag:solution:2} with
\begin{align*}
  \begin{array}{c|cccccc}
      k & 1 & 2 & 3 & 4 & 5 & 6 \\
      \hline
      \alpha_k & 0.444 & 0.444 & 1 & 0.211 & 0.444 & 0.444.
  \end{array}
\end{align*}
All the above predictions match what is empirically reported in Figure~\ref{fig:SDP13:and:everythin:else}. We also note that although the behavior of SDP-3 is not mentioned in Propositions~\ref{prop:failure:SDP-2} and~\ref{prop:sdp13}, that solution can also be predicted by careful examination of the proofs.

\section{Details on Figure~\ref{fig:nmi}: the bias of normalized mutual information for large $K$}

\renewcommand{\nmiscale}{.4}
\begin{figure}[t]
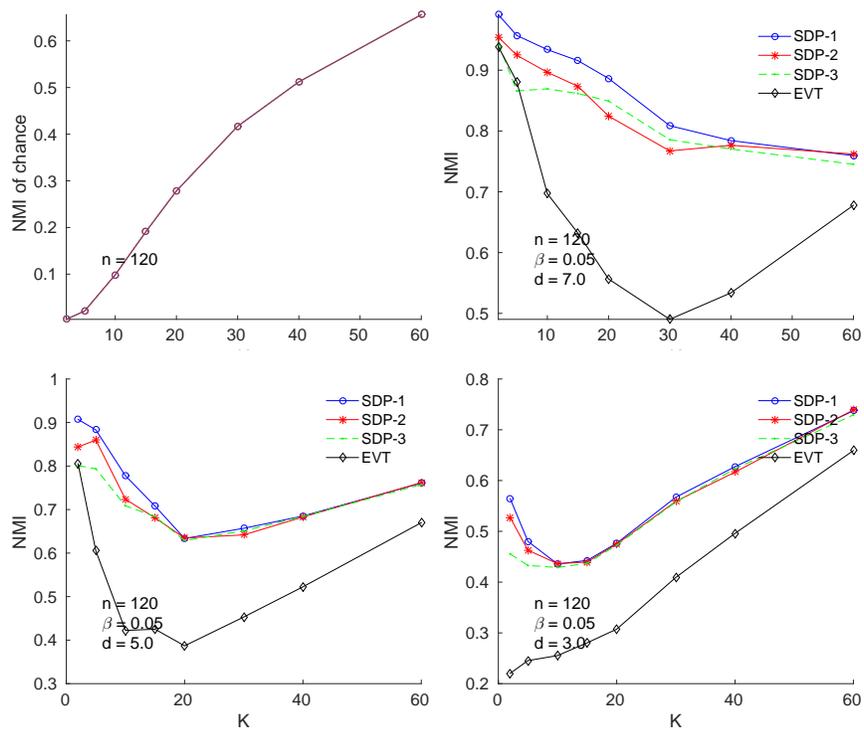

  \centering
  \includegraphics[scale=\nmiscale]{figs/Kvar/{nmi_chance}.eps}
  \includegraphics[scale=\nmiscale]{figs/Kvar/{nmi_n120_l7.0_Kvar_oir0.05_G1}.eps} \\
  \includegraphics[scale=\nmiscale]{figs/Kvar/{nmi_n120_l5.0_Kvar_oir0.05_E1}.eps}
  \includegraphics[scale=\nmiscale]{figs/Kvar/{nmi_n120_l3.0_Kvar_oir0.05_F1}.eps}
 
   \caption{Top left: Average NMI of random guessing (or chance). The other three plots correspond to those of Figure~\ref{fig:nmi} but with raw NMI (no adjustment for chance).}

  \label{fig:nmi:adj}
\end{figure}


Top left panel in Figure~\ref{fig:nmi:adj} shows the average (empirical) NMI of random guessing as a function of $K$.   Contrary to popular belief, the empirical NMI does not automatically adjust so that random guessing corresponds to zero NMI; unless $K$ is small.   It is designed to do so based on population quantities, but for the population quantities to be accurately approximated by empirical ones,  one needs concentration of the counts in the confusion matrix around their means, which does not happen unless $n/K$ is sufficiently large. Figure~\ref{fig:nmi:adj}  shows the plots of Figure~\ref{fig:nmi} after adjusting for random guessing by subtracting the corresponding average NMI. As one would expect, the dip in the curves goes away after the adjustment.

\section*{Acknowledgements}
We would like to thank Karl Rohe and David Choi for interesting discussions regarding the possibility of extending SDPs to mixed models. This research has been partially supported by NSF grants DMS-1106772, DMS-1159005, and DMS-1521551.


\bibliographystyle{abbrv}
\bibliography{bm_sdp7,allref}

\end{document}